%% file: main.tex
\documentclass[opre, nonblindrev]{informs3}
\usepackage{graphicx}
\usepackage{amssymb}
\usepackage{float}    
\usepackage{enumerate}
\usepackage{epstopdf}
\usepackage{bm}
\usepackage{cancel}
\usepackage{enumitem}
\RequirePackage{pdflscape}%
\usepackage{diagbox}
\usepackage{bbold}
\usepackage{mathrsfs}
\usepackage{mathtools}
\usepackage{threeparttable}
\usepackage[usenames,dvipsnames]{xcolor}
\definecolor{DarkBlue}{rgb}{0.0,0.0,0.55}
\usepackage[backref = false, bookmarks, colorlinks = true, plainpages = false, citecolor = DarkBlue , urlcolor = DarkBlue, filecolor = DarkBlue, linkcolor = DarkBlue]{hyperref}
\usepackage{algorithmicx,algorithm}
\usepackage{setspace}
\usepackage{algpseudocode} 
\usepackage{mathrsfs}
\usepackage{url}
\DoubleSpacedXII

\usepackage{endnotes}
\let\footnote=\endnote

\usepackage{subfigure}
\usepackage{natbib}
\usepackage{multirow}
\usepackage{booktabs}
\usepackage{pgf,tikz-cd,pgfplots}
\bibpunct[, ]{(}{)}{,}{a}{}{,}%
\def\bibfont{\small}%
\def\bibsep{\smallskipamount}%
\def\bibhang{20pt}%
\TheoremsNumberedThrough     
\ECRepeatTheorems

\EquationsNumberedThrough    

\input{macros/self-defined.tex}

\begin{document}



\RUNTITLE{Prior-Aligned Meta-RL: Thompson Sampling with Learned Priors and Guarantees in Finite-Horizon MDPs}

\TITLE{{\large{Prior-Aligned Meta-RL: Thompson Sampling with Learned Priors and Guarantees in Finite-Horizon MDPs}}}

\ARTICLEAUTHORS{%
\AUTHOR{Runlin Zhou}
\AFF{Department of Statistics, University of Science and Technology of China, 
}
\AUTHOR{Chixiang Chen}
\AFF{Department of Epidemiology and Public Health, University of Maryland, Baltimore, 
}
\AUTHOR{Elynn Chen\footnote{Correspondence to: Elynn Chen (elynn.chen@stern.nyu.edu)}}
\AFF{Department of Technology, Operations, and Statistics, Stern School of Business, New York University, 
}
} 

\ABSTRACT{%
We study meta-reinforcement learning in finite-horizon MDPs where related tasks share similar structures in their optimal action-value functions. 
Specifically, we posit a linear representation $Q^*_h(s,a)=\Phi_h(s,a)\,\theta^{(k)}_h$ and place a Gaussian meta-prior $ \mathcal{N}(\theta^*_h,\Sigma^*_h)$ over the task-specific parameters $\theta^{(k)}_h$. Building on randomized value functions, we propose two Thompson-style algorithms: (i) MTSRL, which learns only the prior mean and performs posterior sampling with the learned mean and known covariance; and (ii) $\text{MTSRL}^{+}$, which additionally estimates the covariance and employs prior widening to control finite-sample estimation error. Further, we develop a prior-alignment technique that couples the posterior under the learned prior with a meta-oracle that knows the true prior, yielding meta-regret guarantees: we match prior-independent Thompson sampling in the small-task regime and strictly improve with more tasks once the prior is learned. Concretely, for known covariance we obtain $\tilde{O}(H^{4}S^{3/2}\sqrt{ANK})$ meta-regret, and with learned covariance $\tilde{O}(H^{4}S^{3/2}\sqrt{AN^3K})$; both recover a better behavior than prior-independent after $K \gtrsim \tilde{O}(H^2)$ and $K \gtrsim \tilde{O}(N^2H^2)$, respectively. Simulations on a stateful recommendation environment (with feature and prior misspecification) show that after brief exploration, MTSRL/MTSRL\(^+\) track the meta-oracle and substantially outperform prior-independent RL and bandit-only meta-baselines. Our results give the first meta-regret guarantees for Thompson-style RL with learned Q-priors, and provide practical recipes (warm-start via RLSVI, OLS aggregation, covariance widening) for experiment-rich settings.
}%



\maketitle

%


\section{Introduction}

Reinforcement learning (RL) agents are increasingly deployed in settings where they must solve not just one environment, but an array of related tasks. Examples include personalized recommendations, adaptive pricing, and treatment policies in healthcare. In such meta-RL problems, the primary goal, also the central challenge, is to transfer knowledge across tasks so that it can accelerate learning in new environments. Recent work in bandits has begun to address this question. Meta-Thompson Sampling (MetaTS), AdaTS, and hierarchical Bayesian bandits \citep{kveton2021meta,basu2021no,hong2022hierarchical,wan2021metadata} learn priors across bandit tasks, showing that transfer can improve performance. In dynamic pricing, \citet{bastani2022meta} introduced a prior-alignment proof technique and prior widening to control estimation error, though their analysis is confined to horizon-$H=1$ bandits. Meanwhile, meta-RL approaches such as MAML \citep{finn2017model} and PEARL \citep{rakelly2019efficient} focus on representation learning or adaptation, without maintaining explicit Bayesian priors or analyzing Thompson-style regret.

Before delving into new methods in meta-RL, it is worthwhile to highlight that posterior sampling (a.k.a.~Thompson sampling) has emerged as a powerful paradigm for single-task RL, through posterior sampling for RL (PSRL) \citep{osband2013more,osband2016posterior} and randomized value functions such as RLSVI \citep{osband2016generalization,zanette2020frequentist}. 
However, it remains unclear how to extend its benefits to the meta setting where tasks share hidden structure. This paper develops the first Thompson-style algorithms for meta-RL with shared Gaussian priors over optimal value functions. We posit that across tasks, the optimal $Q$-functions admit a linear parameterization with parameters $\theta_h^{(k)}$ drawn from a common Gaussian prior $\mathcal{N}(\theta_h^*,\Sigma_h^*)$. 
This structural assumption shifts the focus from learning dynamics or reward distributions, as in PSRL and RLSVI, to learning a distribution over $Q^*$-parameters. 

Building on this foundation, we design two new Thompson-style meta-RL algorithms:~\underline{M}eta \underline{T}hompson \underline{S}ampling for \underline{RL} (\emph{MTSRL}), which estimates the shared prior mean while assuming known covariance, and \emph{MTSRL$^+$}, which additionally learns the covariance and employs prior widening to ensure robustness. 
We analyze both algorithms through a new prior-alignment framework that couples their learned-prior posteriors to a meta-oracle with the true prior, yielding the first meta-regret guarantees for Thompson sampling in finite-horizon RL. 
Together with simulations in a recommendation environment, our results demonstrate both the theoretical and practical benefits of leveraging learned $Q^*$-priors across tasks.

\smallskip
\noindent
\textbf{Challenges.} 
While prior alignment and prior widening were previously proposed in the bandit setting \citep{bastani2022meta}, extending these techniques to finite-horizon RL is highly non-trivial. 
Algorithmically, the presence of multiple stages $h=1,\ldots,H$ introduces \emph{Bellman dependencies}: each parameter $\theta_h^{(k)}$ must be estimated from temporally correlated trajectories, and errors at later stages propagate backward to earlier ones. 
Designing MTSRL and MTSRL$^+$ required careful integration of (i) OLS-based per-task regression that respects Bellman backups, (ii) cross-task averaging to form a consistent prior mean estimator, and (iii) covariance estimation with \emph{widening} to maintain stability under finite-sample error. 
Theoretically, adapting prior alignment to RL required a new change-of-measure argument that couples the posterior induced by the learned $Q^*$-prior to that of a meta-oracle, while controlling compounding errors across $H$ stages. 
These difficulties make the extension far from a direct generalization of bandit results, and resolving them is central to our analysis.

\smallskip
\noindent
\textbf{Our Contributions.} 
This paper makes the following contributions:
\begin{itemize}
    \item \textbf{First Thompson-style meta-RL algorithms.} We introduce MTSRL and MTSRL$^+$, which learn Gaussian priors over optimal $Q^*$-parameters across tasks and exploit them via posterior sampling.
    \item \textbf{Novel proof technique.} We develop a \emph{prior alignment} argument that couples the learned-prior posterior to a meta-oracle with the true prior, enabling the first \emph{meta-regret guarantees} for Thompson sampling in finite-horizon RL.
    \item \textbf{Robust prior estimation.} We propose covariance \emph{widening} to handle finite-sample uncertainty in estimating $\Sigma_h^*$, ensuring stable performance even under misspecification.
    \item \textbf{Sharp theoretical results.} We show that our algorithms match prior-independent Thompson sampling in the small-$K$ regime and strictly improve in experiment-rich regimes, with bounds of $\tilde{O}(H^4 S^{3/2}\sqrt{ANK})$ (known covariance) and $\tilde{O}(H^4 S^{3/2}\sqrt{AN^3K})$ (unknown covariance).
    \item \textbf{Practical validation.} Simulations in a stateful recommendation environment (with feature and prior misspecification) demonstrate that MTSRL/MTSRL$^+$ closely track the meta-oracle and significantly outperform prior-independent RL and bandit-only baselines.
\end{itemize}

Together, these contributions establish prior-aligned meta-RL as a new direction: Bayesian value-based exploration that learns and exploits shared priors over optimal value functions. Conceptually, our work bridges posterior sampling for single-task RL \citep{osband2013more,osband2016posterior,osband2016generalization,zanette2020frequentist} with meta-Thompson sampling in bandits \citep{kveton2021meta,basu2021no,hong2022hierarchical,bastani2022meta}. Technically, our analysis introduces alignment and widening tools that may be of independent interest in Bayesian RL.

\subsection{Related Work and Our Distinctions}

\noindent\textbf{Posterior sampling and randomized value functions in RL.} 
Posterior Sampling for RL (PSRL) established Bayesian exploration for \textbf{single-task} episodic MDPs and proved near-optimal Bayes regret in tabular settings; subsequent work clarified its limitations in non-episodic settings \citep{osband2013more,osband2016posterior}. 
Randomized Least-Squares Value Iteration (RLSVI) introduced \textit{randomized value functions} with linear function approximation and regret guarantees, motivating posterior-style exploration without optimism \citep{osband2016generalization,zanette2020frequentist}. 
\textit{Our work differs by learning a prior across multiple tasks over $Q^*$-parameters and analyzing meta-regret against a meta-oracle, rather than Bayes regret for a single MDP}. 

\smallskip
\noindent
\textbf{Meta-Thompson sampling and learned priors in bandits.} 
MetaTS, AdaTS, and their extensions study learning the prior across bandit tasks (including contextual and linear) and demonstrate how performance improves as the number of tasks grows \citep{kveton2021meta,basu2021no,hong2022hierarchical}. 
The meta dynamic pricing line goes further by introducing a \textit{prior-alignment} proof technique and \textit{prior widening} for covariance uncertainty \citep{bastani2022meta}. 

We adopt the same high-level idea that learn the prior and then sample, \textit{but extend it to finite-horizon RL with Bellman structure and $H>1$ dynamics}. 
In particular, we learn $Q^*$-priors (rather than reward/arm priors) and establish RL meta-regret via a new alignment analysis tailored to value-function generalization. 
Moreover, while meta-bandit work documents sensitivity of TS to misspecified hyper-priors and proposes prior widening to mitigate finite-sample covariance error, we adapt this idea to \textit{RL with function approximation, proving meta-regret guarantees under learned mean and covariance for $Q^*$ through a prior-alignment change of measure that couples the learned-prior posterior to a meta-oracle}. 

\smallskip
\noindent
\textbf{Hierarchical and multi-task Bayesian bandits.} 
A separate line of work formalizes multi-task learning via \textit{hierarchical priors} and proves Bayes regret benefits from shared structure, with recent advances sharpening bounds and extending to sequential or parallel task arrivals \citep{wang2021multitask,wan2021metadata,hong2022hierarchical,guan2024improved}. 

Beyond hierarchical Bayes approaches, alternative formulations also study shared-plus-private structure across tasks: for example, \citet{xu2025multitask} decompose parameters into a global component plus sparse individual deviations using robust statistics and LASSO, while \citet{bilaj2024meta} assume task parameters lie near a shared low-dimensional affine subspace and use online PCA to accelerate exploration. 

All of these methods, however, operate at horizon $H{=}1$. 
\textit{Our contribution brings hierarchical-prior benefits to multi-step RL, coupling the learned $Q^*$-prior to Bellman updates and analyzing meta-regret in MDPs}. 

\smallskip
\noindent
\textbf{Meta-RL via representation and adaptation (non-Bayesian priors).} 
Meta-RL approaches such as MAML and PEARL learn \textit{initializations or latent task representations} for rapid adaptation \citep{finn2017model,rakelly2019efficient}, while MQL demonstrates strong off-policy meta-training with a context variable for $Q$-learning \citep{fakoor2019meta}. 
Transfer RL across different tasks has been studied in \cite{chen2022transfer,chen2024reinforcement,chen2024data,chai2025deep,chai2025transfer,zhang2025transfer}.
These methods do \textit{not} maintain explicit Bayesian priors over $Q^*$ nor analyze Thompson-style meta-regret. 
\textit{Our approach is complementary: we retain the Bayesian decision-making perspective (posterior sampling) and introduce explicit Gaussian priors over $Q^*$ across tasks}. 

\section{Problem Formulation with Q\(^*\)-Priors}
We study meta-reinforcement learning in finite-horizon MDPs where related tasks share structure in their optimal value functions. 
Unlike classical approaches such as posterior sampling for RL (PSRL) \citep{osband2013more,osband2016posterior} or randomized least-squares value iteration (RLSVI) \citep{osband2016generalization,zanette2020frequentist}, which treat each task independently, we posit that the optimal $Q$-functions admit a linear parameterization with \emph{shared Gaussian priors across tasks}. 
This structural assumption enables posterior-sampling algorithms that explicitly leverage information across MDPs, going beyond existing single-task RL analyses or horizon-$H{=}1$ bandit formulations. 
In particular, Section~\ref{sec:known-prior} develops \emph{TSRL with known Q\(^*\)-priors}, the first Thompson-sampling baseline in RL that admits \emph{meta-regret guarantees relative to a prior-knowing oracle}. 
This benchmark then serves as the foundation for our learned-prior algorithms (MTSRL and MTSRL\(^+\)).

\subsection{Model Setup with Shared Q\(^*\)-Priors}
The $k$-th finite-horizon Markov Decision Process (MDP) is denoted $\mathcal{M}^{(k)} = (\mathcal{S},\mathcal{A},H,P^{(k)},\mathcal{R}^{(k)},\pi)$, where $\mathcal{S}$ is the state space, $\mathcal{A}$ is the action space, $H$ is the horizon, $P^{(k)}$ are the transition kernels, $\mathcal{R}^{(k)}$ are the reward distributions, and $\pi$ is the initial state distribution. 
At each period $h=1,\ldots,H-1$, given state $s_h^{(k)}$ and action $a_h^{(k)}$, the next state $s_{h+1}^{(k)}$ is drawn from $P^{(k)}_{h,s_h^{(k)},a_h^{(k)}}$, and reward $r_h^{(k)} \in [0,1]$ is drawn from $\mathcal{R}^{(k)}_{h,s_h^{(k)},a_h^{(k)}}$. 
Each MDP runs for $N$ episodes, with trajectories indexed by $(s_{nh}^{(k)}, a_{nh}^{(k)}, r_{nh}^{(k)})$.

The optimal value function of MDP $\mathcal{M}^{(k)}$ is 
\[
V_{*,h}^{(k)}(s) = \max_{\mu} \; \mathbb{E}_{\mathcal{M}^{(k)}}\!\left[\sum_{i=h}^{H} R^{(k)}_{i,s^{(k)}_i,\mu(s^{(k)}_i)} \;\big|\; s^{(k)}_h=s\right],
\]
and the corresponding optimal $Q$-function is
\begin{align*}
Q_{*,h}^{(k)}(s,a) = \mathbb{E}_{\mathcal{M}^{(k)}} \big[ R^{(k)}_{h,s^{(k)}_h,a^{(k)}_h} + V_{*,h+1}^{(k)}(s_{h+1}^{(k)})
    \big| s^{(k)}_h = s, a^{(k)}_h = a \big].
\end{align*}
We assume a linear parameterization of the optimal $Q$-function:
\[
Q_{*,h}^{(k)}(s,a) = \Phi_h(s,a)\,\theta^{(k)}_{h}, 
\]
where $\theta^{(k)}_{h}\in \mathbb{R}^M$ is the parameter vector for MDP $k$ and $\Phi_h \in \mathbb{R}^{SA \times M}$ is a generalization matrix whose row $\Phi_h(s,a)$ corresponds to the state–action pair $(s,a)$. 

Crucially, we assume a \textbf{shared Gaussian prior across tasks}:
\[
\theta^{(k)}_{h} \sim \mathcal{N}(\theta^*_h, \Sigma^*_h), \qquad k\in[K], \; h\in[H],
\]
where $(\theta^*_h,\Sigma^*_h)$ are common but unknown. 
This formulation generalizes the bandit setting of \citet{bastani2022meta} (recovered when $H=1$, $S=1$), and forms the basis for the algorithms in Section~\ref{sec:known-prior} and beyond.

\subsection{TSRL with Known $Q^*$-Priors: A Meta-Regret Baseline} \label{sec:known-prior}
We begin with the benchmark case in which the agent is given access to the \emph{true Gaussian prior} over the task-specific optimal $Q^*$-parameters. 
This setting is distinct from existing posterior-sampling methods in RL: PSRL \citep{osband2013more,osband2016posterior} assumes generative models over rewards and transitions, while RLSVI \citep{osband2016generalization,zanette2020frequentist} relies on randomized value functions without cross-task structure. 
It also extends beyond meta-Thompson sampling in bandits \citep{kveton2021meta,basu2021no,bastani2022meta}, which are confined to horizon-$H{=}1$ problems and priors over reward parameters. 

We introduce two algorithms:~the \underline{T}hompson \underline{S}ampling for \underline{RL} algorithm with a \emph{known prior} (TSRL) and its enhanced version TSRL$^+$.  
In contrast to existing methods, TSRL is the first algorithm to employ \textbf{Gaussian priors directly over $Q^*$-parameters} in \textbf{finite-horizon RL}, and we analyze its \textbf{meta-regret against a prior-knowing oracle}. 
This establishes a principled baseline that both clarifies our theoretical target and motivates the learned-prior algorithms ( MTSRL and MTSRL\(^+\)) in Section \ref{sec:algorithms}. 

For convenience, we use the shorthand notation $\{\cdot\}$ to denote the collection $\{\cdot\}_{h=1}^{H}$ of horizon-dependent quantities, whose cardinality is $H$. We also suppresses the task index $k$ for the rest of this section. 

\smallskip
\noindent
\textbf{TSRL.} 
TSRL is defined as a \emph{meta baseline}: it assumes access to the true shared prior $(\theta^*_h, \Sigma^*_h)$ and applies posterior sampling to each task $M^{(k)}$ independently. 
Thus TSRL can be regarded as a prior-informed analogue of RLSVI at the task level, but crucially it serves as the \emph{meta-oracle benchmark} for our regret analysis across multiple tasks.
Given the prior mean $\{\theta^{*}_{h}\}$, covariance $\{\Sigma^{*}_{h}\}$, and the number of episodes $N$, TSRL proceeds in the same manner as RLSVI but incorporates the prior in posterior updates. 
In each episode $n$, the algorithm computes posterior parameters $\{\theta^{TS}_{nh}\}$ and $\{\Sigma^{TS}_{nh}\}$ from the observed trajectory history and the prior $\{\theta^{*}_{h}\}, \{\Sigma^{*}_{h}\}$. 
It then samples $\tilde{\theta}_{nh} \sim \mathcal{N}(\theta^{TS}_{nh}, \Sigma^{TS}_{nh})$ for each $h$, and selects actions greedily according to
\[
\vspace{-1ex}
a_{nh} \in \arg\max_{\alpha \in \mathcal{A}} \big( \Phi_h \tilde{\theta}_{nh} \big)(s_{nh}, \alpha).
\]
The environment returns reward $r_{nh}$ and next state $s_{n,h+1}$, which are used to update posterior estimates. 
Over time, TSRL learns estimates $\tilde{\theta}_{Nh}$ that approximate the underlying parameters $\theta_{h}$. 
Further details and thoeretical guarantees are given in  Section~\ref{sec:TSRL}.

\smallskip
\noindent
\textbf{TSRL\(^+\).} 
TSRL+ enhances TSRL by introducing an initialization phase with RLSVI, which enhances the stability of the prior estimates.
The pseudocode is provided in Algorithm~\ref{alg:TSRL+}.
Specifically, we introduce a positive input parameter $\lambda_e$ and run RLSVI, which is equivalent to TSRL($\{0\}, \{ \tfrac{1}{\lambda} \mathbf{I}\}, 1$), during the initialization phase. This process continues until the Fisher information matrix
\vspace{-1ex}
\[
V^{(k)}_{nh} = \sum_{i=1}^{n-1} \Phi_{h}^\top(s^{(k)}_{ih},a^{(k)}_{ih}) \Phi_{h}(s^{(k)}_{ih},a^{(k)}_{ih})
\]
achieves a minimum eigenvalue of at least $\lambda_e$, ensuring that a well-defined OLS estimate of $\theta_{h}^{(k)}$ is obtained by the end of the $N$ epochs. This prepares the estimates for subsequent use in the meta-Thompson sampling for RL (MTSRL).

Let $\mathcal{N}^{(k)}_{h}$ denote the (random) length of this initialization period,
\[
\mathcal{N}^{(k)}_{h} = \arg\min_{n} \left\{ \lambda_{\min}\!\left(V^{(k)}_{nh}\right) \geq \lambda_e \right\}.
\]
We show in Appendix~\ref{append:proe} that $\mathcal{N}_{h}^{(k)} = \tilde{O}(1)$ with high probability, under the assumption $\min_{h,s,a} \lambda_{\min}(\Phi_{h}^\top(s,a)\Phi_{h}(s,a)) \ge \lambda_{0}$. 
Thus, the initialization occupies only a negligible fraction of the overall runtime, after which $\text{TSRL}^{+}$ proceeds as TSRL with the known prior.

\begin{algorithm}
\caption{$\text{TSRL}^+$($\{\theta^{*}_{h}\}$,$\{\Sigma^{*}_{h}\}$,$\lambda_e$,$N$)}
\begin{algorithmic}[1]
    \setlength{\itemsep}{0pt}
    \State \textbf{Input:} \\
    Data $\left\{\Phi_1(s_{i1}, a_{i1}), r_{i1}, \ldots,\Phi_{H}(s_{iH}, a_{iH}), r_{iH}\right\}_{i < n}$, exploration parameter $\lambda_e$, prior mean vectors $\{\theta^{*}_{h}\}$, 
    covariance matrixs $\{\Sigma^{*}_{h}\}$, epochs' amounts $N$, and noise parameter $\{\beta_n\}_{n=1}^{N}$, $\tilde{\theta}_{H+1} = 0$.
    \State \textbf{Initialization:} $n \leftarrow 1$, 
\While{$\exists h, \lambda_{\min} \left(\sum_{i=1}^{n-1}\Phi_{h}^\top(s_{ih},a_{ih})\Phi_{h}(s_{ih},a_{ih})\right) < \lambda_e$}
    \State Run TSRL($\{0\}$,$\{\frac{1}{\lambda} \mathbf{I}\}$, 1)
    \State $n\leftarrow n+1$
\EndWhile
\While{$n \leq N$}
    \State Run TSRL($\{\theta^{*}_{h}\}$,$\{\Sigma^{*}_{h}\}$, 1)
\EndWhile
\end{algorithmic}
\label{alg:TSRL+}
\end{algorithm}

\section{Learning $Q^*$-Priors: The MTSRL and MTSRL\(^+\) Algorithms}
\label{sec:algorithms}

We now move from the single-MDP setting to the {\it meta setting with multiple tasks}, where the goal is to leverage the shared prior structure $\theta_h^{(k)} \sim \mathcal{N}(\theta_h^*, \Sigma_h^*)$ across tasks.
To this end, we introduce two Thompson sampling-based algorithms: \emph{Meta Thompson Sampling for RL (MTSRL)} and its enhanced variant \emph{MTSRL$^+$}.
MTSRL estimates a common prior mean across tasks via OLS regression while assuming the covariance $\{\Sigma_h^*\}$ is known, and then performs posterior sampling using this learned mean.
MTSRL$^+$ removes the known-covariance assumption by jointly estimating both the prior mean and covariance, and employs \emph{prior widening} to control finite-sample estimation error, thereby achieving improved robustness.


\smallskip
\noindent
\textbf{Meta-oracle (known prior).}~We define the \emph{meta-oracle policy} that, for each task $\mathcal{M}^{(k)}$, runs $\text{TSRL}^+$ with the \emph{true} prior $(\{\theta_h^*\},\{\Sigma_h^*\})$ (Section~\ref{sec:known-prior}).
We compare our learned-prior algorithms to this oracle.

\subsection{MTSRL (Known Covariance)}
\label{sec:alg:mtsrl}
We first consider the setting where the prior covariance $\{\Sigma_h^*\}$ is known. 
The corresponding algorithm, MTSRL, is presented in Algorithm~\ref{alg:MTSRL}. 
In this case, the first $K_0 = \tilde{O}(H^2)$ tasks are allocated to an initial exploration phase, 
during which the algorithm relies on a prior-independent strategy. 
Once this warm-up is completed, MTSRL transitions to exploiting the shared structure across tasks. 
Specifically, for each task $k$, the procedure operates in two regimes:
\vspace{-2ex}
\begin{enumerate}[label=(\roman*)]
    \item \textbf{Epoch $k \leq K_{0}$}: MTSRL executes the prior-independent Thompson sampling algorithm RLSVI \citep{osband2016generalization,russo2019worst}, 
    which corresponds to running Algorithm~\ref{alg:TSRL+} with a conservative prior.  

    \item \textbf{Epoch $k > K_{0}$}: MTSRL leverages past data to estimate the shared prior mean. 
    For each previous task $j < k$ and for every stage $h$, it computes an OLS estimate of the parameter
    \[
    \dot{\theta}_h^{(j)} = V_{Nh}^{(j)^{-1}} \left( \sum_{i=1}^{N} \Phi_{h}(s^{(j)}_{ih}, a^{(j)}_{ih})^\top \dot{b}_{ih}^{(j)} \right),
    \]
    where $\dot{b}_{ih}^{(j)} = r_{ih}^{(j)} + \max_{\alpha} \big( \Phi_{h+1}\dot{\theta}_{h+1}^{(j)} \big)(s_{i,h+1}^{(j)}, \alpha)$ if $h < H$, 
    and $\dot{b}_{ih}^{(j)} = r_{ih}^{(j)}$ if $h=H$ (with $\dot{\theta}^{(j)}_{H+1}=0$). 
    These task-specific estimates are then averaged to form an estimator of the prior mean:
    \begin{equation}\label{equ:methe}
        \hat{\theta}^{(k)}_h = \frac{1}{k-1}\sum_{j=1}^{k-1} \dot{\theta}_h^{(j)}.
    \end{equation}
    Finally, MTSRL runs Thompson Sampling (Algorithm~\ref{alg:TSRL+}) on task $k$ using the estimated prior 
    $(\{\hat{\theta}^{(k)}_h\},\{\Sigma_h^*\})$, i.e., $\text{TSRL}^+(\{\hat{\theta}^{(k)}_h\},\{\Sigma_h^*\}, \lambda_e, L)$.
\end{enumerate}

\begin{algorithm}
\caption{MTSRL Algorithm}\label{alg:MTSRL}
\begin{algorithmic}[1]
\setlength{\itemsep}{0pt}
\State \textbf{Input:} The prior covariance matrix $\{\Sigma^*_h\}$, the total number of MDPs $K$, the episode amount of each MDP $N$, the length of each episode $H$, the noise parameter $\{\beta_n\}_{n=1}^{N}$, $\tilde{\theta}_{H+1} = 0$.
\For{each MDP epoch $k = 1, \ldots, K$}
    \If{$k \leq K_{0}$}
        \State Run $\text{TSRL}^+(\{0\},\{\frac{1}{\lambda} \mathbf{I}\}, \lambda_e, N)$.
    \Else
        \State Update $\{\hat{\theta}_h^{(k)}\}$ according to Eq. \ref{equ:methe}, and run $\text{TSRL}^+(\{\theta^{(k)}_h\},\{\Sigma_h^*\}, \lambda_e, N)$.
    \EndIf
\EndFor
\end{algorithmic}
\end{algorithm}

\subsection{MTSRL\texorpdfstring{$^{+}$}{+} (Unknown Covariance)}
\label{sec:alg:mtsrlplus}
When $\{\Sigma_h^*\}$ is unknown, we additionally estimate and \emph{widen} the prior covariance.
The $\text{MTSRL}^+$ algorithm is presented in Algorithm 3. We first define some additional notation, and then describe the algorithm in detail.

\textbf{Additional notation:} 
To estimate $\Sigma_h^*$, we require unbiased and independent estimates for the unknown true parameter realizations $\theta_h^{(k)}$ across MDPs. Instead of using all $N$ steps as in the MTSRL algorithm, we utilize the initialization steps $n \in [\mathcal{N}_j]$ (where $\mathcal{N}_j = \max_{h}\{\mathcal{N}_h^{(j)}\}$) to generate an estimate $\ddot{\theta}_h^{(j)}$ for $\theta_h^{(j)}$, and an expected $\ddot{\Sigma}_h^{(j)}$ for $\Sigma_h^{(j)}$, i.e.,$\forall j < k$, and $\forall h$
        \begin{align*}
            & \ddot{\theta}_h^{(j)} = V_{\mathcal{N}_{j}h}^{(j)^{-1}} \left( \sum_{i=1}^{\mathcal{N}_{j}} \Phi_{h}(s^{(j)}_{ih}, a^{(j)}_{ih})^\top \ddot{b}_{ih}^{(j)} \right), \\
            &  \ddot{\Sigma}_h^{(j)} = \mathbb{E}\left(\ddot{\theta}_h^{(j)}-\theta_h^{(j)}\right)\left(\ddot{\theta}_h^{(j)}-\theta_h^{(j)}\right)^\top.
        \end{align*}
    Here 
    \begin{align*}
            \ddot{b}_{ih}^{(j)} \leftarrow 
    \begin{cases}
        r_{ih}^{(j)} + \max_{\alpha} \left( \Phi_{h+1} \ddot{\theta}_{h+1}^{(j)} \right) (s_{i, h+1}^{(j)}, \alpha) & \text{if } h < H \\
        r_{ih}^{(j)}  & \text{if } h = H
    \end{cases},
    \end{align*}
    and we define $\ddot{\theta}^{(j)}_{H+1} = 0, \forall j$.

\textbf{Algorithm Description:} The first $K_1$ epochs are treated as exploration epochs, where we employ the prior-independent Thompson Sampling algorithm and $K_1  = \tilde{O}(H^2N^2)$. 

Note that we now require $\tilde{O}(H^2N^2)$ exploration epochs, whereas we only required $\tilde{O}(H^2)$ exploration epochs for the MTSRL algorithm.

As described in the overview, the $\text{MTSRL}^+$ algorithm proceeds in two phases:
\vspace{-2ex}
\begin{enumerate}[label=(\roman*)]
    \item \textbf{Epoch $k \leq K_1$}: the MTSRL algorithm runs the prior-independent Thompson sampling algorithm (\cite{osband2016generalization},\cite{russo2019worst}) RLSVI. This is simply Algorithm \ref{alg:TSRL+} with a conservative prior.

    \item \textbf{Epoch $k > K_1$}: the $\text{MTSRL}^+$ algorithm computes an estimator $\hat{\theta}_h^{(k)}$ of the prior mean $\theta_h^{*}$ using Eq. \ref{equ:methe} (in the same manner as MTSRL algorithm) through $\ddot{\theta}_h^{(j)}$, and an estimator $\hat{\Sigma}_h^{(k)}$ of the prior covariance $\Sigma_h^*$ as 
    \begin{equation}\label{eq:sighat}
    \begin{aligned}
        &\qquad\qquad\quad\hat{\theta}^{(k)}_h = \frac{\sum_{j=1}^{k-1} \ddot{\theta}_h^{(j)}}{k-1},\\
        &\hat{\Sigma}_h^{(k)} = \frac{1}{k-2} \sum_{i=1}^{k-1} \left( \ddot{\theta}_h^{(i)} - \frac{\sum_{j=1}^{k-1} \ddot{\theta}_h^{(j)}}{k-1} \right) \bigg( \ddot{\theta}_h^{(i)}- \frac{\sum_{j=1}^{k-1} \ddot{\theta}_h^{(j)}}{k-1} \bigg)^\top
        - \frac{\sum_{i=1}^{k-1} \ddot{\Sigma}_h^{(k)} }{k-1}.
    \end{aligned}
    \end{equation}

    
    As noted earlier, we then \emph{widen} our estimator to account for finite-sample estimation error:
    \begin{equation}\label{eq:sigw}
        \hat{\Sigma}_h^{w(k)} = \hat{\Sigma}_h^{(k)} + w \cdot I_{M},
    \end{equation}

    where $w$ is widen-parameter, and $I_{M}$ is the $(M)$-dimensional identity matrix.
    
    Then, the $\text{MTSRL}^+$ algorithm runs Thompson Sampling (Algorithm \ref{alg:TSRL+}) with the estimated prior $(\{\hat{\theta}_{h}^{(k)}\}, \{\hat{\Sigma}_h^{w(k)}\}$, i.e., $\text{TSRL}^+(\{\theta^{(k)}_h\},\{\Sigma_h^{w(k)}\}, \lambda_e, L)$. 
    
\end{enumerate}

\begin{algorithm}\label{alg:MTSRLplus}
\caption{$\text{MTSRL}^+$ Algorithm}
\begin{algorithmic}[1]
\setlength{\itemsep}{0pt}
\State \textbf{Input:} The total number of MDPs $K$, the epoch amount of each MDP $N$, the length of each epoch $H$, the noise parameter $\{\beta_n\}_{n=1}^{N}$, widen-parameter $w$, $\tilde{\theta}_{H+1} = 0$.
\For{each MDP epoch $k = 1, \ldots, K$}
    \If{$k \leq K_1$}
        \State Run $\text{TSRL}^+(\{0\},\{\frac{1}{\lambda} \mathbf{I}\}, \lambda_e, N)$.
    \Else
        \State Update $\{\hat{\theta}_h^{(k)}\}$ and $\{\hat{\Sigma}_h^{(k)} \}$ according to Eq. \ref{equ:methe} and \ref{eq:sighat}, 
        \State Compute widened estimate $\{\hat{\Sigma}_h^{w(k)} \}$ according to Eq. \ref{eq:sigw}, 
        \State run $\text{TSRL}^+(\{\theta^{(k)}_h\},\{\Sigma_h^{w(k)}\}, \lambda_e, N)$.
    \EndIf
\EndFor
\end{algorithmic}
\end{algorithm}


\section{Theory: Meta-Regret Analysis}
\label{sec:theory}

We measure performance relative to the \emph{meta-oracle} that knows $(\{\theta_h^*\},\{\Sigma_h^*\})$ and runs $\text{TSRL}^+$ on each task.

\paragraph{Regret and meta-regret.}
For a policy $\mu$ and task $\mathcal{M}^{(k)}$, define the per-task regret over $N$ episodes as
\[
\text{Regret}^{(k)}(N;\mu)=\sum_{n=1}^{N}\mathbb{E}_{\mathcal{M}^{(k)}}\!\left[V_{*,1}^{(k)}(s^{(k)}_{n1})-\sum_{h=1}^{H}r^{(k)}_{nh}\right].
\]
The \emph{meta-regret} of $\mu$ over $K$ tasks is
\[
\mathcal{R}_{K,N}(\mu)=\sum_{k=1}^{K}\mathbb{E}\!\left[\sum_{n=1}^{N}\sum_{h=1}^{H}\big(r^{\text{oracle}(k)}_{nh}-r^{(k)}_{nh}\big)\right],
\]
where $r^{\text{oracle}(k)}_{nh}$ is the reward obtained on task $k$ by the meta-oracle (TSRL$^+$ with the true prior).

We make the following standard assumptions.

\begin{assumption}[Positive-definite prior covariance]
\label{ass:pd}
For all $h\in[H]$, $\lambda_{\min}(\Sigma_h^*)\ge \underline{\lambda}>0$.
\end{assumption}

\begin{assumption}[Bounded features and parameters]
\label{ass:bounded}
For all $(h,s,a)$, $\|\Phi_h(s,a)\|\le \Phi_{\max}$ and $\|\theta_h^*\|\le S$.
\end{assumption}

These assumptions ensure well-posed posteriors and bounded per-step variance, as is standard in linear value-function analyses.

\smallskip
\noindent
\textbf{Known-prior benchmark (oracle).} \label{sec:oracle-bound}
The theorem \ref{thm:oracle} analyzes the Bayes regret of the Meta-oracle policy.

\begin{theorem}[Oracle benchmark]
\label{thm:oracle}
Under Assumptions~\ref{ass:pd}--\ref{ass:bounded}, the regret of running $\text{TSRL}^+$ with the true prior on each task satisfies
\[
\sup_{\{\mathcal{M}^{(k)}\}_{k=1}^{K}} \sum_{k=1}^{K}\text{Regret}^{(k)}(N;\text{TSRL}^+) \;\le\; \tilde{O}\!\left(H^{3}S^{3/2}\sqrt{AN}\,K\right).
\]
\end{theorem}
This result highlights the best possible performance one can achieve with perfect prior knowledge, serving as a benchmark for comparing the MTSRL and MTSRL\(^+\) algorithms, which estimate the prior from data.

\smallskip
\noindent
\textbf{Meta-regret of MTSRL (known covariance) and MTSRL$^{+}$ (unknown covariance).}\label{sec:theory-mtsrl}
Theorems \ref{thm:MTSRL} and \ref{thm:MTSRLplus} provide the meta-regret bounds for the MTSRL and MTSRL\(^+\) algorithms, respectively, characterizing their performance relative to the Meta-oracle policy.
Detailed proofs are presented in Section \ref{append:proMSTRL} and \ref{append:proMSTRL+} in the supplemental material.

\begin{theorem}
\label{thm:MTSRL}
Let $K_0=\tilde{O}(H^2)$. Under Assumptions~\ref{ass:pd}--\ref{ass:bounded}, the meta-regret of Algorithm~\ref{alg:MTSRL} satisfies
\[
\mathcal{R}_{K,N}(\text{MTSRL})=
\begin{cases}
\tilde{O}\!\left(H^{3}S^{3/2}\sqrt{AN}\,K\right), & K\le K_0,\\[3pt]
\tilde{O}\!\left(H^{4}S^{3/2}\sqrt{AN\,K}\right), & K>K_0.
\end{cases}
\]
\end{theorem}

\begin{theorem}
\label{thm:MTSRLplus}
Let $K_1=\tilde{O}(H^2N^2)$ and define $\hat{\Sigma}_h^{w(k)}$ as in \eqref{eq:sighat}--\eqref{eq:sigw}. Under Assumptions~\ref{ass:pd}--\ref{ass:bounded}, the meta-regret of Algorithm~\ref{alg:MTSRLplus} satisfies
\[
\mathcal{R}_{K,N}(\text{MTSRL}^{+})=
\begin{cases}
\tilde{O}\!\left(H^{3}S^{3/2}\sqrt{AN}\,K\right), & K\le K_1,\\[3pt]
\tilde{O}\!\left(H^{4}S^{3/2}\sqrt{AN^{3}K}\right), & K>K_1.
\end{cases}
\]
\end{theorem}

For small numbers of tasks ($K \lesssim \tilde{O}(H^2)$ for MTSRL; $K \lesssim \tilde{O}(H^2N^2)$ for MTSRL$^+$), our meta-regret matches the prior-independent Thompson sampling rate, as shown in Theorems \ref{thm:MTSRL} and \ref{thm:MTSRLplus}, reflecting the exploration phase. As $K$ grows, the learned prior improves performance, yielding the stated $\tilde{O}$ dependencies. These results formalize that prior learning is particularly advantageous in \emph{experiment-rich} regimes.


\section{Details about TSRL algorithm}\label{sec:TSRL}

We next detail the TSRL algorithm and its theoretical bounds. While TSRL can be tightened to a $\sqrt{HS}$ bound \citep{agrawal2021improved}, this refinement is beyond our scope and omitted here.

\subsection{The TSRL algorithm}
Let $H_n = (s_{11}, a_{11}, r_{11}, \dots, s_{n-1,H}, a_{n-1,H}, r_{n-1,H})$ denote the history of observations made prior to period $n$. Observing the actual realized history $H_n$, the algorithm computes the posterior  $\mathcal{N}\left( \theta^{TS}_{nh},\Sigma^{TS}_{nh} \right), h \in [H]$ for round $n$. Specifically, we define $b_{ih}=r_{ih} + \max_{\alpha} \left( \Phi_{h+1} \tilde{\theta}_{i, h+1} \right) (s_{i, h+1}, \alpha)$ for $h < H$, and $b_{ih}=r_{ih}$ for $h = H$.
The posterior at period $n$ is:
    \begin{align*}
         \theta^{TS}_{nh} \leftarrow  &\left( \frac{1}{\beta_n} \sum_{i=1}^{n-1}\Phi_{h}^\top(s_{ih},a_{ih})\Phi_{h}(s_{ih},a_{ih}) + \Sigma_h^{*-1} \right)^{-1} \bigg(\frac{1}{\beta_n}\sum_{i=1}^{n-1}\Phi_{h}^\top(s_{ih},a_{ih})b_{ih} + \Sigma_h^{*-1}\theta^*_h\bigg) \\
        \Sigma^{TS}_{nh} \leftarrow &\left( \frac{1}{\beta_n} \sum_{i=1}^{n-1}\Phi_{h}^\top(s_{ih},a_{ih})\Phi_{h}(s_{ih},a_{ih}) +  \Sigma_h^{*-1} \right)^{-1}
    \end{align*}


\begin{algorithm}[t]
    \caption{TSRL($\{\theta^{*}_{h}\}$,$\{\Sigma^{*}_{h}\}$,$n$):Known-Prior \underline{T}hompson \underline{S}ampling in \underline{RL}}
    \begin{algorithmic}[1]
    \setlength{\itemsep}{0pt}
        \State \textbf{Input:} $\left\{\Phi_1(s_{i1}, a_{i1}), r_{i1}, \ldots,\Phi_{H}(s_{iH}, a_{iH}), r_{iH}\right\}_{i < n}$, the noise parameter $\{\beta_n\}_{n=1}^{N}$,the prior mean vectors $\{\theta^{*}_{h}\}$ and covariance matrixs $\{\Sigma^{*}_{h}\}$,  $\tilde{\theta}_{H+1} = 0$.
        \For{$n = 1, \ldots, N$}
            \For{$h = H, \ldots, 1$}
                \State Compute the posterior $\theta^{TS}_{nh},\Sigma^{TS}_{nh}$
                \State Sample $\tilde{\theta}_{nh} \sim \mathcal{N}\left( \theta^{TS}_{nh},\Sigma^{TS}_{nh} \right)$ from Gaussian posterior
            \EndFor
            \State Observe $s_{n1}$
            \For{$h = 1, \ldots, H-1$}
                \State Sample $a_{nh} \in \arg\max\limits_{\alpha \in \mathcal{A}} \left( \Phi_h \tilde{\theta}_{nh} \right)(s_{nh}, \alpha)$
                \State Observe $r_{nh}$ and $s_{l, h+1}$
            \EndFor
            \State Sample $a_{nH} \in \arg\max\limits_{\alpha \in \mathcal{A}}\left( \Phi_H \tilde{\theta}_{nH} \right)(s_{nH}, \alpha)$
            \State Observe $r_{nH}$
        \EndFor
    \end{algorithmic}
\label{alg_known}
\end{algorithm}




To delve into the motivation of the algorithm, we offer both a mathematical interpretation and an intuitive explanation in Appendix \ref{sec:exT}.

\subsection{TSRL: Bayesian Regret Analysis}

We impose the following standard assumption.
\begin{assumption}
For $\forall (n,h,s,a)$, when $\Sigma^{*}_{h} = \text{diag}(\sigma_h^{*2}(s,a))_{s,a}$, and $\sigma^{*2}_{h}(s,a)/ \beta_{n}= \nu_{nh}(s,a)$, we have: $\nu_{nh}(s,a) \le \overline{\nu}$
\end{assumption}

This assumption is intended to constrain the influence of the prior. With this assumption in place, we now proceed to establish the corresponding results.

\begin{theorem}\label{thm:pri}
If Algorithm \ref{alg_known} is executed with $\Phi_h = I$ for $h = 1, ..., H$, $\Sigma^{*}_{h} = \text{diag}(\sigma_h^{*2}(s,a))_{s,a}$, 
then for a tuning parameter sequence $\{\beta_n\}_{n \in \mathbb{N}} $ with $\beta_n= 4 \max(1,\overline{\nu}) SH^3 log(2HSAn)$:
\vspace{-1.5ex}
\[
\sup_{\mathcal{M}} \text{Regret}(N;\text{TSRL}) \;\le\; \tilde{O}\!\left(H^{3}S^{3/2}\sqrt{AN}\right).
\vspace{-2ex}
\]
\end{theorem}
The proof is given in Section \ref{append:proof_of_alg_known} in the supplemental material. When the prior for $\sigma^{2}_h(s,a)$ is too small (e.g. $\nu \to 0$), the prior dominates and the observed data becomes meaningless. Conversely, if $\beta$ is too small(e.g. $\nu \to \infty$), reducing the algorithm to an unperturbed version that ignores the prior.

\section{Simulation}

In this section, we validate our theoretical results through simulations with a sequential recommendation engine. We empirically compare the performance of our proposed algorithms against prior-independent methods and bandit meta-learning algorithms, focusing on both meta-regret and Bayes regret. The results demonstrate that our meta-learning approach significantly enhances performance.

\textbf{Model.}
An agent sequentially recommends up to $P (\leq \overline{P})$ products from $Z = \{1, 2, \ldots, \overline{P}\}$ to $K$ customers. For customer $k$, let the set of observed products be $\tilde{Z}^{(k)} \subseteq Z$. For each porduct $n \in \tilde{Z}^{(k)}$, $x_n \in \{-1, +1\}$ denotes $\{$dislike, like$\}$; for $n \notin \tilde{Z}^{(k)}$, $x_n = 0$. The probability that customer $k$ likes a new product $a \notin \tilde{Z}$ follows a logistic model:
\begin{equation}\label{equ:sim}
    \mathbb{P}(a|x) = 1 / (1 + \exp(-[\beta_a^{(k)} + \sum_n \gamma_{an}^{(k)}x_n])).
\end{equation}
The agent aims to maximize total likes per customer. It does not know $p(a|x)$ and must learn parameters $\beta^{(k)}$, $\gamma^{(k)}$ through interaction across customers. Each customer forms $N$ episode of horizon H = P with a cold start($\tilde{Z}^{(k)} = \emptyset$). In simulations, $\beta_a^{(k)} = 0$ for all $a$, and $\gamma_{an}^{(k)} \sim N(0,c^2)$, independently.

The state space size is $|S| = |\{-1, 0, +1\}|^H = 3^{P}$, so generalization is essential. We use basis functions $\phi_i(x, a) = \mathbb{1}\{a = i\}$ and $\phi_{ij}(x, a) = x_j \mathbb{1}\{a = i\}$ for $\forall 1 \leq i,j, a \leq \overline{P}$. At period $h$ we form $\Phi_h = (( \phi_i)_i, (\phi_j)_j)$; the function class dimension is $M = \overline{P}^2 +  \overline{P}$, exponentially smaller than $|S|$, though generally misspecified.

\textbf{Experimental Setting.}~We compare: (i) RLSVI without priors (prior-free approach) and (ii)\underline{M}eta \underline{T}hompson \underline{S}ampling in \underline{B}an\underline{d}it algorithm, i.e MTSBD(Appendix \ref{sec:MTSBD}). (iii) our $\text{MTSRL}^{+}$ algorithm. Two practical misspecifications are considered:
\begin{enumerate}[noitemsep,topsep=0pt,parsep=5pt]
    \item Feature misspecification: the true Q-function may lie outside span($\Phi_{h}$).
    \item Prior misspecification: we assume a Gaussian prior on  $\gamma$ rather than directly on $\theta_{h}$, so the implied prior on $\theta_{h}$ need not be Gaussian.
\end{enumerate}
These two forms of misspecification simulate real-world scenarios and further demonstrate the robustness of our algorithm. We use the true $\gamma$ to compute the corresponding true $\theta_{h}^{(k)}$ and its Gaussian-assumed prior as the meta oracle.

\textbf{Parameter settings.} $K=100$, $N=200$, $\overline{P}=10$, $H=P=5$, $c=2$, and algorithm hyperparameters: $\lambda=0.2$, $\lambda_e=2$ , $w = 1$ and $\beta_{n}=10^{-3}n$, $N_1 = 5$. Each MDP is solved exactly to compute regret. Results are averaged over 10 random instances, each with 10 simulation runs. 

\textbf{Results.}~The following figure \ref{fig:2} presents the results for both subfigures, where the function class dimension is set to  \( M = 10 \) and the x-axis represents the number of customers \( K \) in each case. The left panel shows the cumulative meta-regret for four algorithms: prior-free meta-learning, MTSBD, MTSRL\(^+\), and meta oracle. The right panel presents the corresponding Bayes regret for these algorithms. 

\begin{figure}[t] 
    \centering 
    \includegraphics[width=1\textwidth]{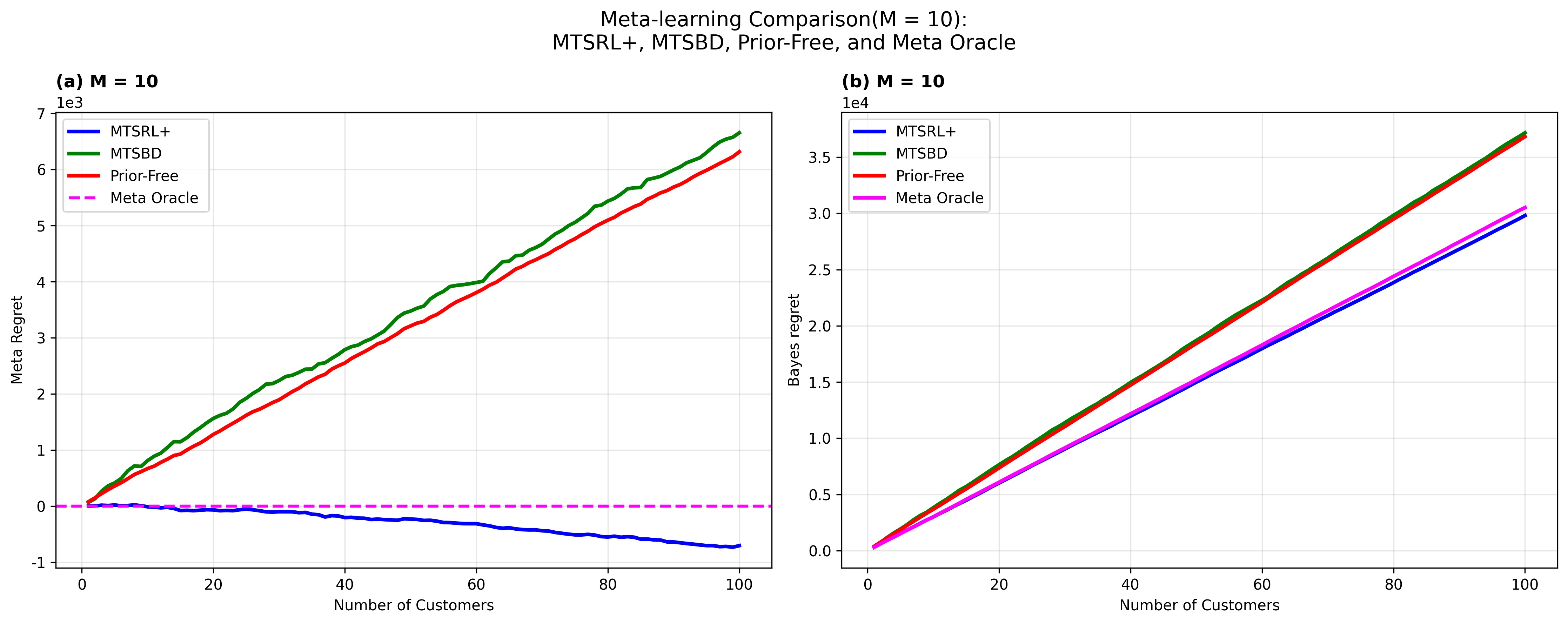} 
    \caption{Comparison between Algorithms}
    \label{fig:2} 
\end{figure}

As expected \emph{(left panel)}, the prior-independent method shows meta-regret growing roughly linearly with \( K \), which aligns with treating customers independently. In contrast, $\text{MTSRL}^{+}$ quickly drives the meta-regret to near zero after the initial exploration phase, effectively learning the prior—and in our runs, even slightly outperforming the meta-oracle! We attribute this to: (i) computational error, arising because the true prior \( \theta_{h} \) is not prescribed directly, but is estimated indirectly via OLS regression based on the computed \( Q_{h}(s,a) \) values (from \( Q_{h}(s,a) = \Phi_{h}(s,a)\theta_{h} \)), which introduces error; and (ii) the widening step in \( \text{MTSRL}^{+} \), which accelerates meta-learning.

Bandit meta-learning (MTSBD) initially outperforms the prior-independent approach by quickly learning a strong myopic policy. However, it is eventually overtaken as the prior-independent method accumulates data to learn richer multi-period policies.

For Bayes regret \emph{(right panel)}, the results more clearly show that the performance of \( \text{MTSRL}^{+} \) and the meta-oracle are comparable, while the performance of the bandit meta-learning algorithm is similar to that of the prior-independent algorithm. At \( K = 200 \), prior-independent Thompson Sampling exhibits 32\% higher Bayes regret than \( \text{MTSRL}^{+} \). These results highlight the advantage of learning shared structure in experiment-rich recommendation environments.

\section{Conclusion}
We proposed \emph{MTSRL} and \emph{MTSRL$^+$}, Thompson-style algorithms for meta-RL with Gaussian priors over $Q^*$-parameters. Using OLS regression, cross-task averaging, and covariance widening, they extend posterior sampling beyond single-task RL and bandit meta-learning.  

Our analysis introduced a \emph{prior-alignment} technique that couples learned and oracle posteriors, giving the first \emph{meta-regret guarantees} for Thompson sampling in finite-horizon RL. The bounds recover prior-independent rates when tasks are few, and improve in experiment-rich regimes. Simulations confirm robustness and gains under misspecification.  

This work establishes \emph{prior-aligned meta-RL} as a principled way to exploit shared structure across tasks. Alignment and widening techniques may benefit Bayesian RL more broadly. Future work includes nonlinear function classes, adaptive horizons, and large-scale applications.


\newpage
\bibliographystyle{informs2014}
\bibliography{bib/ref}
\clearpage

\ECSwitch



\begin{APPENDICES}


\begin{center}
\large{Online Appendix for "Prior-Aligned Meta-RL"}
\end{center}

\input{appendix/appendix_v2}

\vspace{.2in}

\end{APPENDICES}

\end{document}

%% file: macros/self-defined.tex
\renewcommand{\hat}{\widehat}
\renewcommand{\tilde}{\widetilde}

\newcommand{\bfm}[1]{\ensuremath{\boldsymbol{#1}}} 

\def\bi{\bfm i}

\usepackage{mathtools}
\DeclarePairedDelimiterX{\infdivx}[2]{(}{)}{%
  #1 \; \delimsize\| \; #2%
}

\newcommand*\xbar[1]{%
  \hbox{%
    \vbox{%
      \hrule height 0.4pt 
      \kern0.5ex
      \hbox{%
        \kern-0em
        \ensuremath{#1}%
        \kern-0em
      }%
    }%
  }%
} 
\definecolor{royalpurple}{rgb}{0.47, 0.32, 0.66}
\definecolor{greenfresh}{HTML}{00897B}
\definecolor{bluefresh}{HTML}{1E88E5}
\definecolor{redfresh}{HTML}{E53935}

\definecolor{royalpurple}{rgb}{0.47, 0.32, 0.66}

\def\beq{\begin{equation}}
\def\eeq{\end{equation}}

\def\bet{\begin{theorem}}
\def\eet{\end{theorem}}

\def\bel{\begin{lemma}}
\def\eel{\end{lemma}}



%% file: appendix/appendix_v2.tex
\section{Mathematical Explanation and Intuitive Understanding of 'TSRL'}\label{sec:exT}

For simplicity, we define some notation. For empirical estimates. We define $ l_n(h, s, a) = \sum_{i=1}^{n-1} \mathbb{I}\{(s_{ih}, a_{ih}) = (s, a)\} $ to be the number of times action $ a $ has been sampled in state  $s$ , period $ h $. For every tuple $(h, s, a)$ with $ l_n(h, s, a) > 0 $, we define the empirical mean reward and empirical transition probabilities up to period $ h $ by

\begin{equation*}
\hat{R}_{h, s, a} = \frac{1}{l_n(h, s, a)} \sum_{i=1}^{n-1} \mathbb{1}\{(s_{ih}, a_{ih}) = (s, a)\} r_{ih}    
\end{equation*}

\begin{equation*}
\hat{P}_{h, s, a}(s') = \frac{1}{l_n(h, s, a)} \sum_{i=1}^{n-1} \mathbb{1}\{(s_{ih}, a_{ih}, s_{i,h+1}) = (s, a, s')\} \quad \forall s' \in \mathcal{S}.
\end{equation*}

If $(h, s, a)$ was never sampled before episode $ n $, we define $\hat{R}_{h, s, a} = 0$ and $\hat{P}_{h, s, a} = 0 \in \mathbb{R}^S$. And $\hat{M}^{(k)}=(\mathcal{S},\mathcal{A},H,\hat{P}^{(k)},\hat{\mathcal{R}}^{(k)},s_1)$\

\subsection{Posterior estimation Given a Known Prior}\label{sec:PEKP}
    For convince for explanation, we let $H_{nh} = (s_{1h}, a_{1h}, r_{1h}, \dots, s_{n-1,h}, a_{n-1,h}, r_{n-1,h})$, for the data select from timestep h in every epoch before. It's easy for us to use the Bayes rules $\Pr(\theta_h|H_{nh}) \propto \Pr(H_{nh}|\theta_h) \Pr(\theta_h)$
    
    At first, we have a know prior $\theta_{h} \sim \mathcal{N}(\theta^*_h, \Sigma^*_h) \quad \text{for each h} $, so we have:
    \begin{align*}
        \Pr(\theta_h)\propto \exp\left\{ -\frac{1}{2}\left(\theta_h-\theta
        _h^*\right)^\top\Sigma_h^{*-1}\left(\theta_h-\theta_h^*\right)\right\}
    \end{align*}
         
    and  artificially add gaussian noise from $r_h$ to $r_h + z_h$ ,here $\forall h,\, z_h \sim \mathcal{N}(0, \beta_n)$ i.i.d , for 
    when know $\{s_h,a_h,s_{h+1}\}$, we have TD error:$Q_{h}^*(s_h,a_h) = r_h + \max_{a}Q_{h+1}^*(s_{h+1},a) + z_h$. For computational convenience, we aggregate it into matrix form:$A = (\Phi_h(s_{1h}, a_{1h})^\top,\cdots,\Phi_h(s_{n-1,h}, a_{n-1,h})^\top)^\top \in \mathbb{R}^{(n-1) \times M}$,  $b =(b_1,\cdots,b_{n-1})^\top \in \mathbb{R}^{n-1}$, so 
    \begin{align*}
        \Pr(H_{nh}|\theta_h)\propto \exp\left\{ -\frac{1}{2}\left(b-A\theta_h\right)^\top(\beta_{n}I_{n-1})^{-1}\left(b-A\theta_h\right)\right\}
    \end{align*}

    Specifically, we use the update rule for Bayesian linear regression \cite{bishop2006pattern} in value iteration.
    so we have
    \begin{align*}
    &\Pr(\theta_h|H_{nh}) \\
    &\propto \Pr(H_{nh}|\theta_h) \Pr(\theta_h) \\
    &\propto \exp\left\{-\frac{1}{2}\left( \theta_h^\top\Sigma_h^{*-1}\theta_h-2\theta_h^{*\top}\Sigma_h^{*-1}\theta_h + \beta_{n}^{-1}\theta_h^\top A^\top A \theta_h - 2\beta_{n}^{-1} b^\top A \theta_h\right)\right\} \\
    &\propto \exp\left\{-\frac{1}{2}\left( \theta_h^\top(\Sigma_h^{*-1}+\beta_{n}^{-1}A^\top A)\theta_h-2(\theta_h^{*\top}\Sigma_h^{*-1}+\beta_{n}^{-1} b^\top A)\theta_h\right)\right\} \\
    &\propto \exp\left\{-\frac{1}{2}(\theta_h - \theta_{nh}^{TS})^T (\Sigma_{nh}^{TS})^{-1}(\theta_h - \theta_{nh}^{TS})\right\} \\
    &\propto \mathcal{N}(\theta_{nh}^{TS},\Sigma_{nh}^{TS})
    \end{align*}

\subsection{Intuitive Understanding}\label{sec:Ma}
To facilitate a fundamental understanding of our algorithm in subsequent discussions, we first examine the following posterior computation:

Consider Bayes updating of a scalar parameter $\theta \sim N(0, \beta)$ based on noisy observations $Y = (y_1, \ldots, y_n)$ where $y_i \mid \theta \sim N(0, \beta)$. 
The posterior distribution has the closed form
\begin{align*}
    \theta \mid Y \sim N \left( \frac{1}{n+1} \sum_{i=1}^n y_i, \frac{\beta}{n+1} \right).
\end{align*}

To better align with our example, we modify the prior assumption:

Consider Bayes updating of a scalar parameter $\theta \sim N(\theta^*, \sigma^{*2})$ based on noisy observations $Y = (y_1, \ldots, y_n)$ where $y_i \mid \theta \sim N(0, \beta)$. 
\begin{align*}
	\theta \mid Y \sim N \left( \frac{\sigma^{*-2}\beta}{n+\sigma^{*-2}\beta}\theta^* + \frac{1}{n + \sigma^{*-2}\beta} \sum_{i=1}^n y_i, \frac{\beta}{n+\sigma^{*-2}\beta} \right).
\end{align*}

For more any $(s,a)$, we let $\theta_h \sim N(\theta_h^*, \Sigma_h^*)$. When the basis functions $\Phi_{h} = I$, it's easy to find that $Q_h(s,a) = \theta_{h}$. To facilitate proof we let $\Sigma_h^*= \text{diag}(\sigma_h^{*2}(s,a))_{s,a}$; $Y = (y_1, \ldots, y_n)$ where $y = r(s,a) + \max_{a'} Q_{h+1}(s',a')$. We define $l_n(h, s, a) = \sum_{i= 1}^{n-1}\mathbb{1}\{(s_{ih}, a_{ih}) = (s, a)\}$,
\begin{align*}
    \tilde{Q}_h(s, a) \mid \tilde{Q}_{h+1} &\sim N ( \frac{\sigma_h^{*-2}(s,a)\beta}{l_n(h, s, a) + \sigma_h^{*-2}(s,a)\beta }\theta_h^* + \frac{l_n(h, s, a)}{l_n(h, s, a) + \sigma_h^{*-2}(s,a)\beta}\\
    & (\hat{R}_{h,s,a} + \sum_{s' \in S} \hat{P}_{h,s,a}(s') \max_{a' \in A} \tilde{Q}_{h+1}(s', a')), \frac{\beta}{l_n(h, s, a) + \sigma_h^{*-2}(s,a)\beta} )\\
    &\sim \frac{\sigma_h^{*-2}(s,a)\beta}{l_n(h, s, a) + \sigma_h^{*-2}(s,a)\beta }\theta_h^* + \frac{l_n(h, s, a)}{l_n(h, s, a) + \sigma_h^{*-2}(s,a)\beta}\\
    & (\hat{R}_{h,s,a} + \sum_{s' \in S} \hat{P}_{h,s,a}(s') \max_{a' \in A} \tilde{Q}_{h+1}(s', a')) +w_h(s,a)
\end{align*}
where $w_h(s,a) \sim N(0,\frac{\beta}{l_n(h, s, a) + \sigma_h^{*-2}(s,a)\beta})$. This provides a mathematical intuition for our algorithm. Specifically: when we plug  $\Phi_{h} = I$ and $\Sigma_h^*= \text{diag}(\sigma_h^{*2}(s,a))_{s,a}$ into our algorithm's $\Sigma^{TS}_{h}$ and $\theta^{TS}_{h}$, it's easy to find that 
\begin{align*}
    \tilde{\theta}_h(s, a) \mid \tilde{\theta}_{h+1} &\sim N (\theta^{TS}_{h},\Sigma^{TS}_{h})\\
    &\sim N ( \frac{\sigma_h^{*-2}(s,a)\beta}{l_n(h, s, a) + \sigma_h^{*-2}(s,a)\beta }\theta_h^* + \frac{l_n(h, s, a)}{l_n(h, s, a) + \sigma_h^{*-2}(s,a)\beta}\\
    & (\hat{R}_{h,s,a} + \sum_{s' \in S} \hat{P}_{h,s,a}(s') \max_{a' \in A} \tilde{\theta}_{h+1}(s', a')), \frac{\beta}{l_n(h, s, a) + \sigma_h^{*-2}(s,a)\beta} )
\end{align*}

In comparison with conventional MDP estimation:
\begin{align*}
    \tilde{Q}_h(s, a) \mid \tilde{Q}_{h+1} \leftarrow \hat{R}_{h,s,a} + \sum_{s' \in S} \hat{P}_{h,s,a}(s') \max_{a' \in A} \tilde{Q}_{h+1}(s', a').
\end{align*}
In the current form, our $\frac{l_n(h, s, a)}{\sigma_h^{*-2}(s,a)\beta + l_n(h, s, a) }\hat{P}_{h,s,a}(s')$ is no longer a valid probability function, and it is for ease of presentation. 
To deep under stand our design, we can slightly augment the state space by adding one absorbing state for each level $h$ (\cite{agrawal2021improved}); then first $\sigma^{*-2}\beta$ times will transit to the absorbing states, and get the value function $V_h = \theta_{h}^*$. 
		
And the last $l_n(h, s, a)$ times will transit to the normal states without absorbing state, and get the value function  $V_h = r(s,a) + \max_{a'} Q_{h+1}(s',a')$. 

\section{Proof of Theorem \ref{thm:pri}}\label{append:proof_of_alg_known}

Let $\tilde{Q}_{n,h} = \Phi_h \tilde{\theta}_{nh}$ and $\tilde{\mu}_n$ denote the value function and policy generated by RLSVI for episode $n$ and let $\tilde{V}_{n,h}(s) = \max_a \tilde{Q}_{n,h}(s, a)$. We can decompose the per-episode regret
\begin{align*}
    V_{*,1}(s_{1}) - V_{\tilde{\mu}_n,1}(s_{1}) = \tilde{V}_{n,1}(s_{1}) - V_{\tilde{\mu}_n,1}(s_{1}) + V_{*,1}(s_{1}) - \tilde{V}_{n,1}(s_{1}).
\end{align*}

The proof follows from several lemmas.

\textbf{Control of empirical MDP} Through a careful application of Hoeffding's inequality, one can give a high probability bound on the error in applying a Bellman update to the (non-random) optimal value function $V_{h+1}^*$. Through this, and a union bound, Lemma \ref{lem:confidence} bounds the expected number of times the empirically estimated MDP falls outside the confidence set

\begin{align*}
    \mathcal{M}^n = \left\{(H, S, \mathcal{A}, P', R', s_1): \quad \forall (h, s, a) \left|(R'_{h, s, a} - R_{h, s, a}) + \langle P'_{h, s, a} - P_{h, s, a}, V_{h+1}^* \rangle \right| \leq \sqrt{e^k(h, s, a)} \right\}
\end{align*}

where we define
\[
\sqrt{e_n(h, s, a)} = H \sqrt{\frac{\log (2HSAn)}{l_n(h, s, a) + 1}}.
\]
This set is a only a tool in the analysis and cannot be used by the agent since \( V_{h+1}^* \) is unknown.

\begin{lemma}[Validity of confidence sets]
\label{lem:confidence}
\[
\sum_{k=1}^{\infty} \mathbb{P} \left( \hat{M}^n \notin \mathcal{M}^n \right) \leq \frac{\pi^2}{6}.
\]
\end{lemma}

\textbf{From value function error to on policy Bellman error.} For some fixed policy \(\pi\), the next simple lemma expresses the gap between the value functions under two MDPs in terms of the differences between their Bellman operators. We'll apply this lem:core lemma several times.

\begin{lemma}\label{lem:value_diff}
Consider any policy \(\mu\) and two MDPs \(\hat{M} = (H, S, \mathcal{A}, \hat{P}, \hat{R}, s_1)\) and \(\tilde{M} = (H, S, \mathcal{A}, \tilde{P}, \tilde{R}, s_1)\). Let \(\hat{V}_{\mu,h}\) and \(\tilde{V}_{\mu,h}\) denote the respective value functions of \(\pi\) under \(\hat{M}\) and \(\tilde{M}\). Then
\[
\tilde{V}_{\mu,1}(s_1) - \hat{V}_{\mu,1}(s_1) = \mathbb{E}_{\pi, \tilde{M}} \left[ \sum_{h=1}^{H} \left( \tilde{R}_{h, s_h, \mu(s_h)} - \hat{R}_{h, s_h, \mu(s_h)} \right) + \langle \tilde{P}_{h, s_h, \mu(s_h)} - \hat{P}_{h, s_h, \mu(s_h)}  , \hat{V}_{h+1}^{\mu} \rangle\right],
\]
where \(\hat{V}_{H+1}^{\mu} \equiv 0 \in \mathbb{R}^S\) and the expectation is over the sampled state trajectory \(s_1, \ldots, s_H\) drawn from following \(\pi\) in the MDP \(\overline{M}\).
\end{lemma}

\textbf{Sufficient optimism through randomization.} In contrast to approaches like UCB, which maintain optimism for all value functions, our algorithm guarantees that the value function is optimistically estimated with probability at least a fixed constant. Recall $M$ is the unknown true MDP with optimal policy $\mu^{*}$ and $\tilde{M}^n$ lis RLSVI's noise-perturbed MDP under which $\mu^n$ is an optimal policy.

\begin{lemma}\label{lem:needcore}
Let $\pi^*$ be an optimal policy for the true MDP $M$. Then 
\[
\mathbb{P}\left(\tilde{V}_{n,1}(s_{1}) \geq V_{*,0}(s_{1}) \mid \mathcal{H}_{n-1}\right) \geq \Phi(-1).
\]
\end{lemma}

This result is more easily established through the following lemma, which avoids the need to carefully condition on the history $\mathcal{H}_{n-1}$ at each step. We conclude with the proof of Lemma \ref{lem:core} after.

\begin{lemma}\label{lem:core}
Fix any policy $\mu = (\mu_1, \ldots, \mu_H)$ . Consider the MDP $M = (H, \mathcal{S}, \mathcal{A}, P, R, s_1)$, if lemma \ref{lem:confidence} remains valid. Then in $n$ episode,
\[
\mathbb{P} \left( \tilde{V}_{\mu,1}(s_{1}) \geq  V_{\mu,1}(s_{1}) \right) \geq \Phi(-1).
\]
\end{lemma}

\begin{proof}{Proof of lemma \ref{lem:core}:}
To start, we let $s = (s_1, \ldots, s_H)$ denote a random sequence of states drawn by simulating the policy $\mu$ in the MDP $\bar{M}$ from the deterministic initial state $s_1$. Set $a_h = \mu(s_h)$, and $w(h, s, a) \sim N(0,\frac{\beta}{l_n(h, s, a) + \nu_h(s,a)})$ for $h = 1, \ldots, H$. Then by lemma \ref{lem:value_diff}, we have
\begin{align*}
    &\tilde{V}_{\mu,1}(s_{1}) -  V_{\mu,1}(s_{1}) = \mathbb{E} [ \sum_{h=1}^H  \frac{\nu_h(s,a)}{l_n(h, s, a) + \nu_h(s,a) }\theta_h^*(s,a) + \frac{l_n(h, s, a)}{l_n(h, s, a) + \nu_h(s,a)} (\hat{R}_{h,s,a} 
    + \langle \hat{P}_{h,s,a},   V_{\mu,h+1} \rangle)\\
    &+ w(h, s, a)- R_{h, s, a} - \langle P_{h, s, a}, V_{\mu,h+1} \rangle] \\
    &= \mathbb{E} [ \sum_{h=1}^H  (\frac{\nu_h(s,a)}{l_n(h, s, a) + \nu_h(s,a) }\theta_h^*(s,a) + \frac{l_n(h, s, a)}{l_n(h, s, a) + \nu_h(s,a)} (\hat{R}_{h,s,a} 
    + \langle \hat{P}_{h,s,a},   V_{\mu,h+1} \rangle) - \hat{R}_{h, s, a} - \langle \hat{P}_{h, s, a}, V_{\mu,h+1} \rangle ) \\
    &+(\hat{R}_{h, s, a} + \langle \hat{P}_{h, s, a}, V_{\mu,h+1}\rangle- R_{h, s, a} - \langle P_{h, s, a}, V_{\mu,h+1} \rangle) + w(h, s, a)] \\
    & \geq \mathbb{E} \left[  \sum_{h=1}^H w(h, s, a)\right] - \mathbb{E}\left[ \sum_{h=1}^H \frac{\nu_h(s,a)}{l_n(h, s, a) + \nu_h(s,a) }|\theta_h^*(s,a)-\hat{R}_{h,s,a} 
    - \langle \hat{P}_{h,s,a},   V_{\mu,h+1} \rangle|\right]\\
    & - \mathbb{E}\left[|\hat{R}_{h, s, a} + \langle \hat{P}_{h, s, a}, V_{\mu,h+1}\rangle- R_{h, s, a} - \langle P_{h, s, a}, V_{\mu,h+1} \rangle| \right]\\
    & \geq \mathbb{E} \left[  \sum_{h=1}^H \left( w(h, s, a)-  \frac{\nu_h(s,a)}{l_n(h, s, a) + \nu_h(s,a) } H- \sqrt{e(h, s, a)} \right) \right]
\end{align*}

where the expectation is taken over the sequence $s = (s_1, \ldots, s_H)$. Define $d(h, s) =  \mathbb{P}(s_h = s)$ for every $h \leq H$ and $s \in \mathcal{S}$. Then the above equation can be written as
\begin{align*}
&\tilde{V}_{\mu,1}(s_{1}) -  V_{\mu,1}(s_{1}) \geq \sum_{s \in \mathcal{S}, h \leq H} d(h, s) \left( w(h, s, \mu_h(s)) -\frac{\nu_h(s,a)}{l_n(h, s, a) + \nu_h(s,a) } H- \sqrt{e(h, s, \mu_h(s))} \right)\\
&\geq \left( \sum_{s \in \mathcal{S}, h \leq H} d(h, s) w(h, s, \mu_h(s)) \right) - \sqrt{H S} \sqrt{\sum_{s \in \mathcal{S}, h \leq H} d(h, s)^2 (\sqrt{e(h, s, \mu_h(s))}+\frac{\nu_h(s,a)}{l_n(h, s, a) + \nu_h(s,a) } H)^2}\\
&:= X(w)
\end{align*}

where the second inequality applies Cauchy-Schwarz. Now, since

\[
d(h,s)W(h,s,\mu_h(s)) \sim N\left(0, d(h,s)^2\frac{\beta}{l_n(h, s, a) + \nu(s,a)}\right),
\]

we have
\begin{align*}
    X(W)  \sim N( & -\sqrt{HS \sum_{s \in \mathcal{S}, h \leq H} d(h,s)^2 (\sqrt{e(h, s, \mu_h(s))}+\frac{\nu_h(s,a)}{l_n(h, s, a) + \nu_h(s,a) } H)^2}, \\
    & HS \sum_{s \in \mathcal{S}, h \leq H} d(h,s)^2 \frac{\beta}{l_n(h, s, a) + \nu_h(s,a)} ).
\end{align*}

Then, we try to show that $\forall h,s,a$
\begin{equation}\label{equ:in}
    HS(\sqrt{e(h, s, a)}+\frac{\nu_h(s,a)}{l_n(h, s, a) + \nu_h(s,a) } H)^2 \le \frac{\beta}{l_n(h, s, a) + \nu_h(s,a)}
\end{equation}
Given the above inequality, it follows that:$\mathbb{P}(X(W) \geq 0) \ge \Phi(-1)$. Therefore, the validity of our lemma is established:$\mathbb{P} \left( \tilde{V}_{\mu,1}(s_{1}) \geq  V_{\mu,1}(s_{1}) \right) \geq \Phi(-1)$.

For equation \ref{equ:in} LHS, by a simple algebraic manipulation, we obtain:
\begin{align*}
    & (l_n(h, s, a) + \nu_h(s,a))HS(\sqrt{e(h, s, a)}+\frac{\nu_h(s,a)}{l_n(h, s, a) + \nu_h(s,a) } H)^2 \\
    & = (l_n(h, s, a) + \nu_h(s,a))(HSe(h, s, a) + 2H^2S\frac{\nu_h(s,a)}{l_n(h, s, a) + \nu_h(s,a) }\sqrt{e(h, s, a)} + H^3S\frac{\nu_h(s,a)^2}{(l_n(h, s, a) + \nu_h(s,a))^2 } )\\
    & = \frac{l_n(h, s, a) + \nu_h(s,a)}{l_n(h, s, a) +1 }H^3S \log{(2HSAn)} + 2H^3S\nu_h(s,a)\sqrt{\frac{\log{(2HSAn)}}{l_n(h, s, a) +1 }} + H^3S\frac{\nu_h(s,a)^2}{l_n(h, s, a) + \nu_h(s,a)} \\
    & \le 4 \max(1,\overline{\nu})H^3S \log{(2HSAn)} \\
    & \le \beta
\end{align*}
The second-to-last inequality is readily obtained from $\frac{l_n(h, s, a) + \nu_h(s,a)}{l_n(h, s, a) +1 } \le \max(\nu_h(s,a),1)$ and $\frac{\nu_h(s,a)}{l_n(h, s, a) + \nu_h(s,a)} \le 1$, and the last inequality is enforced by the lower bound on beta specified in the theorem. Hence, the inequality \ref{equ:in} has been proved.

\end{proof}

\begin{proof}{Proof of Lemma \ref{lem:needcore}}

Consider some history $\mathcal{H}_{n-1}$ with $\hat{M}^n \in \mathcal{M}^n$. Recall $\mu^*$ is the optimal policy in MDP $\mathcal{M}=(\mathcal{S},\mathcal{A},H,P,R,s_1)$. Applying Lemma \ref{lem:core} conditioned on $\mathcal{H}_{n-1}$ shows that with probability at least $\Phi(-1)$, $\tilde{V}_{\mu^*,1}(s_{1}) \geq  V_{\mu^*,1}(s_{1})$. When this occurs, we always have $\tilde{V}_{\mu^n,1}(s_{1}) \geq  V_{*,1}(s_{1})$, since by definition $\mu^n$ is optimal under our algorithm.
\end{proof}

\textbf{Reduction to bounding online prediction error.}  For the purposes of analysis, we let $\overline{w}$ denote an imagined second sample drawn from the same distribution as $\overline{w}(h,s,a)|\mathcal{H}_{n-1} \sim N(0,Var(w)(h,s,a)) $ under our algorithm. More formally, let $\overline{M}^n $ whose value function $\overline{V}_h(s, a) $ is estimated by our algorithm under $\overline{w}$. Conditioned on the history, $\overline{M}^n$ has the same marginal distribution as $\tilde{M}^n$, but it is statistically independent of the policy $\mu^n$ selected by RLSVI.

\begin{lemma}\label{lem:fin}
For an absolute constant $c = \Phi(-1)^{-1} < 6.31$, we have
\begin{align*}
    \text{Regret}(T,M) &\leq (c+1)\mathbb{E}\left[\sum_{n=1}^{N}| \tilde{V}_{n,1}(s_1)- V_{\mu^n,1}(s_1)|\right] + c\mathbb{E}\left[\sum_{n=1}^{N}|\overline{V}_{\mu^n,1}(s_1) - V_{\mu^n,1}(s_1)|\right]\\
    &+ H\sum_{n=1}^{N}\mathbb{P}(\hat{M}^n \notin \mathcal{M}^n).
\end{align*}

\end{lemma}

\textbf{Online prediction error bounds.} We complete the proof with concentration arguments. Set $\epsilon_R^n(h, s, a) = \hat{R}_{h, s, a}^n - R_{h, s, a} \in \mathbb{R}$ and $\epsilon_P^n(h, s, a) = \hat{P}_{h, s, a}^n - P_{h, s, a} \in \mathbb{R}^S$ to be the error in estimating the mean reward and transition vector corresponding to $(h, s, a)$. The next result follows by bounding each term in Lemma~6. We focus our analysis on bounding $\mathbb{E}\left[\sum_{n=1}^{N}| \tilde{V}_{n,1}(s_1)- V_{\mu^n,1}(s_1)|\right]$. The other term can be bounded in an identical manner, so we omit this analysis.

\begin{lemma}\label{lem:vdif}
Let $c = \Phi(-1)^{-1} < 6.31$. Then for any $N \in \mathbb{N}$,
\begin{align*}
    &\mathbb{E} \left[ \sum_{n=1}^{N} | \tilde{V}_{n,1}(s_1)- V_{\mu^n,1}(s_1)| \right] \leq \sqrt{\mathbb{E} \left[ \sum_{n=1}^{N} \sum_{h=1}^{H-1} ||\epsilon_P^n(h, s_{nh}, a_{nh})||_1^2 \right]} \sqrt{\mathbb{E} \left[ \sum_{n=1}^{N} \sum_{h=1}^{H-1} ||\tilde{V}_{n,h+1}||_{\infty}^2 \right]}\\
    &+ \mathbb{E} \left[ \sum_{n=1}^{N} \sum_{h=1}^{H} |\epsilon_R^n(h, s_{nh}, a_{nh})| \right] + \mathbb{E} \left[ \sum_{n=1}^{N} \sum_{h=1}^{H} |w^n(h, s_{nh}, a_{nh})| \right] + \mathbb{E}\left[ \sum_{h=1}^H \frac{\overline{\nu}}{l_n(h, s_{h}, a_{h}) + \overline{\nu} }H \right].
\end{align*}

\end{lemma}

The remaining lemmas complete the proof. At each stage, RLSVI adds Gaussian noise with standard deviation no larger than $\tilde{O}(H^{3/2}\sqrt{S})$. Ignoring extremely low probability events, we expect,
\[
\|\tilde{V}_{n,h+1}\|_{\infty} \leq \tilde{O}(H^{5/2}\sqrt{S}) \text{ and hence } \sum_{h=1}^{H-1} \|\tilde{V}_{n,h+1}\|_{\infty}^2 \leq \tilde{O}(H^6S).
\]
The proof of this Lemma makes this precise by applying appropriate maximal inequalities.

\begin{proof}{Proof of lemma \ref{lem:vdif}}
We bound each term in the bound in Lemma \ref{lem:vdif}. By applying Lemma \ref{lem:value_diff} with a choice of $\tilde{M}=M$ and $\hat{M}=\tilde{M}^n$, the largest term is bounded, for any $k\in\mathbb{N}$, and reference to the proof of Lemma \ref{lem:core}, we have
\begin{align*}
\left|\tilde{V}_{n,1}(s_1)- V_{\mu^n,1}(s_1)\right|
&\leq \mathbb{E} \left[  \sum_{h=1}^H |w^n(h, s_{nh}, a_{nh})|\right] \\
& + \mathbb{E}\left[ \sum_{h=1}^H \frac{\nu_h(s_{nh},a_{nh})}{l_n(h, s_{h}, a_{h}) + \nu_h(s_{nh},a_{nh}) }|\theta_h^*(s_{nh},a_{nh})-\hat{R}_{h,s_{nh},a_{nh}} + \langle \hat{P}_{h,s_{nh},a_{nh}},   \tilde{V}_{n,h+1} \rangle|\right]\\
& + \mathbb{E}\left[\sum_{h=1}^H|\hat{R}_{h, s_{nh}, a_{nh}} + \langle \hat{P}_{h, s_{nh}, a_{nh}}, \tilde{V}_{\mu,h+1}\rangle- R_{h, s_{nh}, a_{nh}} - \langle P_{h, s_{nh}, a_{nh}}, \tilde{V}_{n,h+1} \rangle| \right]\\
&\leq \mathbb{E} \left[  \sum_{h=1}^H |w^n(h, s_{nh}, a_{nh})|\right] + \mathbb{E}\left[ \sum_{h=1}^H \frac{\nu_h(s_{nh},a_{nh})}{l_n(h, s_{h}, a_{h}) + \nu_h(s_{nh},a_{nh}) }H \right] \\
& + \mathbb{E}\left[\sum_{h=1}^{H-1}\left\|e_P^n(h,s_{nh},a_{nh})\right\|_1\left\|V_{\mu^n,h+1}\right\|_{\infty}\right]+\mathbb{E}\left[\sum_{h=1}^H |e_R^n(h,s_{nh},a_{nh})|\right]
\end{align*}
\end{proof}

\begin{lemma}\label{lem:need0}
\[
\mathbb{E} \left[ \sum_{n=1}^{N} \sum_{h=1}^{H-1} \|\tilde{V}_{n,h+1}\|_{\infty}^2 \right] = \tilde{O} \left( H^3 \sqrt{SN} \right)
\]
\end{lemma}

The next few lemmas are essentially a consequence of analysis in \cite{osband2013more,osband2016generalization}, and many subsequent papers. We give proof sketches in the appendix. The main idea is to apply known concentration

inequalities to bound $\|\epsilon_P^n(h, s, a)\|_1^2$, $|\epsilon_R^n(h, s_{nh}, a_{nh})|$ or $|w^n(h, s_{nh}, a_{nh})|$ in terms of either $1/l_n(h, s_h, a_h)$ or $1/\sqrt{l_n(h, s_h, a_h)}$. The pigeonhole principle gives $\sum_{n=1}^N \sum_{h=1}^{H} 1/l_n(h, s_h, a_h) = O(\log(SA{NH}))$, $\sum_{n=1}^N \sum_{h=1}^{H} \overline{\nu}/(l_n(h, s_h, a_h)+\overline{\nu}) = O(\overline{\nu}\log(SA{NH}))$ and $\sum_{n=1}^N \sum_{h=1}^{H} (1/\sqrt{l_n(h, s_h, a_h)}) = O(\sqrt{SA{NH}})$.

\begin{lemma}\label{lem:need1}
\[
\mathbb{E} \left[ \sum_{n=1}^N \sum_{h=1}^{H} \|\epsilon_P^n(h, s, a)\|_1^2 \right] = \tilde{O}(S^2AH)
\]
\end{lemma}

\begin{lemma}\label{lem:need2}
\[
\mathbb{E} \left[ \sum_{n=1}^N \sum_{h=1}^{H} |\epsilon_R^n(h, s_{nh}, a_{nh})| \right] = \tilde{O}\left( \sqrt{SANH} \right)
\]
\end{lemma}

\begin{lemma}\label{lem:need3}
\[
\mathbb{E} \left[ \sum_{n=1}^N \sum_{h=1}^{H} |w^n(h, s_{nh}, a_{nh})| \right] = \tilde{O}\left( H^{3/2}S\sqrt{ANH} \right)
\]
\end{lemma}
The detail proof of Lemma \ref{lem:need0}, \ref{lem:need1},\ref{lem:need2} and \ref{lem:need3} can be found in lemma 8, 9, 10 and 11 of paper \cite{russo2019worst}. And then we plug lemma \ref{lem:vdif}, \ref{lem:need0}, \ref{lem:need1},\ref{lem:need2} and \ref{lem:need3} in \ref{lem:fin}, than we get the regret bound.

\section{Explanation of Algorithm \ref{alg:TSRL+}'s exploration periods}\label{append:proe}
We first state the following lemma.

\begin{lemma}\label{lem:tau}
For any MDP epoch $k \in [K]$, the length of the random exploration periods $\mathcal{N}_k $ is upper bounded by $\mathcal{N}_e = \frac{\lambda_e}{\lambda_0}$.
\end{lemma}

In other words, we incur at most logarithmic regret due to the initial random exploration in Algorithm 1.



\begin{proof}{Proof of lemma \ref{lem:tau}}
Recall that $V^{(k)}_{nh} = \sum_{i=1}^{n-1}\Phi_{h}^\top(s^{(k)}_{ih},a^{(k)}_{ih})\Phi_{h}(s^{(k)}_{ih},a^{(k)}_{ih})$ is the Fisher information matrix of MDP epoch $j$ after $n$ episode. 




\begin{lemma}\label{lem:lamb}
For all $n \leq \mathcal{N}_k $, the minimum eigenvalue of $V_{i,t}$ is lower bounded as
\begin{align*}
    \lambda_{min}(V^{(k)}_{nh}) \ge \lambda_{0}(n-1), \forall j,h.
\end{align*}
\end{lemma}

Because we have $\min_{h,s,a} \lambda_{min}(\Phi_{h}^\top(s,a)\Phi_{h}(s,a))\ge\lambda_{0}$ from the assumption, it's obvious that $\Phi_{h}^\top(s,a)\Phi_{h}(s,a)\succeq \lambda_{0}\mathbf{I}$, so we have $V^{(k)}_{nh} \succeq \lambda_{0}(n-1)\mathbf{I}$.
It means $\lambda_{min}(V^{(k)}_{nh}) \ge \lambda_{0}(n-1)$.

Then using \ref{lem:lamb}, we know that after at most $\frac{\lambda_e}{\lambda_0}$ episode, we have $\lambda_{\min}(V_{nh}^{(k)}) \geq \lambda_e, \forall j,h$.
\end{proof}

\section{Proof of Theorem \ref{thm:MTSRL}}\label{append:proMSTRL}
We begin by defining some helpful notation. First, let
\begin{align*}
\text{REV} \left( \{\theta_h^{(k)}\}, \{\hat{\theta}_h^{(k)}\}, \{\Sigma_h^{(k)}\},N\right) = \sum_{n=1}^{N}\mathbb{E} \left[ \sum_{h=1}^{H} r_{nh}^{(k)} \right],
\end{align*}

be the expected total reward obtained by running $\text{TSRL}^{+}(\{\hat{\theta}_h^{(k)}\},  \{\Sigma_h^{(k)}\}, \lambda_{e}=0,N)$ — the Thompson sampling algorithm in Algorithm \ref{alg:TSRL+} with the (possibly incorrect) prior $\left( \{\hat{\theta}_h^{(k)}\} , \{\Sigma_h^{(k)}\} \right)$ and exploration parameter $\lambda_e = 0$ — in a MDP epoch with true parameter $\{\theta_h^{(k)}\}$. Second, let
\begin{align*}
\text{REV}_* (\{\theta_h^{(k)}\},N) = \sum_{n=1}^{N}\mathbb{E} \left[ \sum_{h=1}^{H} V_{*,1}(s_{n1}^{(k)}) \right],
\end{align*}

be the expected value over $n$ time steps obtained by the oracle — in a MDP epoch with true parameter $\{\theta_h^{(k)}\}$. And at last, We define $\beta$ as the constant perturbation variance parameter, selected as in \ref{sec:PEKP}, for episode $n$ of MDP $k$ with subscripts omitted for brevity.

\subsection{``Prior Alignment'' Proof Strategy}

In each non-exploration MDP epoch $k > K_0$, the meta oracle starts with the true prior $(\{\theta_h^{*}\}, \{\Sigma^{*}\})$ while our algorithm MTSRL starts with the estimated prior $(\{\hat{\theta}_h^{(k)}\}, \{\Sigma^{*}\})$. The following lemma bounds the error of the estimated prior mean with high probability:

\begin{lemma}\label{lem:theta}
For any fixed $j \geq 2$ and $\delta \in [0, 2/e]$, if $\lambda_{max}(\Sigma_h^{*}) \le \overline{\lambda}$, then with probability at least $1 - \delta$,
\begin{align*}
\left\| \hat{\theta}_h^{(k)}- \theta_h^{*} \right\| \leq 8\sqrt{\frac{M(\beta/\lambda_e + 5\bar{\lambda})\log_e(2M/\delta)}{k}}.
\end{align*}
\end{lemma}

We will first proof this lemma.
\begin{proof}{Proof of lemma \ref{lem:theta}}
Lemma~\ref{lem:theta} establishes that after observing $j$ MDP epochs of length $n$, our estimator $\hat{\theta}_h^{(k)}$ of the unknown prior mean $\theta^*_h$ is close to it with high probability. 
To prove Lemma~\ref{lem:theta}, we first demonstrate that, by the end of each MDP epoch, the estimated parameter vector $\dot{\theta}_h^{(k)}$ is likely close to the true parameter vector $\theta_h^{(k)}$ (Lemma~\ref{lem:j1}). 
This result implies that the empirical average $\frac{1}{k-1}\sum_{i=1}^{k-1}\dot{\theta}_h^{(i)}$ is also close to the average of the true parameters $\frac{1}{k-1}\sum_{i=1}^{k-1}\theta_h^{(i)}$ (Lemma~\ref{lem:j3}). 
We then show that this latter average serves as a good approximation of the true prior mean $\theta^*_h$ (Lemma~\ref{lem:j4}). 
Combining these results through a triangle inequality completes the proof of Lemma~\ref{lem:theta}.

We first state two useful lemmas from the literature regarding the concentration of OLS estimates and the matrix Hoeffding bound.

\begin{lemma}\label{lem:j1}
For any MDP epoch $k \in [K]$ and $\delta \in [0,2/e]$, conditional on $F_{hk} = \sigma(\dot{\theta}_h^{(1)}, \ldots, \dot{\theta}_h^{(k-1)})$, we have
\begin{align*}
\Pr \left( \| \dot{\theta}_h^{(k)} - \theta_h^{(k)} \| \geq 2 \sqrt{\frac{M\beta\log_e(2/\delta)}{\lambda_e}} \biggm| F_{hk} \right) \leq \delta,
\end{align*}
\end{lemma}

\textbf{Proof of Lemma \ref{lem:j1}:} When $h=H$, this result follows from Theorem 4.1 in bandit scenario of \cite{zhu2018learning}, where we note that $K/2 + \log_e(2/\delta) \leq K\log_e(2/\delta)$ for $\delta < 2/e$. By the situation $h<H$, this result can be obtained by simple iteration.

\begin{lemma}[\cite{jin2019short}]\label{lem:j2}
Let random vectors $X_1, \ldots, X_K \in \mathbb{R}^M$, satisfy that for all $k \in [K]$ and $u \in \mathbb{R}$,
\begin{align*}
\mathbb{E}[X_k \mid \sigma(X_1, \ldots, X_{k-1})] = 0, \quad \Pr\left( \|X_k\| \geq u \mid \sigma(X_1, \ldots, X_{k-1}) \right) \leq 2 \exp\left(-\frac{u^2}{2\sigma_k^2}\right),
\end{align*}

then for any $\delta > 0$,
\begin{align*}
\Pr\left( \left\| \sum_{k \in [K]} X_k \right\| \leq 4 \sqrt{\sum_{k \in [K]} \sigma_k^2 \log_e (2M/\delta)} \right) \geq 1 - \delta.
\end{align*}
\end{lemma}

We now show that the average of our estimated parameters from each epoch is close to the average of the true parameters from each epoch with high probability.

\begin{lemma}\label{lem:j3}
For any $k \geq 2$, the following holds with probability at least $1 - \delta$:
\begin{align*}
\left\| \frac{1}{k-1} \sum_{i=1}^{k-1} \left( \dot{\theta}_h^{(i)} - \theta_h^{(i)} \right) \right\| \leq 4  \sqrt{\frac{2\beta M \log_e (2M/\delta)}{\lambda_e (k-1)}}.
\end{align*}
\end{lemma}

\textbf{Proof of Lemma \ref{lem:j3}.} By Lemma \ref{lem:j1}, we have for any $u \in \mathbb{R}$,
\begin{align*}
\Pr(\|\dot{\theta}_h^{(k)} - \theta_h^{(k)}\| \geq u \mid F_{hk}) \leq 2 \exp(-\lambda_e u^2 / 4M\beta).
\end{align*}

Furthermore, since the OLS estimator is unbiased, $\mathbb{E}[\dot{\theta}_h^{(k)} \mid F_{hk}] = \theta_h^{(k)}$. Thus, we can apply the matrix Hoeffding inequality (Lemma \ref{lem:j2}) to obtain
\begin{align*}
\Pr\left( \left\| \frac{1}{k-1} \sum_{i=1}^{k-1} (\dot{\theta}_h^{(i)} - \theta_h^{(i)}) \right\| \leq 4\sqrt{\frac{2\beta M \log_e (2M/\delta)}{\lambda_e (k-1)}} \right) \geq 1-\delta.
\end{align*}

Noting that $k \leq 2(k-1)$ for all $k \in \{2, \ldots, K\}$ concludes the proof. $\square$

\begin{lemma}\label{lem:j4}
For any $k \geq 2$, the following holds with probability at least $1-\delta$:
\begin{align*}
\left\| \frac{1}{k-1} \sum_{i=1}^{k-1} \theta_h^{(i)} - \theta_h^{*} \right\| \leq 4\sqrt{\frac{10 \bar{\lambda} M \log_e (2M/\delta)}{k-1}}.
\end{align*}
\end{lemma}

\textbf{Proof of Lemma \ref{lem:j4}.} We first show a concentration inequality for the quantity $\|\theta_h^{(k)} - \theta_h^{*}\|$ similar to that of Lemma \ref{lem:j1}. Note that for any unit vector $s \in \mathbb{R}^{M}$, $s^\top (\theta_h^{(k)} - \theta_h^{*})$ is a zero-mean normal random variable with variance at most $\bar{\lambda}$. Therefore, for any $u \in \mathbb{R}$,
\begin{equation}\label{equ:22}
\Pr \left( |s^\top (\theta_h^{(k)} - \theta_h^{*})| \geq u \right) \leq 2 \exp \left( -\frac{u^2}{2\bar{\lambda}} \right).
\end{equation}

Consider $W$, a $(1/2)$-cover of the unit ball in $\mathbb{R}^{M}$. We know that $|W| \leq 4^{M}$. Let $s(\theta_h^{(k)}) = (\theta_h^{(k)} - \theta_h^{*}) / \|\theta_h^{(k)} - \theta_h^{*}\|$, then there exists $w_{s(\theta_h^{(k)})} \in W$, such that $\|w_{s(\theta_h^{(k)})} - s(\theta_h^{(k)})\| \leq 1/2$ by definition of $W$. Hence,
\begin{align*}
\|\theta_h^{(k)} - \theta_h^{*}\| &= \langle s(\theta_h^{(k)}), \theta_h^{(k)} - \theta_h^{*} \rangle \\
&= \langle s(\theta_h^{(k)}) - w_{s(\theta_h^{(k)})}, \theta_h^{(k)} - \theta_h^{*} \rangle + \langle w_{s(\theta_h^{(k)})}, \theta_h^{(k)} - \theta_h^{*} \rangle \\
&\leq \frac{\|\theta_h^{(k)} - \theta_h^{*}\|}{2} + \langle w_{s(\theta_h^{(k)})}, \theta_h^{(k)} - \theta_h^{*} \rangle.
\end{align*}

Rearranging the terms yields
\begin{align*}
\|\theta_h^{(k)} - \theta_h^{*}\| \leq 2 \langle w_{s(\theta_h^{(k)})}, \theta_h^{(k)} - \theta_h^{*} \rangle.
\end{align*}

Applying an union bound to all possible $w \in W$ with inequality \ref{equ:22}, we have for any $u \in \mathbb{R}$,
\begin{align*}
\Pr(\|\theta_h^{(k)} - \theta_h^{*}\| \geq u) &\leq \Pr(\exists w \in W : \langle w, \theta_h^{(k)} - \theta_h^{*} \rangle \geq u/2) \\
&\leq 2 \cdot 4^{M} \exp\left(-\frac{u^2}{2\bar{\lambda}}\right) \\
&\leq \exp\left(\frac{5M}{2} - \frac{u^2}{2\bar{\lambda}}\right).
\end{align*}

If $u^2 \leq 5\bar{\lambda}M$, we have
\begin{align*}
\Pr(\|\theta_h^{(k)} - \theta_h^{*}\| \geq u) \leq 1 \leq 2 \exp\left(-\frac{u^2}{10\bar{\lambda}M}\right);
\end{align*}

else if $u^2 = 5\bar{\lambda}M + v$ for some $v \geq 0$, we have
\begin{align*}
\Pr(\|\theta_h^{(k)} - \theta_h^{*}\| \geq u) &\leq \exp\left(-\frac{v}{2\bar{\lambda}}\right) \\
&\leq 2 \exp\left(-\frac{u^2}{10\bar{\lambda}M}\right).
\end{align*}

Thus, for any $u \in \mathbb{R}$, we can write  
\begin{equation}\label{ali:23}
    \Pr(\|\theta_h^{(k)} - \theta_h^{*}\| \geq u) \leq 2 \exp\left(-\frac{u^2}{10 \bar{\lambda} M}\right).
\end{equation}

Applying Lemma \ref{lem:j2}, we have  
\begin{align*}
\Pr\left(\left\|\frac{\sum_{i=1}^{k-1} \theta_h^{(i)}}{k-1} - \theta_h^{*}\right\| \leq 4\sqrt{\frac{5 \bar{\lambda} M \log_e (2M/\delta)}{k-1}}\right) \geq 1 - \delta.
\end{align*}

The proof can be concluded by the observation $k \leq 2(k-1)$ for all $k \in \{2,\ldots,K\}$.

We can now combine Lemmas \ref{lem:j3} and \ref{lem:j4} to prove Lemma \ref{lem:theta}.

\textbf{Proof of Lemma \ref{lem:theta}.} We can use the triangle inequality and a union bound over Lemmas 9 and 10 to obtain
\begin{align*}
\left\|\dot{\theta}_h^{(k)}-\theta_h^{*}\right\| &= \left\|\frac{\sum_{i=1}^{k-1}\dot{\theta}_h^{(i)}}{k-1}-\frac{\sum_{i=1}^{k-1}\theta_h^{(i)}}{k-1}+\frac{\sum_{i=1}^{k-1}\theta_h^{(i)}}{k-1}-\theta_h^{*}\right\| \\
&\leq \left\|\frac{1}{k-1}\sum_{i=1}^{k-1}\left(\dot{\theta}_h^{(i)}-\theta_h^{(i)}\right)\right\| + \left\|\frac{1}{k-1}\sum_{i=1}^{k-1}\theta_h^{(i)}-\theta_h^{*}\right\| \\
&\leq 8\sqrt{\frac{(\beta/\lambda_{e}+5\bar{\lambda})M\log_{e}(2MK/\delta)}{k}},
\end{align*}

with probability at least $1-2\delta$, where we have used the fact that $\sqrt{a}+\sqrt{b}\leq\sqrt{2(a+b)}$. Thus, a second union bound yields the result.

\end{proof}

\begin{lemma}\label{lem:prieq}
Conditioned on $\{\theta_h^*\}$, the posteriors of the meta oracle and our algorithm MTSRL algorithm satisfy
\begin{align*}
\theta_{\mathcal{N}_k+1,h}^{\text{TS}(k)} - \theta_{\mathcal{N}_k+1,h}^{\text{MT}(k)} &= \left( \frac{1}{\beta_{\mathcal{N}_k +1}} \sum_{i=1}^{\mathcal{N}_k }\Phi_{h}^\top(s_{ih},a_{ih})\Phi_{h}(s_{ih},a_{ih}) +  \Sigma_h^{*-1} \right)^{-1} \\
&\left( \Sigma_h^{*-1} \left( \theta_h^* - \hat{\theta}_h^{(k)} \right) + \frac{1}{\beta_{\mathcal{N}_k }} \sum_{i=1}^{\mathcal{N}_k} \Phi_{h}^\top(s_{ih},a_{ih}) \left( z_{ih}^{\text{TS}(k)} - z_{ih}^{\text{MT}(k)} \right) \right), \\
\Sigma_{\mathcal{N}_k+1,h}^{\text{TS}(k)} &=\Sigma_{\mathcal{N}_k+1,h}^{\text{MT}(k)}.
\end{align*}
\end{lemma}

Now consider any non-exploration MDP epoch $k \geq K_0 + 1$. 
Suppose that, upon completing all exploration steps at time $\mathcal{N}_k + 1$, 
the posteriors of the meta-oracle and our MTSRL algorithm are identical, i.e.,
$(\theta_{\mathcal{N}_k+1,h}^{\text{MT}(k)}, \Sigma_{\mathcal{N}_k+1,h}^{\text{MT}(k)}) 
= (\theta_{\mathcal{N}_k+1,h}^{\text{TS}(k)}, \Sigma_{\mathcal{N}_k+1,h}^{\text{TS}(k)})$. 
In this case, both policies would achieve identical expected rewards over the remaining time periods 
$\mathcal{N}_k + 1, \ldots, N$. 
Lemma~\ref{lem:prieq} guarantees that 
$\Sigma_{\mathcal{N}_k+1,h}^{\text{TS}(k)} = \Sigma_{\mathcal{N}_k+1,h}^{\text{MT}(k)}$ always holds; 
thus, the only condition left to verify is when 
$\theta_{h}^{\text{TS}(k)} = \theta_{h}^{\text{MT}(k)}$.

Since the two algorithms begin with different priors but encounter the same covariates $\{\Phi_{h}(s_{ih},a_{ih})\}_{i=1}^{\mathcal{N}_k}$ , their posteriors can only align at time $\mathcal{N}_k + 1$ due to the stochasticity in the observations $z_{ih}^{(k)}$. For convenience, denote the noise terms from $i \in \{1, \cdots, \mathcal{N}_k\}$ of the meta oracle and the MTSRL algorithm respectively as
\begin{align}
\chi_{h}^{\text{TS}(k)} &= (z_{1h}^{\text{TS}(k)}, \cdots, z_{\mathcal{N}_k,h}^{\text{TS}(k)})^{\top},  \label{ali:chi1}\\
\chi_{h}^{\text{MT}(k)} &= (z_{1h}^{\text{MT}(k)}, \cdots, z_{\mathcal{N}_k,h}^{\text{MT}(k)})^{\top}. \label{ali:chi2}
\end{align}

Furthermore, let $\Phi_h^{(k)} = (\Phi_{h}^\top(s_{1h}^{(k)},a_{1h}^{(k)}), \ldots, (\Phi_{h}^\top(s_{\mathcal{N}_kh}^{(k)},a_{\mathcal{N}_kh}^{(k)})) \in \mathbb{R}^{M \times \mathcal{N}_k}$. Lemma \ref{lem:prieq} indicates that if
\begin{equation}\label{equ:err}
    \chi_{h}^{\text{MT}(k)}  - \chi_{h}^{\text{TS}(k)} =  \beta_{\mathcal{N}_k } (\Phi_h^{(k)\top} \Phi_h^{(k)})^{-1} \Phi_h^{(k)\top} \Sigma_h^{*-1} \left( \theta^{*}_h - \hat{\theta}_{h}^{(k)} \right), 
\end{equation}

Recall that for any $n \in \{\mathcal{N}_k + 1, \cdots, N\}$, the meta oracle maintains and samples from its posterior $\{\theta_{nh}^{\text{TS}(k)}\}, \{\Sigma_{nh}^{\text{TS}(k)}\}$ (see Algorithm 1), while our MTSRL algorithm maintains and samples parameters from its posterior $\{\theta_{h}^{\text{MT}(k)}\}, \{\Sigma_{h}^{\text{MT}(k)}\}$ . 
The proof follows from the standard update rules for Bayesian linear regression and is given below.

\textbf{Proof of Lemma \ref{lem:prieq}.} Using the posterior update rule for Bayesian linear regression (Bishop 2006), the posterior of the oracle at $n = \mathcal{N}_k + 1$ is
\begin{align*}
\theta_{\mathcal{N}_k +1,h}^{\text{TS}(k)} &= \left( \frac{1}{\beta_{\mathcal{N}_k +1}} \sum_{i=1}^{\mathcal{N}_k }\Phi_{h}^\top(s_{ih},a_{ih})\Phi_{h}(s_{ih},a_{ih}) + \Sigma_h^{*-1} \right)^{-1} (\frac{1}{\beta_{\mathcal{N}_k +1}}\sum_{i=1}^{\mathcal{N}_k }\Phi_{h}^\top(s_{ih},a_{ih})b^{\text{TS}(k)}_{ih} + \Sigma_h^{*-1}\theta^*_h),
\end{align*}

\begin{align*}
\Sigma_{\mathcal{N}_k +1,h}^{\text{TS}(k)} &= \left( \frac{1}{\beta_{\mathcal{N}_k +1}} \sum_{i=1}^{\mathcal{N}_k }\Phi_{h}^\top(s_{ih},a_{ih})\Phi_{h}(s_{ih},a_{ih}) +  \Sigma_h^{*-1} \right)^{-1}.
\end{align*}

Similarly, the posterior of the MTSRL algorithm at $n = \mathcal{N}_k + 1$ is
\begin{align*}
\theta_{\mathcal{N}_k +1,h}^{\text{MT}(k)} &= \left( \frac{1}{\beta_{\mathcal{N}_k +1}} \sum_{i=1}^{\mathcal{N}_k }\Phi_{h}^\top(s_{ih},a_{ih})\Phi_{h}(s_{ih},a_{ih}) + \Sigma_h^{*-1} \right)^{-1} (\frac{1}{\beta_{\mathcal{N}_k +1}}\sum_{i=1}^{\mathcal{N}_k }\Phi_{h}^\top(s_{ih},a_{ih})b^{\text{MT}(k)}_{ih} + \Sigma_h^{*-1}\hat{\theta}^{(k)}_h),
\end{align*}

\begin{align*}
\Sigma_{\mathcal{N}_k +1,h}^{\text{MT}(k)} &= \left( \frac{1}{\beta_{\mathcal{N}_k +1}} \sum_{i=1}^{\mathcal{N}_k }\Phi_{h}^\top(s_{ih},a_{ih})\Phi_{h}(s_{ih},a_{ih}) +  \Sigma_h^{*-1} \right)^{-1}.
\end{align*}

And we know from Appendix \ref{sec:PEKP} that $b^{\text{TS}(k)}_{ih} - b^{\text{MT}(k)}_{ih} = z^{\text{TS}(k)}_{ih} - z^{\text{MT}(k)}_{ih}$, when $\theta_{\mathcal{N}_k +1,h+1}^{\text{TS}(k)} = \theta_{\mathcal{N}_k +1,h+1}^{\text{MT}(k)}$.
The result follows directly. 

We also note that the prior-independent Thompson sampling algorithm employed in the exploration epochs satisfies a meta regret guarantee:

\begin{lemma}\label{lem:inde}
The meta regret of the prior-independent Thompson sampling algorithm(RLSVI) in a single MDP epoch is $\widetilde{O}\left( H^{3}S^{3/2}\sqrt{AN} \right)$.
\end{lemma}

The proof can be easily adapted from the literature (\cite{russo2019worst}), and is thus omitted. Lemma \ref{lem:inde} ensures that we accrue at most $\widetilde{O}\left(K_{0} H^{3}S^{3/2}\sqrt{AN} \right)$ regret in the $K_{0}$ exploration MDP epochs; from lemma \ref{lem:tau}, we know that $K_{0}$ grows merely $\widetilde{O}(1)$.

\subsection{Details for the proof of Theorem \ref{thm:MTSRL}} 

Consider any non-exploration epoch $k \geq K_0 + 1$. If upon completion of all exploration steps at time $\mathcal{N}_k + 1$, we have that the posteriors of the meta oracle and our MSTRL algorithm coincide — i.e., $(\theta_{\mathcal{N}_k +1,h}^{\text{MT}(k)}, \Sigma_{\mathcal{N}_k +1,h}^{\text{MT}(k)}) = (\theta_{\mathcal{N}_k +1,h}^{\text{TS}(k)}, \Sigma_{\mathcal{N}_k +1,h}^{\text{TS}(k)})$ — then both policies would achieve the same expected revenue over the time periods $\mathcal{N}_k + 1, \cdots, N$, i.e., we would have
\begin{align*}
\text{REV} \left( \{\theta_h^{(k)}\}, \{\theta_{\mathcal{N}_k +1,h}^{\text{MT}(k)}\}, \{\Sigma_{\mathcal{N}_k +1,h}^{\text{MT}(k)}\},N - \mathcal{N}_k \right) = \text{REV} \left( \{\theta_h^{(k)}\}, \{\theta_{\mathcal{N}_k +1,h}^{\text{TS}(k)}\}, \{\Sigma_{\mathcal{N}_k +1,h}^{\text{TS}(k)}\},N - \mathcal{N}_k \right).
\end{align*}

By Lemma \ref{lem:prieq}, we know that $\Sigma_{\mathcal{N}_k +1,h}^{\text{TS}(k)} = \Sigma_{\mathcal{N}_k +1,h}^{\text{MT}(k)}$ always, so all that remains is establishing when $\theta_{\mathcal{N}_k +1,h}^{\text{TS}(k)} = \theta_{\mathcal{N}_k +1,h}^{\text{MT}(k)}$.

Since the two algorithms begin with different priors but encounter the same covariates and take the same decisions in $n \in \{ 1, \cdots, \mathcal{N}_k \}$, their posteriors can only align at time $\mathcal{N}_k + 1$ due to the stochasticity in the error we introduced. As shown in Eq. \ref{equ:err}, alignment occurs with $\theta_{\mathcal{N}_k +1,h}^{\text{TS}(k)} = \theta_{\mathcal{N}_k +1,h}^{\text{MT}(k)}$ if
\begin{align*}
\chi_{h}^{\text{MT}(k)}  - \chi_{h}^{\text{TS}(k)} =  \beta_{\mathcal{N}_k } (\Phi_h^{(k)\top} \Phi_h^{(k)})^{-1} \Phi_h^{(k)\top} \Sigma_h^{*-1} \left( \theta^{*}_h - \hat{\theta}_{h}^{(k)} \right), 
\end{align*}

where we recall $\chi_{h}^{\text{MT}(k)}, \chi_{h}^{\text{TS}(k)}$ were defined in Eqs. \ref{ali:chi1} and \ref{ali:chi2}.

Now, we start by defining the clean event
\begin{equation}\label{equ:E}
    \mathcal{E} = \left\{ \left\| \hat{\theta}_h^{(k)}- \theta_h^{*} \right\| \leq 8\sqrt{\frac{M(\beta/\lambda_e + 5\bar{\lambda})\log_e(2M/\delta)}{k}}, \quad \mathcal{N}_k \leq \mathcal{N}_e \quad \forall k \geq K_0 + 1,h \right\},
\end{equation}

which stipulates that for every epoch $k$ following the initial $K_0$ exploration epochs: 
(i) the estimated prior mean $\hat{\theta}_h^{(k)}$ is close to the true prior mean $\theta_h^{*}$ 
(with high probability, as guaranteed by Lemma~\ref{lem:theta}); and 
(ii) Lemma~\ref{lem:tau} holds, ensuring that each epoch contains only a small number of exploration periods. 
Since $\mathcal{E}$ occurs with high probability, we begin by analyzing the meta-regret conditioned on $\mathcal{E}$.

Let $R_{K,N}(n)\mid \mathcal{E}$ denote the meta-regret of MDP epoch $n$ conditioned on the event $\mathcal{E}$ defined in Eq.~\ref{equ:E}. 
The following lemma provides an upper bound on the meta-regret for any epoch $n \geq K_0$ under this event $\mathcal{E}$.

\begin{lemma}\label{lem:Rj}
The meta regret of an epoch $n \geq K_0 + 1$ satisfies
\begin{align*}
R_{K,N}(n) \mid \mathcal{E} = \tilde{O} \left( H^{4}S^{3/2}A^{1/2}N^{1/2} \sqrt{\frac{1}{n}} + \frac{H^2}{K} \right).
\end{align*}
Here: \begin{equation}\label{equ:5}
    K_{0} = 4c_{1}^{2}H^2M\mathcal{N}_{e}^{2}\log_{e}(2MK^{2}N)\log_{e}(2KN), 
\end{equation}

where $\mathcal{N}_{e}=\frac{\lambda_e}{\lambda_0}= \tilde{O}(1)$ ($\mathcal{N}_{e}$ is a upper bound on all $\mathcal{N}_{k}$'s, see Lemma \ref{lem:tau} in Appendix), and the constant is given by
\begin{align*}
c_{1}=\frac{32\sqrt{\Phi_{max}\beta(\beta\lambda_{e}^{-1}+5\bar{\lambda})}}{\lambda_{e}\underline{\lambda}}.
\end{align*}
\end{lemma}

\textbf{Proof of \ref{lem:Rj}.} As noted earlier, during the exploration espisodes $1 \leq n \leq \mathcal{N}_k $, the meta oracle and our MTSRL algorithm encounter the same covariates; thus, by construction, they achieve the same reward and the resulting meta regret is 0. Then, we can write
\begin{equation}\label{ali:25}
\begin{aligned}
 &R_{K,N}(n) \mid \mathcal{E} = \mathbb{E}_{\{\theta_h^{(k)}\}, \{\hat{\theta}_h^{(k)}\}, \{\chi_{h}^{\text{TS}(k)}\}, \{\chi_{h}^{\text{MT}(k)}\}} [ \text{REV} \left( \{\theta_h^{(k)}\}, \{\theta_{\mathcal{N}_k +1,h}^{\text{TS}(k)}\}, \{\Sigma_{\mathcal{N}_k +1,h}^{\text{TS}(k)}\},N - \mathcal{N}_k \right)- \\
&\text{REV} \left( \{\theta_h^{(k)}\}, \{\theta_{\mathcal{N}_k +1,h}^{\text{MT}(k)}\}, \{\Sigma_{\mathcal{N}_k +1,h}^{\text{MT}(k)}\},N - \mathcal{N}_k \right) \mid \mathcal{E} ] \\
&= \mathbb{E}_{\{\theta_h^{(k)}\}, \{\hat{\theta}_h^{(k)}\}, \{\chi_{h}^{\text{MT}(k)}}\} [ \text{REV}_* \left( \{\theta_h^{(k)}\},N - \mathcal{N}_k \right) - \text{REV} \left( \{\theta_h^{(k)}\}, \{\theta_{\mathcal{N}_k +1,h}^{\text{MT}(k)}\}, \{\Sigma_{\mathcal{N}_k +1,h}^{\text{MT}(k)}\},N - \mathcal{N}_k \right) \mid \mathcal{E}] \\
&- \mathbb{E}_{\{\theta_h^{(k)}\}, \{\hat{\theta}_h^{(k)}\}, \{\chi_{h}^{\text{TS}(k)}\}} [ \text{REV}_* \left( \{\theta_h^{(k)}\},N - \mathcal{N}_k \right) - \text{REV} \left( \{\theta_h^{(k)}\}, \{\theta_{\mathcal{N}_k +1,h}^{\text{TS}(k)}\}, \{\Sigma_{\mathcal{N}_k +1,h}^{\text{TS}(k)}\},N - \mathcal{N}_k \right) \mid \mathcal{E} ].
\end{aligned}
\end{equation}

We will use our prior alignment technique to express the first term in Eq. \ref{ali:25} in terms of the second term in Eq. \ref{ali:25}; in other words, we will use a change of measure suggested by Eq.  \ref{equ:err} to express the true regret of our MTSRL algorithm as a function of the true regret of the meta oracle.

We start by expanding the first term of Eq.  \ref{ali:25} as
\begin{equation}\label{equ:cal}
\begin{aligned}
\mathbb{E}_{\{\chi_{h}^{\text{MT}(k)}\}} &\left[ \text{REV}_* \left( \{\theta_h^{(k)}\},N - \mathcal{N}_k \right) - \text{REV} \left( \{\theta_h^{(k)}\}, \{\theta_{\mathcal{N}_k +1,h}^{\text{MT}(k)}\}, \{\Sigma_{\mathcal{N}_k +1,h}^{\text{MT}(k)}\},N - \mathcal{N}_k \right) \mid \mathcal{E} \right] \\
= \int_{\{\chi_{h}^{\text{MT}(k)}\}} &\frac{\exp \left( -\sum_{h=1}^{H}\|\chi_{h}^{\text{MT}(k)}\|^2 / 2\beta \right)}{(2\pi\beta)^{H\mathcal{N}_k/2}} \left[ \text{REV}_* \left( \{\theta_h^{(k)}\},N - \mathcal{N}_k \right) - \text{REV} \left( \{\theta_h^{(k)}\}, \{\theta_{\mathcal{N}_k +1,h}^{\text{MT}(k)}\}, \{\Sigma_{\mathcal{N}_k +1,h}^{\text{MT}(k)}\},N - \mathcal{N}_k \right) \right] \\
& d\chi_{h}^{\text{MT}(k)} \mid \mathcal{E}.
\end{aligned}
\end{equation}

Given a realization of $\chi_{h}^{\text{MT}(k)}$, we denote $\chi_{h}^{\text{TS}(k)}(\chi_{h}^{\text{MT}(k)})$ (with some abuse of notation) as the corresponding realization of $\chi_{h}^{\text{TS}(k)}$ that satisfies Eq.  \ref{equ:err}. Note that this is a unique one-to-one mapping. We then perform a change of measure to continue:
\begin{align*}
&\int_{\chi_{h}^{\text{MT}(k)}} \frac{\exp\left(-\left\|\chi_{h}^{\text{MT}(k)}\right\|^{2}/2\beta\right)}{\exp\left(-\left\|\chi_{h}^{\text{TS}(k)}(\chi_{h}^{\text{MT}(k)})\right\|^{2}/2\beta\right)}\frac{\exp\left(-\left\|\chi_{h}^{\text{TS}(k)}(\chi_{h}^{\text{MT}(k)})\right\|^{2}/2\beta\right)}{(2\pi\beta)^{\mathcal{N}_{k}/2}} d\chi_{h}^{\text{MT}(k)}  \mid \mathcal{E} \\
&= \int_{\left\|\chi_{h}^{\text{MT}(k)}\right\|\leq 4\sqrt{\beta\mathcal{N}_{k}\log_{e}(2KN)}} \frac{\exp\left(-\left\|\chi_{h}^{\text{MT}(k)}\right\|^{2}/2\beta\right)}{\exp\left(-\left\|\chi_{h}^{\text{TS}(k)}(\chi_{h}^{\text{MT}(k)})\right\|^{2}/2\beta\right)}\frac{\exp\left(-\left\|\chi_{h}^{\text{TS}(k)}(\chi_{h}^{\text{MT}(k)})\right\|^{2}/2\beta\right)}{(2\pi\beta)^{\mathcal{N}_{k}/2}} d\chi_{h}^{\text{MT}(k)}  \mid \mathcal{E} \\
&\quad + \int_{\left\|\chi_{h}^{\text{MT}(k)}\right\|\geq 4\sqrt{\beta\mathcal{N}_{k}\log_{e}(2KN)}} \frac{\exp\left(-\left\|\chi_{h}^{\text{MT}(k)}\right\|^{2}/2\beta\right)}{\exp\left(-\left\|\chi_{h}^{\text{TS}(k)}(\chi_{h}^{\text{MT}(k)})\right\|^{2}/2\beta\right)}\frac{\exp\left(-\left\|\chi_{h}^{\text{TS}(k)}(\chi_{h}^{\text{MT}(k)})\right\|^{2}/2\beta\right)}{(2\pi\beta)^{\mathcal{N}_{k}/2}} d\chi_{h}^{\text{MT}(k)}  \mid \mathcal{E} \\
&\leq \max_{\left\|\chi_{h}^{\text{MT}(k)}\right\|\leq 4\sqrt{\beta\mathcal{N}_{k}\log_{e}(2KN)}} \exp\left(\frac{\left\|\chi_{h}^{\text{TS}(k)}(\chi_{h}^{\text{MT}(k)})\right\|^{2}-\left\|\chi_{h}^{\text{MT}(k)}\right\|^{2}}{2\beta}\right) \\
&\quad \times \int_{\left\|\chi_{h}^{\text{MT}(k)}\right\|\leq 4\sqrt{\beta\mathcal{N}_{k}\log_{e}(2KN)}} \frac{\exp\left(-\left\|\chi_{h}^{\text{TS}(k)}(\chi_{h}^{\text{MT}(k)})\right\|^{2}/2\beta\right)}{(2\pi\beta)^{\mathcal{N}_{k}/2}} d\chi_{h}^{\text{MT}(k)}  \mid \mathcal{E} \\
&\quad + \int_{\left\|\chi_{h}^{\text{MT}(k)}\right\|\geq 4\sqrt{\beta\mathcal{N}_{k}\log_{e}(2KN)}} \frac{\exp\left(-\left\|\chi_{h}^{\text{MT}(k)}\right\|^{2}/2\beta\right)}{\exp\left(-\left\|\chi_{h}^{\text{TS}(k)}(\chi_{h}^{\text{MT}(k)})\right\|^{2}/2\beta\right)}\frac{\exp\left(-\left\|\chi_{h}^{\text{TS}(k)}(\chi_{h}^{\text{MT}(k)})\right\|^{2}/2\beta\right)}{(2\pi\beta)^{\mathcal{N}_{k}/2}} d\chi_{h}^{\text{MT}(k)}  \mid \mathcal{E}\\
&\leq \max_{ \|\chi_{h}^{\text{MT}(k)}\|\leq 4\sqrt{\beta\mathcal{N}_{k}\log_{e}(2KN)}} \exp\left(\frac{\|\chi_{h}^{\text{TS}(k)}(\chi_{h}^{\text{MT}(k)})\|^{2}-\sum_{h=1}^{H}\|\chi_{h}^{\text{MT}(k)}\|^{2}}{2\beta}\right) \\
& \int_{\chi_{h}^{\text{MT}(k)}} \frac{\exp\left(-\|\chi_{h}^{\text{TS}(k)}(\chi_{h}^{\text{MT}(k)})\|^{2}/2\beta\right)}{(2\pi\beta)^{\mathcal{N}_{k}/2}} d\chi_{h}^{\text{MT}(k)}  \mid \mathcal{E} \\
& + \int_{\|\chi_{h}^{\text{MT}(k)}\|\geq 4\sqrt{\beta\mathcal{N}_{k}\log_{e}(2KN)}} \frac{\exp\left(-\sum_{h=1}^{H}\|\chi_{h}^{\text{MT}(k)}\|^{2}/2\beta\right)}{(2\pi\beta)^{\mathcal{N}_{k}/2}} d\chi_{h}^{\text{MT}(k)}  \mid \mathcal{E}
\end{align*}

Then we plug it back to previous equation \ref{equ:cal}, we have
\begin{equation}\label{ali:27}
\begin{aligned}
&\mathbb{E}_{\{\chi_{h}^{\text{MT}(k)}\}} \left[ \text{REV}_* \left( \{\theta_h^{(k)}\},N - \mathcal{N}_k \right) - \text{REV} \left( \{\theta_h^{(k)}\}, \{\theta_{\mathcal{N}_k +1,h}^{\text{MT}(k)}\}, \{\Sigma_{\mathcal{N}_k +1,h}^{\text{MT}(k)}\},N - \mathcal{N}_k \right) \mid \mathcal{E} \right] \\
&\leq \max_{\{ \|\chi_{h}^{\text{MT}(k)}\|\leq 4\sqrt{\beta\mathcal{N}_{k}\log_{e}(2KN)}\}} \exp\left(\sum_{h=1}^{H}\frac{\|\chi_{h}^{\text{TS}(k)}(\chi_{h}^{\text{MT}(k)})\|^{2}-\sum_{h=1}^{H}\|\chi_{h}^{\text{MT}(k)}\|^{2}}{2\beta}\right) \mathbb{E}_{\{\chi_{h}^{\text{TS}(k)}\}}\bigg[\text{REV}_{*}(\{\theta_h^{(k)}\},N-\mathcal{N}_{k}) \\
& -\text{REV}(\{\theta_h^{(k)}\},\{\{\theta_{\mathcal{N}_{k}+1,h}^{\text{TS}(k)}\}\},\{\Sigma_{\mathcal{N}_{k}+1,h}^{\text{TS}(h)}\},N-\mathcal{N}_{k}) \mid \mathcal{E}\bigg] \\
&\quad + \mathbb{E}_{\{\chi_{h}^{\text{MT}(k)}\}}\bigg[\text{REV}_{*}(\{\theta_h^{(k)}\},N-\mathcal{N}_{k})-\text{REV}(\{\theta_h^{(k)}\},\{\theta_{\mathcal{N}_{k}+1,h}^{\text{MT}(k)}\},\{\Sigma_{\mathcal{N}_{k}+1,h}^{\text{MT}(k)}\},N-\mathcal{N}_{k}) \mid \mathcal{E},\\
& \{\|\chi_{h}^{\text{MT}(k)}\|\geq 4\sqrt{\beta\mathcal{N}_{k}\log_{e}(2KN)}\}\bigg] \times \Pr\left(\{\|\chi_{h}^{\text{MT}(k)}\|\geq 4\sqrt{\beta\mathcal{N}_{k}\log_{e}(2KN)}\}\right).
\end{aligned}
\end{equation}

Here follows from the observation that 
$\text{REV}_{*}(\{\theta_h^{(k)}\}, N - \mathcal{N}_{k}) 
\geq \text{REV}(\{\theta_h^{(k)}\}, \theta, \Sigma, N - \mathcal{N}_{k})$ 
for any choice of $\theta$ and $\Sigma$. 
Accordingly, we can decompose the true regret of our MTSRL algorithm into two parts: 
a leading term that scales with the regret of the meta-oracle, 
and an additional component that depends on the tail probability of $\chi_{h}^{\text{MT}(k)}$. 
To establish our bound, we show that 
(i) the coefficient of the first term converges to one as the epoch index $k$ increases, 
which ensures that the meta-regret vanishes for large epochs; and 
(ii) the second term is negligible with high probability, 
as $\chi_{h}^{\text{MT}(k)}$ follows a sub-Gaussian distribution.

We start by characterizing the core coefficient of the first term:
\begin{equation}\label{ali:28}
\begin{aligned}
\max_{\|\chi_{h}^{\text{MT}(k)}\|\leq 4\sqrt{\beta\mathcal{N}_{k}\log_{e}(2KN)}} \exp&\left(\frac{\|\chi_{h}^{\text{TS}(k)}(\chi_{h}^{\text{MT}(k)})\|^{2}-\|\chi_{h}^{\text{MT}(k)}\|^{2}}{2\beta}\right) \\
= \max_{\|\chi_{h}^{\text{MT}(k)}\|\leq 4\sqrt{\beta\mathcal{N}_{k}\log_{e}(2KN)}} \exp & \bigg(\left(\chi_{h}^{\text{MT}(k)}\right)^{\top}(\Phi_h^{(k)\top} \Phi_h^{(k)})^{-1} \Phi_h^{(k)\top} \Sigma_h^{*-1} \left( \theta^{*}_h - \hat{\theta}_{h}^{(k)} \right)+ \\ & \frac{\beta\left\|(\Phi_h^{(k)\top} \Phi_h^{(k)})^{-1} \Phi_h^{(k)\top} \Sigma_h^{*-1} \left( \theta^{*}_h - \hat{\theta}_{h}^{(k)} \right)\right\|^{2}}{2}\bigg) \\
\leq \max_{\|\chi_{h}^{\text{MT}(k)}\|\leq 4\sqrt{\beta\mathcal{N}_{k}\log_{e}(2KN)}} \exp&\bigg(\left\|\chi_{h}^{\text{MT}(k)}\right\|\left\|(\Phi_h^{(k)\top} \Phi_h^{(k)})^{-1} \Phi_h^{(k)\top} \Sigma_h^{*-1} \left( \theta^{*}_h - \hat{\theta}_{h}^{(k)} \right)\right\|\\
& + \frac{\beta\left\|(\Phi_h^{(k)\top} \Phi_h^{(k)})^{-1} \Phi_h^{(k)\top} \Sigma_h^{*-1} \left( \theta^{*}_h - \hat{\theta}_{h}^{(k)} \right)\right\|^{2}}{2}\bigg) \\
= \exp\bigg(4\sqrt{\beta\mathcal{N}_{k}\log_{e}(2KN)} & \left\|(\Phi_h^{(k)\top} \Phi_h^{(k)})^{-1} \Phi_h^{(k)\top} \Sigma_h^{*-1} \left( \theta^{*}_h - \hat{\theta}_{h}^{(k)} \right)\right\| +\frac{\beta\left\|(\Phi_h^{(k)\top} \Phi_h^{(k)})^{-1} \Phi_h^{(k)\top} \Sigma_h^{*-1} \left( \theta^{*}_h - \hat{\theta}_{h}^{(k)} \right)\right\|^{2}}{2}\bigg). 
\end{aligned}
\end{equation}

Note that
\begin{equation}\label{ali:29}
\begin{aligned}
&4\left\|(\Phi_h^{(k)\top} \Phi_h^{(k)})^{-1} \Phi_h^{(k)\top} \Sigma_h^{*-1} \left( \theta^{*}_h - \hat{\theta}_{h}^{(k)} \right)\right\| \\
&\leq \lambda_{\max}\left((\Phi_h^{(k)\top} \Phi_h^{(k)})^{-1}\right)\sqrt{\lambda_{\max}(\Phi_h^{(k)} \Phi_h^{(k)\top})}\lambda_{\max}(\Sigma_h^{*-1})\left\|\hat{\theta}_{h}^{(k)}-\theta_{h}^{*}\right\| \\
&\leq 32\sqrt{\frac{\mathcal{N}_{k}\Phi_{max}(\beta\lambda_{e}^{-1}+5\bar{\lambda})M\log_{e}(2MK^{2}N)}{\lambda_{e}^{2}\underline{\lambda}^{2}j}} \\
&\leq c_{1}\sqrt{\frac{M\mathcal{N}_{k}\log_{e}(2MK^{2}N)}{\beta k}}. 
\end{aligned}
\end{equation}

Furthermore, by the definition of $K_0$ in Eq. \ref{equ:5} , we have for all $k \geq K_0 + 1$,
\begin{align*}
H\left(4\sqrt{\beta\mathcal{N}_k \log_e (2KN)} \left\| (\Phi_h^{(k)\top} \Phi_h^{(k)})^{-1} \Phi_h^{(k)\top} \Sigma_h^{*-1} \left( \theta^{*}_h - \hat{\theta}_{h}^{(k)} \right) \right\| + \beta\left\| (\Phi_h^{(k)\top} \Phi_h^{(k)})^{-1} \Phi_h^{(k)\top} \Sigma_h^{*-1} \left( \theta^{*}_h - \hat{\theta}_{h}^{(k)} \right) \right\|^2\right) \leq 1.
\end{align*}

Combining Eqs. \ref{ali:28} and \ref{ali:29}, it yields
\begin{align*}
&\max_{\| \chi_{h}^{\text{MT}(k)} \| \leq 4 \sqrt{\beta\mathcal{N}_k \log_e (2KN)}} \exp \left( \frac{\| \chi_{h}^{\text{TS}(k)} (\chi_{h}^{\text{MT}(k)}) \|^2 - \| \chi_{h}^{\text{MT}(k)} \|^2}{2\beta} \right) \\
&\leq \exp \left(4\sqrt{\beta\mathcal{N}_k \log_e (2KN)} \left\| (\Phi_h^{(k)\top} \Phi_h^{(k)})^{-1} \Phi_h^{(k)\top} \Sigma_h^{*-1} \left( \theta^{*}_h - \hat{\theta}_{h}^{(k)} \right) \right\| +\frac{\beta\left\| (\Phi_h^{(k)\top} \Phi_h^{(k)})^{-1} \Phi_h^{(k)\top} \Sigma_h^{*-1} \left( \theta^{*}_h - \hat{\theta}_{h}^{(k)} \right) \right\|^2}{2} \right) \\
&\leq \exp \left( 8\sqrt{\beta\mathcal{N}_k \log_e (2KN)} \left\| (\Phi_h^{(k)\top} \Phi_h^{(k)})^{-1} \Phi_h^{(k)\top} \Sigma_h^{*-1} \left( \theta^{*}_h - \hat{\theta}_{h}^{(k)} \right) \right\|\right) \\
&\leq\exp  \left(2 c_1 \mathcal{N}_k\sqrt{ \frac{M\log_e (2MK^2 N) \log_e (2MK)}{k}} \right). 
\end{align*}

Plugging this into Eq. \ref{ali:27}, and $\exp (a) \le 1 + 2a,a \in [0, 1]$,we can now bound
\begin{equation}\label{ali:32}
\begin{aligned}
&\mathbb{E}_{\{\chi_{h}^{\text{MT}(k)}\}} \left[\text{REV}_{*}\left(\{\theta_{h}^{(k)}\},N-\mathcal{N}_{k}\right)-\text{REV}\left(\{\theta_{h}^{(k)}\},\{\theta_{\mathcal{N}_{k}+1,h}^{\text{MT}(k)}\},\{\Sigma_{\mathcal{N}_{k}+1,h}^{\text{MT}(k)}\},N-\mathcal{N}_{k}\right)\mid\mathcal{E}\right] \\
&\leq \left(1+4Hc_{1}\mathcal{N}_{k}\sqrt{\frac{M\log_{e}(2MK^{2}N)\log_{e}(2MK)}{k}}\right)\mathbb{E}_{\{\chi_{h}^{\text{TS}(k)}\}}\bigg[\text{REV}_{*}\left(\{\theta_{h}^{(k)}\},N-\mathcal{N}_{k}\right) \\
& -\text{REV}\left(\{\theta_{h}^{(k)}\},\{\theta_{\mathcal{N}_{k}+1,h}^{\text{TS}(k)}\},\{\Sigma_{\mathcal{N}_{k}+1,h}^{\text{TS}(h)}\},N-\mathcal{N}_{k}\right)\mid\mathcal{E}\bigg] \\
&\quad +\mathbb{E}_{\{\chi_{h}^{\text{MT}(k)}\}}\bigg[\text{REV}_{*}\left(\{\theta_{h}^{(k)}\},N-\mathcal{N}_{k}\right)-\text{REV}\left(\{\theta_{h}^{(k)}\},\{\theta_{\mathcal{N}_{k}+1,h}^{\text{MT}(k)}\},\{\Sigma_{\mathcal{N}_{k}+1,h}^{\text{MT}(k)}\},N-\mathcal{N}_{k}\right)\mid\mathcal{E}, \\ 
&\{\left\|\chi_{h}^{\text{MT}(k)}\right\|\geq 4\sqrt{\beta\mathcal{N}_{k}\log_{e}(2KN)}\}\bigg] \times \text{Pr}\left(\{\left\|\chi_{h}^{\text{MT}(k)}\right\|\geq 4\sqrt{\beta\mathcal{N}_{k}\log_{e}(2KN)}\}\right). 
\end{aligned}
\end{equation}

As desired, this establishes that the coefficient of our first term decays to $1$ as $j$ grows large. Thus, our meta regret from the first term approaches $0$ for large $j$. We now show that the second term in Eq. \ref{ali:32} is negligible with high probability. Similar to the proof of lemma \ref{lem:j4}, for any $u\in\mathbb{R}$, we can write $\text{Pr}\left(\left\|\chi_{h}^{\text{MT}(k)}\right\|\geq u\right)\leq 2\exp\left(-u^{2}/(10\beta\mathcal{N}_{k})\right)$, which implies
\begin{equation}\label{ali:33}
\begin{aligned}
\text{Pr}\left(\left\|\chi_{h}^{\text{MT}(k)}\right\|\geq 4\sqrt{\beta\mathcal{N}_{k}\log_{e}(2KN)}\right)\leq\frac{1}{KN}.
\end{aligned}
\end{equation}

Moreover, noting that the worst-case regret achievable in a single time period is $1$, and $\mathcal{N}_{k}\leq\mathcal{T}_{e}$ on the event $\mathcal{E}$, we can bound
\begin{equation}\label{ali:34}
\begin{aligned}
&\mathbb{E}_{\{\chi_{h}^{\text{MT}(k)}\}} \bigg[\text{REV}_{*}\left(\{\theta_{h}^{(k)}\},N-\mathcal{N}_{k}\right)-\text{REV}\left(\{\theta_{h}^{(k)}\},\{\theta_{\mathcal{N}_{k}+1,h}^{\text{MT}(k)}\},\{\Sigma_{\mathcal{N}_{k}+1,h}^{\text{MT}(k)}\},N-\mathcal{N}_{k}\right)\mid\mathcal{E},\\
& \left\|\chi_{h}^{\text{MT}(k)}\right\|\geq 4\sqrt{\beta\mathcal{N}_{k}\log_{e}(2KN)}\bigg] \\
&\leq 2(N-\mathcal{N}_{k})H \\
&=O \left({KN}\right)
\end{aligned}
\end{equation}

Substituting Eqs. \ref{ali:33} and \ref{ali:34}, into Eq. \ref{ali:32}, we obtain
\begin{align*}
&\left(1 + 4Hc_1 \mathcal{N}_k \sqrt{\frac{M \log_e (2MK^2N) \log_e (2MK)}{k}} \right) \mathbb{E}_{\{\chi_{h}^{\text{TS}(k)}\}} \bigg[ \text{REV}_* (\{\theta_h^{(k)}\},N - \mathcal{N}_k ) \\ 
& - \text{REV} (\{\theta_h^{(k)}\}, \{\theta_{\mathcal{N}_k +1,h}^{\text{TS}(k)}\}, \{\Sigma_{\mathcal{N}_k +1,h}^{\text{TS}(k)}\},N - \mathcal{N}_k ) \mid \mathcal{E} \bigg] + O \left( \frac{H^2}{K} \right) 
\end{align*}

Substituting the above into Eq.\ref{ali:25}, we can bound the meta regret of epoch $i$ as
\begin{align*}
\mathcal{R}_{K,N}(k) \mid \mathcal{E}
&\leq \left( 4Hc_1 \mathcal{N}_k \sqrt{\frac{M \log_e (2MK^2N) \log_e (2MK)}{k}} \right) \\
&\mathbb{E}_{\{\chi_{h}^{\text{TS}(k)}\}} \left[ \text{REV}_* (\{\theta_h^{(k)}\},N - \mathcal{N}_k ) - \text{REV} (\{\theta_h^{(k)}\}, \{\theta_{\mathcal{N}_k +1,h}^{\text{TS}(k)}\}, \{\Sigma_{\mathcal{N}_k +1,h}^{\text{TS}(k)}\},N - \mathcal{N}_k ) \mid \mathcal{E} \right] + O \left( \frac{\sqrt{d}}{N} \right) \\
&= \tilde{O} \left( H^{4}S^{3/2}A^{1/2}N^{1/2} \sqrt{\frac{1}{k}} + \frac{H^2}{K} \right).
\end{align*}

Here, we have used the fact that the meta oracle's true regret is bounded (Theorem \ref{thm:pri}), i.e.,
\begin{align*}
\mathbb{E}_{\{\chi_{h}^{\text{TS}(k)}\}} \left[ \text{REV}_* (\{\theta_h^{(k)}\},N - \mathcal{N}_k ) - \text{REV} (\{\theta_h^{(k)}\}, \{\theta_{\mathcal{N}_k +1,h}^{\text{TS}(k)}\}, \{\Sigma_{\mathcal{N}_k +1,h}^{\text{TS}(k)}\},N - \mathcal{N}_k ) \mid \mathcal{E} \right] \leq \tilde{O} (H^{3}S^{3/2}\sqrt{AN}).
\end{align*}

The remaining proof of Theorem \ref{thm:MTSRL} follows straightforwardly.

\textbf{Proof of Theorem \ref{thm:MTSRL}.} The meta regret can then be decomposed as follows:
\begin{align*}
\mathcal{R}_{K,N} &= (\mathcal{R}_{K,N} \mid \mathcal{E}) \Pr(\mathcal{E}) + (\mathcal{R}_{K,N} \mid \neg\mathcal{E}) \Pr(\neg\mathcal{E}) \\
&\leq (\mathcal{R}_{K,N} \mid \mathcal{E}) + (\mathcal{R}_{K,N} \mid \neg\mathcal{E}) \Pr(\neg\mathcal{E}).
\end{align*}

Recall that the event $\mathcal{E}$ is composed of event: a bound on $\|\hat{\theta}_h^{(k)} - \theta_h^*\|$ (bounded by Lemma 1). Applying a union bound over the MDP epochs $k \geq K_0 + 1$ to Lemma 1 (setting $\delta = 1/(K^2H^2N)$),and yielding a bound
\begin{align*}
\Pr(\mathcal{E}) \geq 1 - 1/(KHN).
\end{align*}

Recall that when the event $\mathcal{E}$ is violated, the meta regret is $O(NT)$, so we can bound $(\mathcal{R}_{K,N} \mid \neg\mathcal{E}) \Pr(\neg\mathcal{E}) \leq O(KNH \times 1/(KNH)) = O(1)$. Therefore, the overall meta regret is simply
\begin{align*}
\mathcal{R}_{K,N} \leq (\mathcal{R}_{K,N} \mid \mathcal{E}) + O(1).
\end{align*}

When \( k > K_0 \), applying our result in \ref{lem:Rj} yields
\begin{align*}
    \sum_{j=1}^{K_0} (\mathcal{R}_{K,N}(k) | \mathcal{E}) + \sum_{k=K_0+1}^{K} (\mathcal{R}_{K,N}(k) | \mathcal{E}) + O(1) 
    &\leq K_0 \tilde{O}(H^{3}S^{3/2}\sqrt{AN}) + \sum_{k=K_0+1}^{K} \tilde{O} \left(H^{4}S^{3/2}A^{1/2}N^{1/2} \sqrt{\frac{1}{k}} + \frac{H^2}{K}\right)\\
    &+ O(1) \\
    &\leq \sum_{k=1}^{K} \tilde{O} \left(H^{4}S^{3/2}A^{1/2}N^{1/2} \sqrt{\frac{1}{k}} + \frac{H^2}{K}\right) + \tilde{O}(H^{3}S^{3/2}\sqrt{AN})\\
    &= \tilde{O} \left(H^{4}S^{3/2}\sqrt{ANK} \right),
\end{align*}

where we have used the fact that \(\sum_{k=1}^{K} 1/\sqrt{k} \leq 2\sqrt{K}\) in the last step. \hfill $\square$

\section{Proof of Theorem \ref{thm:MTSRLplus}}\label{append:proMSTRL+}

Following the same proof strategy as for the MTSRL algorithm, 
we again employ \emph{prior alignment} to align the means of the meta-oracle’s (random) posterior estimates 
and those of $\text{MTSRL}^{+}$. 
In the previous section, where $\Sigma_h^{*}$ was assumed known, 
equality of the posterior means $\theta_{\mathcal{N}_k+1,h}^{\text{MT}} = \theta_{\mathcal{N}_k+1,h}^{\text{TS}}$ 
implied equality of the full posterior distributions (see Lemma~\ref{lem:prieq}). 
This correspondence allowed us to exactly match the expected regrets of the meta-oracle and our MTSRL algorithm after alignment. 

When $\Sigma_h^{*}$ is unknown, however, 
matching the posterior means $\theta_{\mathcal{N}_k+1,h}^{\text{MTS}} = \theta_{\mathcal{N}_k+1,h}^{\text{TS}}$ 
no longer guarantees equality of the posterior distributions. 
Therefore, the main additional challenge in proving Theorem~\ref{thm:MTSRLplus} 
is to bound the regret gap between $\text{MTSRL}^{+}$ and the meta-oracle 
after aligning the means of their posteriors at time $n = \mathcal{N}_k$. 

Specifically, for each non-exploration epoch $k > K_1$, 
the meta-oracle begins with the true prior $(\{\theta_h^*\}, \{\Sigma_h^*\})$, 
whereas $\text{MTSRL}^{+}$ initializes with the (widened) estimated prior 
$(\{\hat{\theta}_h^{(k)}\}, \{\hat{\Sigma}_h^{w(k)}\})$. 
Lemma~\ref{lem:theta} already bounds the estimation error 
$\|\hat{\theta}_h^{(k)} - \theta_h^{*}\|$, 
and the following lemma provides a bound on the covariance estimation error 
$\|\hat{\Sigma}_h^{(k)} - \Sigma_h^{(k)}\|$, 
as well as on the widened covariance error 
$\|\hat{\Sigma}_h^{w(k)} - \Sigma_h^{(k)}\|$, both with high probability:
\begin{lemma}\label{lem:Sigop}
For any fixed $k \geq 3$ and $\delta \in [0, 2/e]$,if $\lambda_{max}(\Sigma_h^{*}) \le \overline{\lambda}$, then with probability at least $1 - 2\delta$,
\begin{align*}
\left\| \hat{\Sigma}_h^{(k)} - \Sigma_h^{(k)} \right\|_{op} \leq \frac{128(\bar{\lambda}\lambda_e^2 + 8\beta M)}{\lambda_e^2} \left( \sqrt{\frac{5/2 M\log_e(2/\delta)}{k}} \vee \frac{5/2 M\log_e(2/\delta)}{k} \right).    
\end{align*} 
\end{lemma} 

We then proof this core lemma.
\subsection{Convergence of Prior Covariance Estimate}

Lemma \ref{lem:Sigop} shows that, after observing $k$ MDP epochs of length $N$, our estimator $\hat{\Sigma}_h^{(k)}$ is close to $\Sigma^{*}_h$ with high probability. For ease of notation, denote the average of the estimated parameters from each MDP epoch as
\begin{align*}
&\bar{\theta}_{h}^{(k)} = \frac{1}{k-1} \sum_{i=1}^{k-1} \hat{\theta}_{h}^{(i)}, \\
&\Delta_{k}=\hat{\theta}_{h}^{(k)}-\theta_{h}^{(k)}. 
\end{align*}

Then noting that $\mathbb{E}[\Delta_{k}\Delta_{k}^{\top}]=\beta\mathbb{E}\left[V_{\mathcal{N}_kh}^{-1}\right]$ from \ref{sec:PEKP}.
Then, recall from the definition in Eq. \ref{eq:sighat} that
\begin{align*}
\hat{\Sigma}_h^{(k)} = \frac{1}{k-2} \sum_{i=1}^{k-1} \left( \hat{\theta}_{h}^{(i)} - \bar{\theta}_{h}^{(k)} \right) \left( \hat{\theta}_{h}^{(i)} - \bar{\theta}_{h}^{(k)} \right)^{\top} - \frac{\beta}{k-1} \sum_{i=1}^{k-1} \mathbb{E} \left[ V_{\mathcal{N}_ih}^{-1}\right].
\end{align*}

Then, we can expand
\begin{equation}\label{equ:Sigcaop}
\begin{aligned}
\left\|\hat{\Sigma}_h^{(k)}-\Sigma^{*}_h\right\|_{\text{op}} &= \left\|\frac{1}{k-2}\sum_{i=1}^{k-1}\left(\hat{\theta}_{h}^{(i)}-\bar{\theta}_{h}^{(k)}\right)\left(\hat{\theta}_{h}^{(i)}-\bar{\theta}_{h}^{(k)}\right)^{\top}-\frac{\beta\sum_{i=1}^{k-1}\mathbb{E}\left[V_{\mathcal{N}_ih}^{-1}\right]}{k-1}-\Sigma^{*}_h\right\|_{\text{op}} \\
&= \left\|\frac{1}{k-2}\sum_{i=1}^{k-1}\left(\hat{\theta}_{h}^{(i)}-\theta_h^{*}\right)\left(\hat{\theta}_{h}^{(i)}-\theta_h^{*}\right)^{\top}-\frac{k-1}{k-2}\left(\theta_h^{*}-\bar{\theta}_{h}^{(k)}\right)\left(\theta_h^{*}-\bar{\theta}_{h}^{(k)}\right)^{\top}-\frac{\beta\sum_{i=1}^{k-1}\mathbb{E}\left[V_{\mathcal{N}_ih}^{-1}\right]}{k-1}-\Sigma^{*}_h\right\|_{\text{op}} \\
&= \left\|\frac{1}{k-2}\sum_{i=1}^{k-1}\left(\hat{\theta}_{h}^{(i)}-\theta_h^{*}\right)\left(\hat{\theta}_{h}^{(i)}-\theta_h^{*}\right)^{\top}-\frac{k-1}{k-2}\Sigma^{*}_h-\frac{\beta\sum_{i=1}^{k-1}\mathbb{E}\left[V_{\mathcal{N}_ih}^{-1}\right]}{k-2}\right. \\
&\quad \left.-\frac{k-1}{k-2}\left(\theta_h^{*}-\bar{\theta}_{h}^{(k)}\right)\left(\theta_h^{*}-\bar{\theta}_{h}^{(k)}\right)^{\top}+\frac{1}{k-2}\Sigma^{*}_h+\frac{\beta\sum_{i=1}^{k-1}\mathbb{E}\left[V_{\mathcal{N}_ih}^{-1}\right]}{(k-1)(k-2)}\right\|_{\text{op}} \\
&\leq \frac{k-1}{k-2}\left\|\frac{1}{k-1}\sum_{i=1}^{k-1}\left(\hat{\theta}_{h}^{(i)}-\theta_h^{*}\right)\left(\hat{\theta}_{h}^{(i)}-\theta_h^{*}\right)^{\top}-\Sigma^{*}_h-\frac{\beta\sum_{i=1}^{k-1}\mathbb{E}\left[V_{\mathcal{N}_ih}^{-1}\right]}{k-1}\right\|_{\text{op}} \\
&\quad +\frac{k-1}{k-2}\left\|\left(\theta_h^{*}-\bar{\theta}_{h}^{(k)}\right)\left(\theta_h^{*}-\bar{\theta}_{h}^{(k)}\right)^{\top}-\frac{1}{k-1}\Sigma^{*}_h-\frac{\beta\sum_{i=1}^{k-1}\mathbb{E}\left[V_{\mathcal{N}_ih}^{-1}\right]}{(k-1)^{2}}\right\|_{\text{op}}. 
\end{aligned}
\end{equation}

We proceed by showing that each of the two terms is a subgaussian random variable, and therefore satisfies standard concentration results. The following lemma first establishes that both terms have expectation zero, \textit{i.e.}, $\hat{\Sigma}_h^{(k)}$ is an unbiased estimator of the true prior covariance matrix $\Sigma^{*}_h$.

\begin{lemma}\label{lem:sigeq}
For any epoch $k\geq 3$,
\begin{align*}
\mathbb{E}\left[\frac{1}{k-1}\sum_{i=1}^{k-1}\left(\hat{\theta}_{h}^{(i)}-\theta_h^{*}\right)\left(\hat{\theta}_{h}^{(i)}-\theta_h^{*}\right)^{\top}\right]&= \Sigma^{*}_h+\frac{\beta\sum_{i=1}^{k-1}\mathbb{E}\left[V_{\mathcal{N}_ih}^{-1}\right]}{k-1}, \\
\mathbb{E}\left[\left(\theta_h^{*}-\bar{\theta}_{h}^{(k)}\right)\left(\theta_h^{*}-\bar{\theta}_{h}^{(k)}\right)^{\top}\right]&=\frac{1}{k-1}\Sigma^{*}_h+\frac{\beta\sum_{i=1}^{k-1}\mathbb{E}\left[V_{\mathcal{N}_ih}^{-1}\right]}{(k-1)^{2}}.
\end{align*}
\end{lemma}

\textbf{Proof of Lemma \ref{lem:sigeq}.}The random exploration time steps are completed before $n$ time steps. Then noting that $\mathbb{E}[\theta_{h}^{(k)}]=\theta_h^{*}$, $\mathbb{E}[\Delta_{k}]=0$, we can write
\begin{align*}
\mathbb{E}\left[\left(\hat{\theta}_{h}^{(i)}-\theta_h^{*}\right)(\hat{\theta}_{h}^{(i)}-\theta_h^{*\top})\right] &=\mathbb{E}\left[\left(\theta_{h}^{(i)}+\Delta_{k}\right)(\theta_{h}^{(i)}+\Delta_{k})^{\top}-\theta_h^{*}\theta_h^{*\top}\right] \\
&=\mathbb{E}\left[\theta_{h}^{(i)}\theta_{h}^{(i)\top}-\theta_h^{*}\theta_h^{*\top}\right]+\mathbb{E}\left[\Delta_{k}\Delta_{k}^{\top}\right] \\
&=\Sigma^{*}_h+\beta\mathbb{E}\left[V_{\mathcal{N}_kh}^{-1}\right].
\end{align*}

Summing over $i$ and dividing by $(k-1)$ on both sides yields the first statement. For the second statement, we can write
\begin{align*}
\mathbb{E}\left[\left(\bar{\theta}_{h}^{(k)}-\theta_h^{*}\right)(\bar{\theta}_{h}^{(k)}-\theta_h^{*})^{\top}\right]&=\mathbb{E}\left[\bar{\theta}_{h}^{(k)}\bar{\theta}_{h}^{(j)\top}-\theta_h^{*}\theta_h^{*\top}\right]\\
&=\mathbb{E}\left[\left(\frac{\sum_{i=1}^{k-1}\hat{\theta}_h^{(i)}}{k-1}\right)\left(\frac{\sum_{i=1}^{k-1} \hat{\theta}_h^{(i)}}{k-1}\right)^{\top}-\theta_h^{*}\theta_h^{*\top}\right]\\
&=\mathbb{E}\left[\frac{\sum_{i=1}^{k-1}\theta_h^{(i)}\theta_h^{(i)\top}+\sum_{i=1}^{k-1}\Delta_i\Delta_i^{\top}+ \sum_{1\leq k_1<k_2\leq k-1}\theta_{k_1}\theta_{k_2}^{\top}}{(k-1)^2}-\theta_h^{*}\theta_h^{*\top}\right]\\
&=\mathbb{E}\left[\frac{\sum_{i=1}^{k-1}\theta_h^{(i)}\theta_h^{(i)\top}+\sum_{i=1}^{k-1}\Delta_i\Delta_i^{\top}} {(k-1)^2}-\frac{1}{k-1}\theta_h^{*}\theta_h^{*\top}\right]\\
&=\frac{1}{k-1}\Sigma_h^{*}+\frac{\beta\sum_{i=1}^{k-1}\mathbb{E}\left[V_{\mathcal{N}_ih}^{-1}\right]}{(k-1)^2}\,.
\end{align*}

$\square$

Having established that both terms in Eq. \ref{equ:Sigcaop} have expectation zero, the following lemma shows that these terms are subgaussian and therefore concentrate with high probability.

\begin{lemma}
For any $\delta \in [0,1]$, the following holds with probability at least $1-2\delta$:

\begin{align*}
&\left\| \frac{\sum_{i=1}^{k-1} \left( \hat{\theta}_h^{(i)} - \theta_h^{*} \right) \left( \hat{\theta}_h^{(i)} - \theta_h^{*} \right)^\top}{k-1} - \Sigma_h^{*} - \frac{\beta \sum_{i=1}^{k-1} \mathbb{E} \left[ V_{\mathcal{N}_ih}^{-1} \right]}{k-1} \right\|_{\text{op}} \\
&\quad \leq \frac{16(\bar{\lambda}^2 + 8\beta M)}{\lambda_e^2} \left( \sqrt{\frac{5/2 M+2\log_e (2/\delta)}{k-1}} \vee \frac{5/2 M+2\log_e (2/\delta)}{k-1} \right), \\
&\left\| (\theta_h^{*} - \bar{\theta}_h^{(k)}) (\theta_h^{*} - \bar{\theta}_h^{(k)})^\top - \frac{1}{k-1} \Sigma_h^{*} - \frac{\beta \sum_{i=1}^{k-1} \mathbb{E} \left[ V_{\mathcal{N}_ih}^{-1} \right]}{(k-1)^2} \right\|_{\text{op}} \\
&\quad \leq \frac{16(\bar{\lambda}^2 + 8\beta M)(5/2 M+2\log_e (2/\delta))}{\lambda_e^2(k-1)}.
\end{align*}
\end{lemma}

\textbf{Proof of Lemma 14.} First, since the OLS estimator is unbiased, we have that $\mathbb{E} \left[ \hat{\theta}_h^{(k)} - \theta_h^{*} \right] = 0$ for all $k$, and consequently, $\mathbb{E} \left[ \bar{\theta}_h^{(k)} - \theta_h^{*} \right] = 0$. Recall also our definition of $\Delta_k$ from Eq. (36). Then, for any $v \in \mathbb{R}^{M}$ such that $\|v\| = 1$, we can write for all $u \in \mathbb{R}$,

\begin{align*}
\mathbb{E} \left[ \exp(u \langle v, \hat{\theta}_h^{(k)} - \theta_h^{*} \rangle) \right] &= \mathbb{E} \left[ \exp(u \langle v, \theta_h^{(k)} - \theta_h^{*} \rangle) \exp(u \langle v, \Delta_j \rangle) \right] \\
&= \mathbb{E} \left[ \exp(u \langle v, \theta_h^{(k)} - \theta_h^{*} \rangle) \right] \mathbb{E} \left[ \exp(u \langle v, \Delta_j \rangle) \right] \\
&= \exp \left( \frac{u^2 v^\top \Sigma_h^{*} v}{2} \right) \mathbb{E} \left[ \exp(u \langle v, \Delta_j \rangle) \right] \\
&\leq \exp \left( u^2 \left( \frac{\bar{\lambda}}{2} + \frac{4\beta M}{\lambda_e^2} \right) \right),
\end{align*}

where we have re-used Lemma \ref{lem:j1} and Lemma 1.5 of \cite{rigollet2018high} in the last step. Similarly,
\begin{align*}
\mathbb{E} \left[ \exp(u \langle v, \bar{\theta}_h^{(k)} - \theta_h^{*} \rangle) \right] \leq \exp \left( \frac{u^2}{k-1} \left( \frac{\bar{\lambda}}{2} + \frac{4\beta M}{\lambda_e^2} \right) \right).
\end{align*}

By definition, along with Lemma \ref{lem:sigeq}, this implies that $\hat{\theta}_h^{(k)} - \theta_h^{*}$ is a $\left( \sqrt{(\bar{\lambda} \lambda_e^2 + 8\beta M)/2\lambda_e^2} \right)$-subgaussian vector and, similarly $\bar{\theta}_h^{(k)} - \theta_h^{*}$ is a $\left( \sqrt{(\bar{\lambda} \lambda_e^2 + 8\beta M)/[\lambda_e^2(k-1)]} \right)$-subgaussian vector. Applying concentration results for subgaussian random variables (see Theorem 6.5 of \cite{wainwright2019high}), we have with probability at least $1 - \delta$,

\begin{align*}
&\left\| \frac{\sum_{i=1}^{k-1} \left( \hat{\theta}_h^{(i)} - \theta_h^{*} \right) \left( \hat{\theta}_h^{(i)} - \theta_h^{*} \right)^\top}{k-1} - \Sigma_h^{*} - \frac{\beta \sum_{i=1}^{k-1} \mathbb{E} \left[ V_{\mathcal{N}_ih}^{-1} \right]}{k-1} \right\|_{\text{op}} \\
&\quad \leq \frac{16 (\bar{\lambda} \lambda_e^2 + 8\beta M)}{\lambda_e^2} \left( \sqrt{\frac{5/2 M + 2 \log_e (2/\delta)}{k-1}} \vee \frac{5/2 M + 2 \log_e (2/\delta)}{k-1} \right).
\end{align*}

Similarly, with probability at least $1 - \delta$,
\begin{align*}
&\left\| (\theta_h^{*} - \bar{\theta}_h^{(i)}) (\theta_h^{*} - \bar{\theta}_h^{(i)})^\top - \frac{1}{k-1} \Sigma_h^{*} - \frac{\beta \sum_{i=1}^{k-1} \mathbb{E} \left[ V_{\mathcal{N}_ih}^{-1} \right]}{(k-1)^2} \right\|_{\text{op}} \\
&\quad \leq \frac{16(\bar{\lambda}^2 + 8\beta M)(5/2 M + 2\log_e (2/\delta))}{\lambda_e^2(k-1)}.
\end{align*}

Combining these with a union bound yields the result. $\square$

\textbf{Proof of Lemma \ref{lem:Sigop}.} We can apply Lemma 14 to Eq. (35). It is helpful to note that $(k-1)/(k-2) \leq 2$ and $1/(k-1) \leq 2/k$ for all $k \geq 3$, and $5/2 M + 2\log_e (2/\delta) \leq 5M\log_e (2/\delta)$ for all $\delta \in [0, 2/e]$. Thus, a second union bound yields the result. $\square$

After establishing Lemma~\ref{lem:Sigop}, we proceed to derive the overall regret bound. 
At time $n = \mathcal{N}_i + 1$, we perform a change of measure to \emph{align} the prior of our $\text{MTSRL}^{+}$ algorithm, 
$(\{\theta_{\mathcal{N}_k+1,h}^{\text{MTS}}\}, \{\Sigma_{\mathcal{N}_k+1,h}^{\text{MTS}}\})$, 
with that of the meta-oracle, 
$(\{\theta_{\mathcal{N}_k+1,h}^{\text{TS}}\}, \{\Sigma_{\mathcal{N}_k+1,h}^{\text{MTS}}\})$. 
By combining Lemma~\ref{lem:Sigop} with the fact that both policies generate identical histories during the random exploration phases, 
we conclude that $\Sigma_{\mathcal{N}_k+1,h}^{\text{TS}}$ and $\Sigma_{\mathcal{N}_k+1,h}^{\text{MTS}}$ 
remain close with high probability in later MDP epochs. 

What remains is to bound the regret difference between the meta-oracle, 
which employs the prior $(\{\theta_{\mathcal{N}_k+1,h}^{\text{TS}}\}, \{\Sigma_{\mathcal{N}_k+1,h}^{\text{TS}}\})$, 
and our $\text{MTSRL}^{+}$ algorithm, 
which uses the prior $(\{\theta_{\mathcal{N}_k+1,h}^{\text{TS}}\}, \{\Sigma_{\mathcal{N}_k+1,h}^{\text{MTS}}\})$. 
Bounding this residual term constitutes the final step of the proof. 
Here, \emph{prior widening} plays a crucial role in guaranteeing that the importance weights remain well-behaved and do not diverge.

\subsection{$\text{MTSRL}^{+}$ Regret Analysis}
We first focus on the more substantive case where $K > K_1$. We define a new clean event

\begin{equation}\label{equ:J}
\begin{aligned}
\mathcal{J} = \left\{ 
\begin{array}{ll}
\forall k \geq K_1, & \mathcal{N}_k \leq \mathcal{N}_e, \\
& \| \hat{\theta}_h^{(k)} - \theta_h^{*} \| \leq 4 \sqrt{\frac{2(\beta/\lambda_e + 5\bar{\lambda})M\log_e(2MKHN)}{k}}, \\
& \| \hat{\Sigma}_h^{(k)} - \Sigma_h^{*} \|_{\text{op}} \leq \frac{128(\bar{\lambda}\lambda + 8\beta M)}{\lambda_e^2} \left( \sqrt{\frac{5/2 M\log_e(2KHN)}{k}} \vee \frac{5/2 M\log_e(2KHN)}{k} \right) (\le w), \\
& \| \theta_h^{(k)} \| \leq S + 5/2 \sqrt{2\beta M\log_e(2K^2N)},
\end{array}
\right.
\end{aligned}    
\end{equation}

which stipulates that for every epoch following the initial $K_1$ exploration epochs: 
(i) Lemma~\ref{lem:tau} holds, ensuring that the number of exploration periods per epoch is small;  
(ii) our estimated prior mean $\hat{\theta}_h^{(k)}$ is close to the true prior mean $\theta_h^{*}$;  
(iii) our estimated prior covariance $\hat{\Sigma}_h^{(k)}$ is close to the true prior covariance $\Sigma_h^{*}$; and  
(iv) the true parameter for epoch $k$, $\theta_h^{(k)} \sim \mathcal{N}(\theta_h^{*}, \Sigma_h^{*})$, is bounded in $\ell_2$-norm.  
All of these properties hold with high probability by Lemmas~\ref{lem:tau}, \ref{lem:theta}, and \ref{lem:Sigop}, and by standard properties of multivariate Gaussian distributions.  
Hence, the event $\mathcal{J}$ itself occurs with high probability.  

We denote by $\mathcal{R}_{K,N}(k) \mid \mathcal{J}$ the meta-regret of epoch $k$ conditioned on the event $\mathcal{J}$ defined in Eq.~\ref{equ:J}.  
As discussed earlier, during the exploration periods $1 \leq n \leq \mathcal{N}_k$, both the meta-oracle and our $\text{MTSRL}^{+}$ algorithm experience identical histories and thus achieve the same expected rewards; consequently, the conditional meta-regret in these periods is zero.  
Following the argument used in the proof of Theorem~\ref{thm:MTSRL}, we can then express
\begin{align*}
\mathcal{R}_{K,N}(k) \mid \mathcal{J} &= \mathbb{E}_{\{\theta_h^{(k)}\}, \{\hat{\theta}_h^{(k)}\}, \{\chi_h^{\text{TS}(k)}\}, \{\chi_h^{\text{MTS}(k)}\}} \bigg[ \text{REV} \left( \{\theta_h^{(k)}\}, \{\theta_{\mathcal{N}_k+1,h}^{\text{TS}}\}, \{\Sigma_{\mathcal{N}_k+1,h}^{\text{TS}}\}, N - \mathcal{N}_k\right)\\
& - \text{REV} \left( \{\theta_h^{(k)}\}, \{\theta_{\mathcal{N}_k+1,h}^{\text{MTS}}\}, \{\Sigma_{\mathcal{N}_k+1,h}^{\text{MTS}}\}, N - \mathcal{N}_k\right) \mid \mathcal{J} \bigg] \\
&= \mathbb{E}_{\{\theta_h^{(k)}\}, \{\hat{\theta}_h^{(k)}\}, \{\chi_h^{\text{MTS}(k)}\}} \left[ \text{REV}_* \left( \{\theta_h^{(k)}\}, N - \mathcal{N}_k\right) - \text{REV} \left( \{\theta_h^{(k)}\}, \{\theta_{\mathcal{N}_k+1,h}^{\text{MTS}}\}, \{\Sigma_{\mathcal{N}_k+1,h}^{\text{MTS}}\}, N - \mathcal{N}_k\right) \mid \mathcal{J} \right] \\
&\quad - \mathbb{E}_{\{\theta_h^{(k)}\}, \{\hat{\theta}_h^{(k)}\}, \{\chi_h^{\text{TS}(k)}\}} \left[ \text{REV}_* \left( \{\theta_h^{(k)}\}, N - \mathcal{N}_k\right) - \text{REV} \left( \{\theta_h^{(k)}\}, \{\theta_{\mathcal{N}_k+1,h}^{\text{TS}}\}, \{\Sigma_{\mathcal{N}_k+1,h}^{\text{TS}}\}, N - \mathcal{N}_k\right) \mid \mathcal{J} \right]. \tag{38}
\end{align*}

Appendix \ref{sec:inter} states two intermediate lemmas and Appendix \ref{sec:proof3} provides the proof of Theorem 3.

\subsubsection{Intermediate Lemmas}\label{sec:inter}

First, as we did for the proof of Theorem \ref{thm:MTSRLplus}, we characterize the meta regret accrued by aligning the mean of the meta oracle's posterior $\theta_{\mathcal{N}_k+1,h}^{\text{TS}}$ and the mean of our $\text{MTSRL}^{+}$ algorithm $\theta_{\mathcal{N}_k+1,h}^{\text{MTS}}$.

\begin{lemma}\label{lem:rev1}
For an epoch $k\geq K_1$,
\begin{align*}
&\mathbb{E}_{\{\theta_h^{(k)}\},\{\hat{\theta}_h^{(k)}\},\{\chi_h^{\text{MTS}(k)}\}} \left[\text{REV}_*(\{\theta_h^{(k)}\},N-\mathcal{N}_k)-\text{REV}\left(\{\theta_h^{(k)}\},\{\theta_{\mathcal{N}_k+1,h}^{\text{MTS}}\},\{\Sigma_{\mathcal{N}_k+1,h}^{\text{MTS}}\},N-\mathcal{N}_k\right)\mid \mathcal{J}\right] \\
&\leq \left(1+\frac{16c_3M^{3/2}\mathcal{N}_k\log_e^{3/2}(2MK^2N)}{\sqrt{k}}\right)\mathbb{E}_{\{\theta_h^{(k)}\},\{\hat{\theta}_h^{(k)}\},\{\chi_h^{\text{TS}(k)}\}} \bigg[\text{REV}_*(\{\theta_h^{(k)}\},N-\mathcal{N}_k) \\
& -\text{REV}\left(\{\theta_h^{(k)}\},\{\theta_{\mathcal{N}_k+1,h}^{\text{TS}}\},\{\Sigma_{\mathcal{N}_k+1,h}^{\text{TS}}\},N-\mathcal{N}_k\right)\mid \mathcal{J}\bigg] + O\left(\frac{H^2}{K}\right).
\end{align*}

Here:
\begin{align*}
K_1 = \max \left\{ K_0, 64c_2^2 H^2 \mathcal{N}_e^2  \log_e^3 (2MK^2N), c_3^2 N^2 H^2 \log_e^3 (2K^2N) \right\},
\end{align*}

and the constants are given by
\begin{align*}
c_2 &= \frac{8\sqrt{2\beta(\beta \lambda_e^{-1} + 5\bar{\lambda})M}}{ \lambda_e \overline{\lambda}} + \frac{256(\overline{\lambda}\lambda_e^2 + 8\beta M)}{\lambda_e^2 \underline{\lambda}^2} \left( \frac{8\Phi_{max}}{\lambda_e} + \frac{S\sqrt{\beta}}{\lambda_e} \right), \\
c_3 &= \frac{10^4 M^{5/2}\beta (\overline{\lambda}^2\lambda_e^2 + 16\beta)}{\lambda_e^2 \underline{\lambda}^2}.
\end{align*}
\end{lemma}

\textbf{Proof of lemma \ref{lem:rev1}.} By the posterior update rule of Bayesian linear regression (Bishop 2006), we have

\begin{align*}
\theta_{\mathcal{N}_{k}+1,h}^{\text{TS}(k)} &= \left( \frac{1}{\beta_{\mathcal{N}_k +1}} \sum_{i=1}^{\mathcal{N}_k }\Phi_{h}^\top(s_{ih},a_{ih})\Phi_{h}(s_{ih},a_{ih}) + \Sigma_h^{*-1} \right)^{-1} (\frac{1}{\beta_{\mathcal{N}_k +1}}\sum_{i=1}^{\mathcal{N}_k }\Phi_{h}^\top(s_{ih},a_{ih})b^{\text{TS}(k)}_{ih} + \Sigma_h^{*-1}\theta^*_h), \\
\theta_{\mathcal{N}_{k}+1,h}^{\text{MTS}(k)} &= \left( \frac{1}{\beta_{\mathcal{N}_k +1}} \sum_{i=1}^{\mathcal{N}_k }\Phi_{h}^\top(s_{ih},a_{ih})\Phi_{h}(s_{ih},a_{ih}) + (\hat{\Sigma}_h^{w(k)})^{-1} \right)^{-1} (\frac{1}{\beta_{\mathcal{N}_k +1}}\sum_{i=1}^{\mathcal{N}_k }\Phi_{h}^\top(s_{ih},a_{ih})b^{\text{MTS}(k)}_{ih} + (\hat{\Sigma}_h^{w(k)})^{-1}\hat{\theta}_h^{(k)}).
\end{align*}

Denoting $\Phi_h^{(k)} = (\Phi_{h}^\top(s_{1h}^{(k)},a_{1h}^{(k)}), \ldots, (\Phi_{h}^\top(s_{\mathcal{N}_kh}^{(k)},a_{\mathcal{N}_kh}^{(k)})) \in \mathbb{R}^{M \times \mathcal{N}_k}$, and follow the proof in \ref{lem:prieq},we observe that prior alignment is achieved with $\theta_{\mathcal{N}_{k}+1,h}^{\text{MTS}(k)}=\theta_{\mathcal{N}_{k}+1,h}^{\text{TS}(k)}$ when the following holds:
\begin{equation}\label{equ:errx2}
\begin{aligned}
\underbrace{\chi_h^{\text{TS}(k)}-\chi_h^{\text{MTS}(k)}}_{\Delta_{n}} &= \beta_{\mathcal{N}_k +1}(\Phi_h^{(k)\top}\Phi_h^{(k)})^{-1}\bigg[\left(\hat{\Sigma}_h^{w(k)}\right)^{-1}\hat{\theta}_{h}^{(k)}-\Sigma_h^{*-1}\theta_h^{*} \\
&\quad +\left(\Sigma_h^{*-1}-\left(\hat{\Sigma}_h^{w(k)}\right)^{-1}\right)\left(\left(\hat{\Sigma}_h^{w(k)}\right)^{-1}\hat{\theta}_{h}^{(k)}+\frac{1}{\beta_{\mathcal{N}_k +1} }\Phi_h^{(k)}\Phi_h^{(k)\top}\theta_h^{(k)}+\Phi_h^{(k)}\chi_h^{\text{MTS}(k)}\right)\bigg].
\end{aligned}
\end{equation}

We denote the RHS of the above equation as $\Delta_{n}$ for ease of exposition. While this expression is more complicated than before, it still induces a mapping between $\chi_h^{\text{TS}(k)}$ and $\chi_h^{\text{MTS}(k)}$. We then proceed similarly to the proof of Lemma \ref{lem:Rj}. We start by expanding
\begin{align*}
&\mathbb{E}_{\{\chi_h^{\text{MTS}(k)}\}} \left[ \text{REV}_* (\{\theta_h^{(k)}\}, N- \mathcal{N}_k) - \text{REV} (\{\theta_h^{(k)}\}, \{\theta_{\mathcal{N}_k+1,h}^{\text{MTS}}\}, \{\Sigma_{\mathcal{N}_k+1,h}^{\text{MTS}}\}, N- \mathcal{N}_k) \middle| \mathcal{J} \right] \\
&\leq \int_{\{\chi_h^{\text{MTS}(k)}\}} \frac{\exp \left( -\sum_{h=1}^{H}\|\chi_{h}^{\text{MTS}(k)}\|^2 / 2\beta \right)}{(2\pi\beta)^{H\mathcal{N}_k/2}} \cdot \\
&\left[ \text{REV}_* (\{\theta_h^{(k)}\}, N- \mathcal{N}_k) - \text{REV} (\{\theta_h^{(k)}\}, \{\theta_{\mathcal{N}_k+1,h}^{\text{MTS}}\}, \{\Sigma_{\mathcal{N}_k+1,h}^{\text{MTS}}\}, N- \mathcal{N}_k) \middle| \mathcal{J} \right] dx.
\end{align*}

Given a realization of $\chi_h^{\text{MTS}(k)}$, we denote $\chi_h^{\text{TS}(k)}(\chi_h^{\text{MTS}(k)})$ (with some abuse of notation) as the corresponding realization of $\chi_h^{\text{TS}(k)}$ that satisfies Eq. \ref{equ:errx2}. It is easy to see that this is a unique one-to-one mapping. We then perform a change of measure (similar to Eq. \ref{ali:27}) to continue:
\begin{equation}\label{ali:40}
\begin{aligned}
&\mathbb{E}_{\{\chi_{h}^{\text{MTS}(k)}\}} \left[ \text{REV}_* \left( \{\theta_h^{(k)}\},N - \mathcal{N}_k \right) - \text{REV} \left( \{\theta_h^{(k)}\}, \{\theta_{\mathcal{N}_k +1,h}^{\text{MTS}(k)}\}, \{\Sigma_{\mathcal{N}_k +1,h}^{\text{MT}(k)}\},N - \mathcal{N}_k \right) \mid \mathcal{E} \right] \\
& \quad \leq \max_{\{ \|\chi_{h}^{\text{MTS}(k)}\|\leq 4\sqrt{\beta\mathcal{N}_{k}\log_{e}(2KN)}\}} \exp\left(\sum_{h=1}^{H}\frac{\|\chi_{h}^{\text{TS}(k)}(\chi_{h}^{\text{MT}(k)})\|^{2}-\sum_{h=1}^{H}\|\chi_{h}^{\text{MT}(k)}\|^{2}}{2\beta}\right)\\ &\mathbb{E}_{\{\chi_{h}^{\text{TS}(k)}\}}\left[\text{REV}_{*}(\{\theta_{h}^{(k)}\},N-\mathcal{N}_{k})-\text{REV}(\{\theta_{h}^{(k)}\},\theta_{\mathcal{N}_{k}+1,h}^{\text{TS}(k)},\{\Sigma_{\mathcal{N}_{k}+1,h}^{\text{TS}(h)}\},N-\mathcal{N}_{k}) \mid \mathcal{E}\right] \\
&\quad + \mathbb{E}_{\{\chi_{h}^{\text{MTS}(k)}\}}\bigg[\text{REV}_{*}(\{\theta_{h}^{(k)}\},N-\mathcal{N}_{k})-\text{REV}(\{\theta_{h}^{(k)}\},\{\theta_{\mathcal{N}_{k}+1,h}^{\text{MT}(k)}\},\{\Sigma_{\mathcal{N}_{k}+1,h}^{\text{MT}(k)}\},N-\mathcal{N}_{k}) \mid \mathcal{E}, \\
& \{\|\chi_{h}^{\text{MT}(k)}\|\geq 4\sqrt{\beta\mathcal{N}_{k}\log_{e}(2KN)}\}\bigg]  \times \Pr\left(\{\|\chi_{h}^{\text{MTS}(k)}\|\geq 4\sqrt{\beta\mathcal{N}_{k}\log_{e}(2KN)}\}\right) \\
&\leq \max_{\{ \|\chi_{h}^{\text{MTS}(k)}\|\leq 4\sqrt{\beta\mathcal{N}_{k}\log_{e}(2KN)}\}} \exp\left(\sum_{h=1}^{H}\frac{\|\chi_{h}^{\text{TS}(k)}(\chi_{h}^{\text{MT}(k)})\|^{2}-\sum_{h=1}^{H}\|\chi_{h}^{\text{MT}(k)}\|^{2}}{2\beta}\right) \mathbb{E}_{\{\chi_{h}^{\text{TS}(k)}\}}\bigg[\text{REV}_{*}(\{\theta_{h}^{(k)}\},N-\mathcal{N}_{k})\\
&-\text{REV}(\{\theta_{h}^{(k)}\},\theta_{\mathcal{N}_{k}+1,h}^{\text{TS}(k)},\{\Sigma_{\mathcal{N}_{k}+1,h}^{\text{TS}(h)}\},N-\mathcal{N}_{k}) \mid \mathcal{E}\bigg]  + O \left( \frac{H^2}{K} \right) 
\end{aligned}
\end{equation}

where the last step follows from Eqs. \ref{ali:33} and \ref{ali:34}. Thus, we have expressed the true regret of our $\text{MTSRL}^{+}$ algorithm as the sum of a term that is proportional to the true regret of a policy that is aligned with the meta oracle (\textit{i.e.}, it employs the prior $(\{\theta_{\mathcal{N}_{k}+1,h}^{\text{MTS}(k)}\},\{\Sigma_{\mathcal{N}_{k}+1,h}^{\text{MTS}(k)}\})$, and an additional term that is small (\textit{i.e.}, scales as $1/N$).

We now characterize the coefficient of the first term in Eq. \ref{ali:40}:
\begin{equation}\label{ali:41}
\begin{aligned}
&\max_{\|\chi_{h}^{\text{MTS}(k)}\|\leq 4\sqrt{\beta\mathcal{N}_{k}\log_{e}(2KN)}} \exp\left(\frac{\left\| \chi^{\text{TS}(k)}_{h}(\chi^{\text{MTS}(k)}_{h})\right\|^{2}-\left\|\chi^{\text{MTS}(k)}_{h}\right\|^{2}}{2\beta}\right) \\
&=\max_{\|\chi_{h}^{\text{MTS}(k)}\|\leq 4\sqrt{\beta\mathcal{N}_{k}\log_{e}(2KN)}}\exp\left(\frac{\left\|\chi^{\text{MTS}(k)}_{h}+\Delta_{n}\right\|^{2}-\left\|\chi^{\text{MTS}(k)}_{h}\right\|^{2}}{2\beta}\right) \\
&=\max_{\|\chi_{h}^{\text{MTS}(k)}\|\leq 4\sqrt{\beta\mathcal{N}_{k}\log_{e}(2KN)}}\exp\left(\frac{\left(\chi^{\text{MTS}(k)}_{h}\right)^{\top}\Delta_{n}}{\beta}+\frac{\left\|\Delta_{n}\right\|^{2}}{2\beta}\right) \\
&\leq\max_{\|\chi_{h}^{\text{MTS}(k)}\|\leq 4\sqrt{\beta\mathcal{N}_{k}\log_{e}(2KN)}}\exp\left(\frac{\left\|\chi^{\text{MTS}(k)}_{h}\right\|\left\|\Delta_{n}\right\|}{\beta}+\frac{\left\|\Delta_{n}\right\|^{2}}{2\beta}\right) \\
&=\exp\left(\frac{4\sqrt{\mathcal{N}_{k}\log_{e}(2KN)}\|\Delta_{n}\|}{\sqrt{\beta}}+\frac{\left\|\Delta_{n}\right\|^{2}}{2\beta}\right). 
\end{aligned}
\end{equation}

To continue, we must characterize $\|\Delta_{n}\|$. Applying the triangle inequality, we have that
\begin{equation}\label{ali:42}
\begin{aligned}
\|\Delta_{n}\|\leq \frac{\beta}{\lambda_{e}}\left\|\left(\hat{\Sigma}^{w}_{h}\right)^{-1}\hat{\theta}_{h}^{(k)}-\Sigma^{*-1}_h\theta_h^{*}\right\|+\frac{\beta}{\lambda_{e}}\left\|\left(\Sigma^{*-1}_h-\left(\hat{\Sigma}^{w}_{h}\right)^{-1}\right)\left(\left(\hat{\Sigma}^{w}_{h}\right)^{-1}\hat{\theta}_{h}^{(k)}+\frac{1}{\beta} \Phi_h^{(k)}\Phi_h^{(k)\top}\theta_h^{(k)}+\Phi_h^{(k)}\chi^{\text{MTS}(k)}_{h}\right)\right\|. 
\end{aligned}
\end{equation}

The first term of Eq. \ref{ali:42} satisfies
\begin{align*}
&\frac{\beta}{\lambda_e}\left\| \left(\hat{\Sigma}_h^{w(k)}\right)^{-1}\hat{\theta}_{h}^{(k)}-\Sigma^{*-1}_h\theta_h^{*}\right\| \\
&=\frac{\beta}{\lambda_e}\left\|\Sigma^{*-1}_h\left(\hat{\theta}_{h}^{(k)}-\theta_h^{*}\right)+\left(\left(\hat{\Sigma}_h^{w(k)}\right)^{-1}-\Sigma^{*-1}_h\right)\left(\hat{\theta}_{h}^{(k)}-\theta_h^{*}\right)+\left(\left(\hat{\Sigma}_h^{w(k)}\right)^{-1}-\Sigma^{*-1}_h\right)\theta_h^{*}\right\| \\
&\leq\frac{\beta}{\lambda_e}\left\|\Sigma^{*-1}_h\left(\hat{\theta}_{h}^{(k)}-\theta_h^{*}\right)\right\|+\frac{\beta}{\lambda_e}\left\|\left(\left(\hat{\Sigma}_h^{w(k)}\right)^{-1}-\Sigma^{*-1}_h\right)\left(\hat{\theta}_{h}^{(k)}-\theta_h^{*}\right)\right\|+\frac{\beta}{\lambda_e}\left\|\left(\left(\hat{\Sigma}_h^{w(k)}\right)^{-1}-\Sigma^{*-1}_h\right)\theta_h^{*}\right\| \\
&\leq 4\sqrt{\frac{2\beta^2(\beta/\lambda_e+5\bar{\lambda})M\log_e(2MK^{2}N)}{\lambda_e^2 k}}\left(\frac{1}{\bar{\lambda}}+\left\|\left(\hat{\Sigma}_h^{w(k)}\right)^{-1}-\Sigma^{*-1}_h\right\|_{op}\right)+\frac{S\beta}{\lambda_e}\left\|\left(\hat{\Sigma}_h^{w(k)}\right)^{-1}-\Sigma^{*-1}_h\right\|_{op}.
\end{align*}

Next, the second term of Eq. \ref{ali:42} satisfies
\begin{align*}
&\frac{\beta}{\lambda_e}\left\|\left(\Sigma^{*-1}_h-\left(\hat{\Sigma}_h^{w(k)}\right)^{-1}\right)\left(\left(\hat{\Sigma}_h^{w(k)}\right)^{-1}\hat{\theta}_{h}^{(k)}+\frac{1}{\beta} \Phi_h^{(k)}\Phi_h^{(k)\top}\theta_h^{(k)}+\Phi_h^{(k)}\chi_h^{\text{MTS}(k)}\right)\right\| \\
&\leq\frac{\beta\left\|\Sigma^{*-1}_h-\left(\hat{\Sigma}_h^{w(k)}\right)^{-1}\right\|_{op}}{\lambda_e}\left(\left\|\left(\hat{\Sigma}_h^{w(k)}\right)^{-1}\hat{\theta}_{h}^{(k)}\right\|+\left\|\frac{1}{\beta} \Phi_h^{(k)}\Phi_h^{(k)\top}\theta_h^{(k)}\right\|+\left\|\Phi_h^{(k)}\chi_h^{\text{MTS}(k)}\right\|\right) \\
&\leq\frac{\beta\left\|\Sigma^{*-1}_h-\left(\hat{\Sigma}_h^{w(k)}\right)^{-1}\right\|_{op}}{\lambda_e}\left(\left\|\left(\hat{\Sigma}_h^{w(k)}\right)^{-1}\right\|_{op}(S+1)+\frac{1}{\beta}\mathcal{N}_{k}\Phi_{max}^2+4\Phi_{max}\sqrt{\beta\mathcal{N}_{k}\log_e(2KN)}\right) \\
&\leq\frac{\beta\left\|\Sigma^{*-1}_h-\left(\hat{\Sigma}_h^{w(k)}\right)^{-1}\right\|_{op}}{\lambda_e}\left(\left\|\Sigma^{*-1}_h\right\|_{op}(S+1)+\frac{1}{\beta}\mathcal{N}_{k}\Phi_{max}^2+4\Phi_{max}\sqrt{\beta\mathcal{N}_{k}\log_e(2KN)}\right)\\
&\leq \frac{8\Phi_{max}\sqrt{\beta\mathcal{N}_{k}\log_e(2KN)}}{\lambda_e}\left\| \Sigma^{*-1}_h - \left( \hat{\Sigma}_h^{w(k)} \right)^{-1} \right\|_{op},
\end{align*}

where penult inequality follows from the fact that $\left\|\hat{\Sigma}_h^{w(k)}\right\|_{op} \geq \left\|\Sigma^{*}_h\right\|_{op}$ (on the event $\mathcal{J}$) and because both matrices are positive semi-definite (since they are covariance matrices). Applying matrix norm inequality, we can simplify the term
\begin{equation}\label{ali:46}
\begin{aligned}
&\left\| \Sigma^{*-1}_h - \left( \hat{\Sigma}_h^{w(k)} \right)^{-1} \right\|_{op} = \left\| \left( \hat{\Sigma}_h^{w(k)} \right)^{-1} \left( \hat{\Sigma}_h^{w(k)} - \Sigma^{*}_h \right) \Sigma^{*-1}_h \right\|_{op} \\
&\leq \left\| \left( \hat{\Sigma}_h^{w(k)} \right)^{-1} \right\|_{op} \left\| \hat{\Sigma}_h^{w(k)} - \Sigma^{*}_h \right\|_{op} \left\| \Sigma^{*-1}_h \right\|_{op} \\
&\leq \frac{256(\overline{\lambda}\lambda_e^2 + 8\beta M)}{\lambda_e^2 \underline{\lambda}^2} \sqrt{\frac{5/2 M \log_e (2K^2N)}{k}}.
\end{aligned}
\end{equation}

Combining Eqs. \ref{ali:42}-\ref{ali:46}, we have
\begin{align*}
\|\Delta_n\| \leq c_2\sqrt{\frac{\beta\mathcal{N}_k\log_e(2MK^2N)\log_e(2K^2N)}{k}}.
\end{align*}

Substituting this expression into Eq. \ref{ali:41}, we can bound the coefficient
\begin{align*}
&\max_{\|X_h^{\text{MTS}(k)}\| \leq 4\sigma\sqrt{\mathcal{N}_k\log_e(2KN)}} \exp\left(\frac{\|X_h^{\text{TS}(k)}(X_h^{\text{MTS}(k)})\|^2-\|X_h^{\text{MTS}(k)}\|^2}{2\beta}\right) \\
&\leq \exp\left(8c_2\mathcal{N}_k\log_e(2K^2N)\sqrt{\frac{\log_e(2MK^2N)}{k}}\right) \\
&\leq \exp{\left(8c_2\mathcal{N}_k\log_e^{3/2}(2MK^2N)\sqrt{\frac{1}{k}}\right),}
\end{align*}

Substituting into Eq. \ref{ali:40} yields the result. $\square$

We will use lemma \ref{lem:rev1} in the proof of Theorem 3 to characterize the meta regret from prior alignment. The next lemma will help us characterize the remaining meta regret due to the difference in the covariance matrices post-alignment.
\begin{lemma}\label{lem:needthm}
When the event $\mathcal{}$ holds, we can write
\begin{align*}
&\prod_{n=\mathcal{N}_{k}+1}^N \max_{\theta: \| \theta - \theta_{nh}^{\text{TS}(k)} \| \leq C} \frac{d\mathcal{N} \left( \theta_{nh}^{\text{TS}(k)}, \Sigma_{nh}^{\text{MTS}(k)} \right)}{d\mathcal{N} \left( \theta_{nh}^{\text{TS}(k)}, \Sigma_{nh}^{\text{TS}(k)} \right)} \leq 1 + \frac{2c_3  N \log_e^{3/2} (2K^2 N)}{\sqrt{k}} \leq 3.
\end{align*}
\end{lemma}

\textbf{Proof of Lemma \ref{lem:needthm}}. By the definition of the multivariate normal distribution, we have
\begin{align*}
&\max_{\theta: \| \theta - \theta_{nh}^{\text{TS}(k)} \| \leq C} \frac{d\mathcal{N} \left( \theta_{nh}^{\text{TS}(k)}, \Sigma_{nh}^{\text{MTS}(k)} \right)}{d\mathcal{N} \left( \theta_{nh}^{\text{TS}(k)}, \Sigma_{nh}^{\text{TS}(k)} \right)} \\
&= \sqrt{\frac{\det \left( \Sigma_{nh}^{\text{TS}(k)} \right)}{\det \left( \Sigma_{nh}^{\text{MTS}(k)} \right)}}\max_{\theta: \| \theta - \theta_{nh}^{\text{TS}(k)} \| \leq C} \exp \bigg( \frac{\left( \theta - \theta_{nh}^{\text{TS}(k)} \right)^\top \left( \Sigma_{nh}^{\text{TS}(k)} \right)^{-1} \left( \theta - \theta_{nh}^{\text{TS}(k)} \right)}{2} \\
& - \frac{\left( \theta - \theta_{nh}^{\text{TS}(k)} \right)^\top \left( \Sigma_{nh}^{\text{MTS}(k)} \right)^{-1} \left( \theta - \theta_{nh}^{\text{TS}(k)} \right)}{2} \bigg) \\
&= \sqrt{\frac{\det \left( \Sigma_{nh}^{\text{TS}(k)} \right)}{\det \left( \Sigma_{nh}^{\text{MTS}(k)} \right)}} \cdot \max_{\theta: \| \theta - \theta_{nh}^{\text{TS}(k)} \| \leq C} \exp \left( \frac{\left( \theta - \theta_{nh}^{\text{TS}(k)} \right)^\top \left( \left( \Sigma_{nh}^{\text{TS}(k)} \right)^{-1} - \left( \Sigma_{nh}^{\text{MTS}(k)} \right)^{-1} \right) \left( \theta - \theta_{nh}^{\text{TS}(k)} \right)}{2} \right) \\
&\leq \sqrt{\frac{\det \left( \left( \hat{\Sigma}_h^{w(k)} \right)^{-1} + \frac{1}{\beta} \sum_{i=1}^n\Phi_{h}^\top(s_{ih},a_{ih})\Phi_{h}(s_{ih},a_{ih}) \right)}{\det \left( \Sigma^{*-1}_h +  \frac{1}{\beta} \sum_{i=1}^n\Phi_{h}^\top(s_{ih},a_{ih})\Phi_{h}(s_{ih},a_{ih}) \right)}} \exp \left( \frac{C^2 \left\| \left( \Sigma_{nh}^{\text{TS}(k)} \right)^{-1} - \left( \Sigma_{nh}^{\text{MTS}(k)} \right)^{-1} \right\|_{op}}{2} \right) \\
&\leq \sqrt{\frac{\det \left( \left( \hat{\Sigma}_h^{w(k)} \right)^{-1} +  \frac{1}{\beta} \sum_{i=1}^n\Phi_{h}^\top(s_{ih},a_{ih})\Phi_{h}(s_{ih},a_{ih}) \right)}{\det \left( \Sigma^{*-1}_h +  \frac{1}{\beta} \sum_{i=1}^n\Phi_{h}^\top(s_{ih},a_{ih})\Phi_{h}(s_{ih},a_{ih}) \right)}} \exp \left( \frac{128C^2 (\overline{\lambda}\lambda_e^2 + 8\beta M)}{\lambda_e^2 \underline{\lambda}^2} \sqrt{\frac{5/2 M \log_e (2K^2 N)}{k}} \right),
\end{align*}

where we have used Eq. \ref{ali:46} in the last step. Since our estimated covariance matrix is widened, we know that on the event $\mathcal{J}$, 
$\Sigma^{*-1}_h - \left( \hat{\Sigma}_h^{w(k)} \right)^{-1} = \Sigma^{*-1}_h \left( \hat{\Sigma}_h^{w(k)} - \Sigma^{*}_h \right) \left( \hat{\Sigma}_h^{w(k)} \right)^{-1}$ is positive semi-definite, and thus it is evident that
$\left( \Sigma^{*-1}_h + \frac{1}{\beta} \sum_{i=1}^n\Phi_{h}^\top(s_{ih},a_{ih})\Phi_{h}(s_{ih},a_{ih}) \right) - \left( \left( \hat{\Sigma}_h^{w(k)} \right)^{-1} + \frac{1}{\beta} \sum_{i=1}^n\Phi_{h}^\top(s_{ih},a_{ih})\Phi_{h}(s_{ih},a_{ih}) \right)$ is also positive semi-definite. Therefore, conditioned on the clean event $\mathcal{J}$, 
$\sqrt{\frac{\det \left( \left( \hat{\Sigma}_h^{w(k)} \right)^{-1} + \frac{1}{\beta} \sum_{i=1}^n\Phi_{h}^\top(s_{ih},a_{ih})\Phi_{h}(s_{ih},a_{ih}) \right)}{\det \left( \Sigma^{*-1}_h + \frac{1}{\beta} \sum_{i=1}^n\Phi_{h}^\top(s_{ih},a_{ih})\Phi_{h}(s_{ih},a_{ih}) \right)}} \leq 1.$The result follows directly.

\subsubsection{Detail for  Proof of Theorem \ref{thm:MTSRLplus}}\label{sec:proof3}

\begin{proof}
First, we consider the ``small $K$'' regime, where $k \leq K_1$. In this case, our $\text{MTSRL}^{+}$ algorithm simply executes $k$ instances prior-independent Thompson sampling. Thus, the result already holds in this case.

We now turn our attention to the ``large $K$'' regime, i.e., $k > K_1$. The meta regret can be decomposed as
\begin{align*}
\mathcal{R}_{K,N} &= (\mathcal{R}_{K,N} | \mathcal{J}) \Pr(\mathcal{J}) + (\mathcal{R}_{K,N} | \neg \mathcal{J}) \Pr(\neg \mathcal{J}) \\
&\leq (\mathcal{R}_{K,N} | \mathcal{J}) + (\mathcal{R}_{K,N} | \neg \mathcal{J}) \Pr(\neg \mathcal{J}).
\end{align*}

Recall that the event $\mathcal{J}$ is composed of four events, each of which hold with high probability. Applying a union bound over the epochs $k \geq K_1 + 1$ to  Lemma \ref{lem:theta} (setting $\delta = 1/(KNH)$), Lemma \ref{lem:Sigop} (with $\delta = 1/(KNH)$), and Eq. \ref{ali:23} (with $u = 5/2 \sqrt{2\beta K \log_e (2N^2 L)}$), we obtain that
\begin{align*}
\Pr(\mathcal{J}) \geq 1 - 6/(KNH) \geq 1 - 6/(KNH).
\end{align*}

Recall that when the event $\mathcal{J}$ is violated, the meta regret is $O(KNH)$, so we can bound $(\mathcal{R}_{K,N} | \neg \mathcal{J}) \Pr(\neg \mathcal{J}) = O(KNH \times 1/(KNH)) = O(1)$. Therefore, the overall meta regret is simply
\begin{align*}
\mathcal{R}_{K,N} \leq (\mathcal{R}_{K,N} | \mathcal{J}) + O(1).
\end{align*}

Thus, it suffices to bound $\mathcal{R}_{K,N} \mid \mathcal{J}$. As described before, we consider bounding the meta regret post-alignment $(N = \mathcal{N}_k + 1, \cdots,N)$, 
where our $\text{MTSRL}^{+}$ algorithm follows the aligned posterior $ (\{\theta_{\mathcal{N}_k+1,h}^{\text{TS}(k)}\},\{\Sigma_{\mathcal{N}_k+1,h}^{\text{MTS}(k)}\})$. Let  $(\{\theta_{nh}^{\text{TS}(k)}\}, \{\Sigma_{nh}^{\text{MTS}(k)})\}$ denote the posterior of our $\text{MTSRL}^{+}$ algorithm at time step $n$, if it begins with the prior $\mathcal{N}(\{\theta_{\mathcal{N}_k+1,h}^{\text{TS}(k)}\},\{\Sigma_{\mathcal{N}_k+1,h}^{\text{MTS}(k)}\})$  in time step $\mathcal{N}_k + 1$,  but follows the randomness of the oracle. Then, we can write
\begin{align*}
&\mathbb{E}_{\{\theta_h^{(k)}\}, \{\hat{\theta}_h^{(k)}\}} \left[ \text{REV}_* (\{\theta_h^{(k)}\}, N - \mathcal{N}_k) - \text{REV} (\{\theta_h^{(k)}\}, \{\theta_{\mathcal{N}_k+1,h}^{\text{TS}(k)}\},\{\Sigma_{\mathcal{N}_k+1,h}^{\text{MTS}(k)}\}, N - \mathcal{N}_k) \middle| \mathcal{J} \right] \\
&= \mathbb{E}_{\{\theta_h^{(k)}\}, \{\hat{\theta}_h^{(k)}\}} \bigg[ \int_{\theta} \text{REV}_* (\{\theta_h^{(k)}\}, N - \mathcal{N}_k) - \text{REV} (\{\theta_h^{(k)}\}, \{\theta_{h}\}, 0, 1) \\
& - \text{REV} (\{\theta_h^{(k)}\}, \{\theta_{\mathcal{N}_k+2,h}^{\text{MTS}(k)}\}, \{\Sigma_{\mathcal{N}_k+2,h}^{\text{MTS}(k)}\}, N - \mathcal{N}_k - 1) \, d\mathcal{N}(\{\theta_{\mathcal{N}_k+1,h}^{\text{TS}(k)}\},\{\Sigma_{\mathcal{N}_k+1,h}^{\text{MTS}(k)}\}) \bigg| \mathcal{J} \bigg] \\
&= \mathbb{E}_{\{\theta_h^{(k)}\}, \{\hat{\theta}_h^{(k)}\}} \bigg[ \int_{\theta: \|\theta\| \leq C} \text{REV}_* (\{\theta_h^{(k)}\}, N - \mathcal{N}_k) - \text{REV} (\{\theta_h^{(k)}\}, \{\theta_{h}\}, 0, 1) \\
& - \text{REV} (\{\theta_h^{(k)}\}, \{\theta_{\mathcal{N}_k+2,h}^{\text{MTS}(k)}\}, \{\Sigma_{\mathcal{N}_k+2,h}^{\text{MTS}(k)}\}, N - \mathcal{N}_k - 1) \, d\mathcal{N}(\{\theta_{\mathcal{N}_k+1,h}^{\text{TS}(k)}\},\{\Sigma_{\mathcal{N}_k+1,h}^{\text{MTS}(k)}\}) \bigg| \mathcal{J} \bigg] \\
&\quad + \mathbb{E}_{\{\theta_h^{(k)}\}, \{\hat{\theta}_h^{(k)}\}} \bigg[ \int_{\theta: \|\theta\| > C} \text{REV}_* (\{\theta_h^{(k)}\}, N - \mathcal{N}_k) - \text{REV} (\{\theta_h^{(k)}\}, \{\theta_{h}\}, 0, 1) \\
&- \text{REV} (\{\theta_h^{(k)}\}, \{\theta_{\mathcal{N}_k+2,h}^{\text{MTS}(k)}\}, \{\Sigma_{\mathcal{N}_k+2,h}^{\text{MTS}(k)}\}, N - \mathcal{N}_k - 1) \, d\mathcal{N}(\{\theta_{\mathcal{N}_k+1,h}^{\text{TS}(k)}\},\{\Sigma_{\mathcal{N}_k+1,h}^{\text{MTS}(k)}\}) \bigg| \mathcal{J} \bigg] \\
&\leq \mathbb{E}_{\{\theta_h^{(k)}\}, \{\hat{\theta}_h^{(k)}\}} \bigg[ \max_{\theta: \|\theta - \theta_{nh}^{\text{TS}(k)}\| \leq C} \left( \frac{d\mathcal{N}(\theta_{\mathcal{N}_k+1,h}^{\text{TS}(k)},\Sigma_{\mathcal{N}_k+1,h}^{\text{MTS}(k)})}{d\mathcal{N}(\theta_{\mathcal{N}_k+1,h}^{\text{TS}(k)}, \Sigma_{\mathcal{N}_k+1,h}^{\text{TS}(k)})} \right)^{H}\\
&\left( \text{REV}_* (\{\theta_h^{(k)}\}, 1) - \text{REV} (\{\theta_h^{(k)}\}, \{\theta_{\mathcal{N}_k+1,h}^{\text{TS}(k)}\}, \{\Sigma_{\mathcal{N}_k+1,h}^{\text{TS}(k)}\}, 1) \right) \bigg| \mathcal{J} \bigg]\\
&+ \mathbb{E}_{\{\theta_h^{(k)}\}, \{\hat{\theta}_{h}^{(k)}\}} \bigg[ \max_{\theta: \| \theta - \theta_{\mathcal{N}_{k}+1,h}^{\text{TS}(k)} \| \leq C} \left(\frac{d\mathcal{N}(\theta_{\mathcal{N}_{k}+1,h}^{\text{TS}(k)}, \Sigma_{\mathcal{N}_{k}+1,h}^{\text{MTS}(k)})}{d\mathcal{N}(\theta_{\mathcal{N}_{k}+1,h}^{\text{TS}(k)}, \Sigma_{\mathcal{N}_{k}+1,h}^{\text{TS}(k)})} \right)^{H}\\
&\left( \text{REV}_* (\{\theta_h^{(k)}\}, N - \mathcal{N}_k- 1) - \text{REV} (\{\theta_h^{(k)}\}, \{\theta_{\mathcal{N}_{k}+2,h}^{\text{TS}(k)}\}, \{\Sigma_{\mathcal{N}_{k}+2,h}^{\text{MTS}(k)}\}, N - \mathcal{N}_k- 1) \right) \bigg| \mathcal{J} \bigg] \\
&+ \mathbb{E}_{\{\theta_h^{(k)}\}, \{\hat{\theta}_{h}^{(k)}\}} \left[ \int_{\theta: \| \theta - \theta_{\mathcal{N}_{k}+1,h}^{\text{TS}(k)} \| > C} \text{REV}_* (\{\theta_h^{(k)}\}, N- \mathcal{N}_k) d\mathcal{N}(\theta_{\mathcal{N}_{k}+1,h}^{\text{TS}(k)}, \Sigma_{\mathcal{N}_{k}+1,h}^{\text{MTS}(k)}) \middle| \mathcal{J} \right],
\end{align*}

where $C = 5/2 \sqrt{2\beta M\log_e (KN)}$. Inductively, we have
\begin{equation}\label{equ:48}
\begin{aligned}
&\mathbb{E}_{\{\theta_h^{(k)}\}, \{\hat{\theta}_{h}^{(k)}\}, \{X_h^{\text{TS}(k)}\}} \left[ \text{REV}_* (\{\theta_h^{(k)}\}, N- \mathcal{N}_k) - \text{REV} (\{\theta_h^{(k)}\}, \theta_{\mathcal{N}_{k}+1,h}^{\text{TS}(k)}, \Sigma_{\mathcal{N}_{k}+1,h}^{\text{MTS}(k)}, N- \mathcal{N}_k) \middle| \mathcal{J} \right] \\
&\leq \mathbb{E}_{\{\theta_h^{(k)}\}, \{\hat{\theta}_{h}^{(k)}\}} \bigg[ \prod_{n= \mathcal{N}_{k}+1}^N \max_{\| \theta - \theta_{nh}^{\text{TS}(k)} \| \leq C} \left(\frac{d\mathcal{N}(\theta_{nh}^{\text{TS}(k)}, \Sigma_{nh}^{\text{MTS}(k)})}{d\mathcal{N}(\theta_{nh}^{\text{TS}(k)}, \Sigma_{nh}^{\text{TS}(k)})}\right)^{H} \\
& \left( \text{REV}_* (\{\theta_h^{(k)}\}, N- \mathcal{N}_k) - \text{REV} (\{\theta_h^{(k)}\}, \theta_{\mathcal{N}_{k}+1,h}^{\text{TS}(k)}, \Sigma_{\mathcal{N}_{k}+1,h}^{\text{TS}(k)}, N- \mathcal{N}_k) \right) \bigg| \mathcal{J} \bigg] + \sum_{n= \mathcal{N}_{k}+1}^N \mathbb{E}_{\{\theta_h^{(k)}\}, \{\hat{\theta}_{h}^{(k)}\}} \\
&\left[ \prod_{n= \mathcal{N}_{k}+2}^N \max_{\| \theta - \theta_{nh}^{\text{TS}(k)} \| \leq C} \left(\frac{d\mathcal{N}(\theta_{nh}^{\text{TS}(k)}, \Sigma_{nh}^{\text{MTS}(k)})}{d\mathcal{N}(\theta_{nh}^{\text{TS}(k)}, \Sigma_{nh}^{\text{TS}(k)})}\right)^{H}  \int_{\theta: \| \theta \| > C} \text{REV}_* (\{\theta_h^{(k)}\}, N- n) d\mathcal{N}(\theta_{nh}^{\text{TS}(k)}, \Sigma_{nh}^{\text{MTS}(k)}) \middle| \mathcal{J} \right].
\end{aligned}
\end{equation}

Applying Lemma \ref{lem:needthm}, we can bound Eq. \ref{equ:48} as
\begin{align*}
&\mathbb{E}_{\{\theta_h^{(k)}\},\{\hat{\theta}_{h}^{(k)}\},\{X_{h}^{\text{TS}(k)}\}}\left[ \text{REV}_{*}\left(\{\theta_h^{(k)}\},N-\mathcal{N}_{k}\right)-\text{REV}\left(\{\theta_h^{(k)}\},\theta_{\mathcal{N}_{k}+1,h}^{\text{TS}(k)},\Sigma_{\mathcal{N}_{k}+1,h}^{\text{MTS}(k)},N-\mathcal{N}_{k}\right)\middle|\mathcal{J}\right] \\
&\leq \left(1+\frac{2c_{3}N\log_{e}^{3/2}(2K^{2}N)}{\sqrt{k}}\right)^{H}\mathbb{E}_{\{\theta_h^{(k)}\},\{\hat{\theta}_{h}^{(k)}\}}\left[\text{REV}_{*}\left(\{\theta_h^{(k)}\},N-\mathcal{N}_{k}\right)-\text{REV}\left(\{\theta_h^{(k)}\},\theta_{\mathcal{N}_{k}+1,h}^{\text{TS}(k)},\Sigma_{\mathcal{N}_{k}+1,h}^{\text{TS}(k)},N-\mathcal{N}_{k}\right)\middle|\mathcal{J}\right] \\
&+\sum_{n=\mathcal{N}_{k}+1}^N\mathbb{E}_{\{\theta_h^{(k)}\},\{\hat{\theta}_{h}^{(k)}\}}\left[e\int_{\theta:\|\theta\|>C}\text{REV}_{*}\left(\{\theta_h^{(k)}\}, N-n\right)d\mathcal{N}(\theta_{nh}^{\text{TS}(k)},\Sigma_{nh}^{\text{MTS}(k)})\middle|\mathcal{J}\right] \\
&= \left(1+\frac{2c_{3}N\log_{e}^{3/2}(2K^{2}N)}{\sqrt{k}}\right)^H\mathbb{E}_{\{\theta_h^{(k)}\},\{\hat{\theta}_{h}^{(k)}\}}\left[\text{REV}_{*}\left(\{\theta_h^{(k)}\},N-\mathcal{N}_{k}\right)-\text{REV}\left(\{\theta_h^{(k)}\},\theta_{\mathcal{N}_{k}+1,h}^{\text{TS}(k)},\Sigma_{\mathcal{N}_{k}+1,h}^{\text{TS}(k)},N-\mathcal{N}_{k}\right)\middle|\mathcal{J}\right]\\
&+O\left(\frac{H^2}{K}\right),
\end{align*}

where we used Eq. \ref{ali:23} in the last step. Thus, we have expressed the post-alignment meta regret as the sum of a term that is proportional to the true regret of the meta oracle and a negligibly small term. We can now apply lemma \ref{lem:rev1} to further include the meta regret accrued from our prior alignment step to obtain
\begin{align*}
&\mathbb{E}_{\{\theta_h^{(k)}\},\{\hat{\theta}_{h}^{(k)}\},\{X_{h}^{\text{MTS}(k)}\}}\left[ \text{REV}_{*}\left(\{\theta_h^{(k)}\},N-\mathcal{N}_{k}\right)-\text{REV}\left(\{\theta_h^{(k)}\},\theta_{\mathcal{N}_{k}+1,h}^{\text{MTS}(k)},\Sigma_{\mathcal{N}_{k}+1,h}^{\text{MTS}(k)},N-\mathcal{N}_{k}\right)\middle|\mathcal{J}\right] \\
&\leq \left(1+\frac{8c_{2}\mathcal{N}_{k}\log_{e}^{3/2}(2MK^{2}N)}{\sqrt{k}}\right)\left(1+\frac{2c_{3}N\log_{e}^{3/2}(2K^{2}N)}{\sqrt{k}}\right)^H \\
&\quad \times\mathbb{E}_{\{\theta_h^{(k)}\},\{\hat{\theta}_{h}^{(k)}\}}\left[\text{REV}_{*}\left(\{\theta_h^{(k)}\},N-\mathcal{N}_{k}\right)-\text{REV}\left(\{\theta_h^{(k)}\},\theta_{\mathcal{N}_{k}+1,h}^{\text{TS}(k)},\Sigma_{\mathcal{N}_{k}+1,h}^{\text{TS}(k)},N-\mathcal{N}_{k}\right)\middle|\mathcal{J}\right]+O\left(\frac{H^2}{K}\right).
\end{align*}

As desired, this establishes that the coefficient of our first term decays to 1 as $k$ grows large. Thus, our meta regret from the first term approaches 0 for large  $k$, and all other terms are clearly negligible.
Noting that  $K > K_{1} = \tilde{O}(N^{2}T^{2})$ in the ``large $K$'' regime, we can upper bound the meta regret as
\begin{align*}
&\sum_{k=K_{1}+1}^{K} \left[ \left(1+\frac{8c_{2}H\mathcal{N}_{k}\log_{e}^{3/2}(2MK^{2}N)}{\sqrt{k}}\right)\left(1+\frac{2c_{3}N\log_{e}^{3/2}(2K^{2}N)}{\sqrt{k}}\right)^H-1 \right] \\
&\quad \times \mathbb{E}_{\{\theta_h^{(k)}\},\{\hat{\theta}_{h}^{(k)}\}} \left[ \text{REV}_{*} \left(\{\theta_h^{(k)}\},N-\mathcal{N}_{k}\right) - \text{REV} \left(\{\theta_h^{(k)}\},\theta^{\text{TS}(k)}_{\mathcal{N}_{k}+1,h},\Sigma^{\text{TS}(k)}_{\mathcal{N}_{k}+1,h},N-\mathcal{N}_{k} \right) \middle| \mathcal{J} \right] + O\left(\frac{H^2}{K}\right) \\
&= \tilde{O} \left( \sum_{k=K_{1}+1}^{K} \frac{H^{4}S^{3/2}A^{1/2}N^{3/2}}{\sqrt{k}} \right) = \tilde{O} \left( H^{4}S^{3/2}\sqrt{AN^3K}  \right) 
\end{align*}
\end{proof}

\section{Bandit Meta-learning Algorithm}\label{sec:MTSBD}
Let $H_n = (s_{11}, a_{11}, r_{11}, \dots, s_{n-1,h}, a_{n-1,h}, r_{n-1,h})$ denote the history of observations made prior to period $n$. Observing the actual realized history $H_n$, the algorithm computes the posterior  $\mathcal{N}\left( \theta^{TS}_{nh},\Sigma^{TS}_{nh} \right), h \in [H]$ for round $n$. Specifically, 
$\underline{b_{ih} \leftarrow r_{ih}}$, the posterior at period l is:
    \begin{align*}
         \theta^{TS}_{nh} &\leftarrow  \left( \frac{1}{\beta_n} \sum_{i=1}^{n-1}\Phi_{h}^\top(s_{ih},a_{ih})\Phi_{h}(s_{ih},a_{ih}) + \Sigma_h^{*-1} \right)^{-1} (\frac{1}{\beta_n}\sum_{i=1}^{n-1}\Phi_{h}^\top(s_{ih},a_{ih})b_{ih} + \Sigma_h^{*-1}\theta^*_h) \\
        \Sigma^{TS}_{nh} &\leftarrow \left( \frac{1}{\beta_n} \sum_{i=1}^{n-1}\Phi_{h}^\top(s_{ih},a_{ih})\Phi_{h}(s_{ih},a_{ih}) +  \Sigma_h^{*-1} \right)^{-1}
    \end{align*}

\begin{algorithm}[t]
    \caption{TSBD($\{\theta^{*}_{h}\}$,$\{\Sigma^{*}_{h}\}$, $n$):Known-Prior \underline{T}hompson \underline{S}ampling in \underline{B}an\underline{d}it}
    \begin{algorithmic}[1]
    \setlength{\itemsep}{0pt}
        \State \textbf{Input:} Data $\left\{\Phi_1(s_{i1}, a_{i1}), r_{i1}, \ldots,\Phi_{H}(s_{iH}, a_{iH}), r_{iH}\right\}_{i < n}$, the noise parameter $\{\beta_n\}_{n=1}^N$,
        \Statex \quad the prior mean vectors $\{\theta^{*}_{h}\}$ and covariance matrixs $\{\Sigma^{*}_{h}\}$,  $\tilde{\theta}_{H+1} = 0$.
        \For{$n = 1, \ldots, N$}
            \For{$h = H, \ldots, 1$}
                \State Compute the posterior $\theta^{TS}_{nh},\Sigma^{TS}_{nh}$
                \State Sample $\tilde{\theta}_{nh} \sim \mathcal{N}\left( \theta^{TS}_{nh},\Sigma^{TS}_{nh} \right)$ from Gaussian posterior
            \EndFor
            \State Observe $s_{l0}$
            \For{$h = 1, \ldots, H-1$}
                \State Sample $a_{nh} \in \arg\max\limits_{\alpha \in \mathcal{A}} \left( \Phi_h \tilde{\theta}_{nh} \right)(s_{nh}, \alpha)$
                \State Observe $r_{nh}$ and $s_{n, h+1}$
            \EndFor
            \State Sample $a_{nH} \in \arg\max\limits_{\alpha \in \mathcal{A}}\left( \Phi_H \tilde{\theta}_{nH} \right)(s_{nH}, \alpha)$
            \State Observe $r_{nH}$
        \EndFor
    \end{algorithmic}
\label{alg_bandit}
\end{algorithm}

And replace TSRL to TSBD in other algorithm, we can get Bandit meta-learning algorithm. The differences are mainly concentrated in choice of $b_{ih}$.